%% file: main.tex
\documentclass[mnsc,nonblindrev]{informs3}

\OneAndAHalfSpacedXI

\input{tex/math.tex}

\input{tex/macros.tex}

\usepackage{natbib}
 \bibpunct[, ]{(}{)}{,}{a}{}{,}%
 \def\bibfont{\small}%

\TheoremsNumberedThrough
\ECRepeatTheorems

\EquationsNumberedThrough

\MANUSCRIPTNO{}

\begin{document}

\RUNAUTHOR{Niazadeh et al.}

\RUNTITLE{Online Learning via Offline Greedy}

\TITLE{Online Learning via Offline Greedy Algorithms:\\ Applications in Market Design and Optimization}

\ARTICLEAUTHORS{%
\AUTHOR{Rad Niazadeh}
\AFF{Chicago Booth School of Business, Operations Management, \EMAIL{rad.niazadeh@chicagobooth.edu}}
\AUTHOR{Negin Golrezaei}
\AFF{MIT Sloan School of Management, Operations Management, \EMAIL{golrezae@mit.edu}}
\AUTHOR{Joshua Wang}
\AFF{Google Research Mountain View, \EMAIL{joshuawang@google.com}}
\AUTHOR{Fransisca Susan}
\AFF{MIT Sloan School of Management, Operations Management, \EMAIL{fsusan@mit.edu}}
\AUTHOR{Ashwinkumar Badanidiyuru}
\AFF{Google Research Mountain View, \EMAIL{ashwinkumarbv@google.com}}
}

\ABSTRACT{
Motivated by online decision-making in time-varying combinatorial environments, we study the problem of transforming offline algorithms to their online counterparts. We focus on offline combinatorial problems that 
are amenable to a constant factor  approximation using a greedy algorithm that is robust to local errors. For such problems, we provide a general framework that efficiently transforms offline robust greedy algorithms to online ones using Blackwell approachability. 
We show that the resulting online algorithms have $O(\sqrt{T})$ (approximate) regret under the full information setting. We further introduce a bandit extension of Blackwell approachability that we call 
Bandit Blackwell approachability. We leverage this notion to transform  greedy robust offline algorithms into a $O(T^{2/3})$ (approximate) regret in the bandit setting. Demonstrating the flexibility of our framework, we apply our offline-to-online transformation to several problems at the intersection of revenue management, market design, and online optimization, including product ranking optimization in online platforms, reserve price optimization in auctions, and  submodular maximization. We also extend our reduction to greedy-like first order methods used in continuous optimization, such as those used for maximizing continuous strong DR monotone submodular functions subject to convex constraints.  We show that our transformation, when applied to these applications, leads to new regret bounds or improves the current known bounds. We complement our theoretical studies by conducting numerical simulations for two of our applications, in both of which we observe that the numerical performance of our transformations outperforms the theoretical guarantees in practical instances.
}

\KEYWORDS{\renewcommand{\thefootnote}{\large{$\dagger$}\normalsize} 
Blackwell approachability,
Offline-to-online,
No-regret,
Submodular maximization, Product ranking, Reserve price optimization.
}

\maketitle

\section{Introduction}
\label{sec:intro}
\input{tex/intro.tex}

\section{Preliminaries and Notations}
\label{sec:prelim}
\input{tex/prelim.tex}

\section{Approximation Algorithms for the Offline Problem: Iterative Greedy}\label{sec:offline_alg}

\input{tex/framework.tex}
\label{sec:framework}
\section{Online Algorithm under  Full Information Feedback Structure }
\label{sec:full-info}

\input{tex/fullinfo.tex}

\section{Online Algorithm under Bandit Information Feedback Structure}
\label{sec:bandit-info}
\input{tex/bandit.tex}

\input{tex/applications.tex}


\ACKNOWLEDGMENT{N.G. was supported in part by the Young Investigator Program (YIP) Award from the Office of Naval
Research (ONR) N00014-21-1-2776 and the MIT Research Support Award. We thank Tim Roughgarden, Dimitris Bertsimas, and Amin Karbasi for their insightful comments during
this work.}

\bibliographystyle{plainnat}
\bibliography{refs}

\newpage
\renewcommand{\theHsection}{A\arabic{section}}
\begin{APPENDICES}
\input{tex/numerics.tex}

\input{tex/appendix.tex}
\input{tex/apx-lowerbound.tex}
\input{tex/conti_SM}

\input{tex/DR_SM}
\input{tex/apx-ranking.tex}
\input{tex/apx-mmr.tex}

\input{tex/apx-usm.tex}
\end{APPENDICES}

\end{document}

%% file: tex/math.tex
\newcounter{examp}[section]

\newcommand{\prob}[2][]{\text{\bf Pr}\ifthenelse{\not\equal{}{#1}}{_{#1}}{}\!\left[{\def\givenn{\middle|}#2}\right]}
\newcommand{\expect}[2][]{\mathbb{E}\ifthenelse{\not\equal{}{#1}}{_{#1}}{}\!\left[{\def\givenn{\middle|}#2}\right]}

\makeatletter
\newcommand*{\pmzeroslash}{%
  \nfss@text{%
    \sbox0{0}%
    \sbox2{/}%
    \sbox4{%
      \raise\dimexpr((\ht0-\dp0)-(\ht2-\dp2))/2\relax\copy2 %
    }%
    \ooalign{%
      \hfill\copy4 \hfill\cr
      \hfill0\hfill\cr
    }%
    \vphantom{0\copy4 }
  }%
}
\makeatother

\makeatletter
\newcommand*{\rom}[1]{\expandafter\@slowromancap\romannumeral #1@}
\makeatother

%% file: tex/macros.tex
\usepackage{graphics,xcolor}
\usepackage{colortbl}
\definecolor{cornellred}{rgb}{0.7, 0.11, 0.11}
\definecolor{maroon}{rgb}{0.52, 0, 0}
\definecolor{dgreen}{rgb}{0.0, 0.5, 0.0}
\definecolor{ballblue}{rgb}{0.13, 0.67, 0.8}
\definecolor{royalblue(web)}{rgb}{0.25, 0.41, 0.88}
\definecolor{bleudefrance}{rgb}{0.19, 0.55, 0.91}
\definecolor{royalazure}{rgb}{0.0, 0.22, 0.66}

\usepackage[ruled,vlined]{algorithm2e}
\usepackage{multirow}
\usepackage{amsfonts}
\usepackage{amsmath}
\usepackage{bbm}
\usepackage{bm}
\usepackage{enumitem}
\usepackage{mathrsfs}
\usepackage{thm-restate}
\usepackage{pgf}
\usepackage{pgfplots}
\pgfplotsset{compat=1.16}
\usepackage{tikz}
\usepackage{csquotes}

\usepackage[LGR,T1]{fontenc}

\usepackage{hyperref}
\hypersetup{
	colorlinks = true,
	linkcolor=cornellred,
	citecolor=royalazure,
	linkbordercolor = {white}
}
\usepackage{cleveref}
\usepackage{dsfont}

\newcommand{\comment}[1]{}

\newcommand{\E}{\mathbb{E}}
\newcommand{\indicator}[1]{\mathbbm{1}\left[\text{#1}\right]}

\newcommand{\bigO}[1]{O\left(#1\right)}
\newcommand{\bigOmega}[1]{\Omega\left(#1\right)}

\newcommand{\norm}[2]{\left\lVert#1\right\rVert_{#2}}
\newcommand{\standardnorm}[1]{\norm{#1}{\infty}}  
\newcommand{\diameter}[1]{\textrm{D}\left(#1\right)}
\newcommand{\diameterdefinition}[1]{\diameter{#1} \triangleq \max_{\bm{x} \in \algspaceB, \bm{y} \in \advspaceB} \standardnorm{#1(\bm{x}, \bm{y})}}
\newcommand{\distance}[3]{d_{#3}\left(#1, #2\right)}
\newcommand{\standarddistance}[2]{\distance{#1}{#2}{\infty}}  

\newcommand{\coordinatevalues}{\mathcal{R}}
\newcommand{\coordinatevalue}{\rho}
\newcommand{\setofdistributions}{\Delta}

\newcommand{\ubar}[1]{\text{\b{$#1$}}}

\newcommand{\usmlower}{\ubar{\bm{z}}}
\newcommand{\usmupper}{\bar{\bm{z}}}
\newcommand{\usmoptimal}{{\zbf}^*}

\newcommand{\indexintovector}[2]{\left[#1\right]_{#2}}
\newcommand{\diagonal}[1]{\text{diag}\left(#1\right)}

\newcommand{\subpeek}[1]{#1_{\textrm{exp}}}

\newcommand{\feasiblereserves}{\mathcal{R}}
\newcommand{\feasiblereserve}{\rho}

\newcommand{\rbf}{\bm{r}}
\newcommand{\vbf}{\bm{v}}
\newcommand{\vibf}{\mathbf{v}}
\newcommand{\rcal}{\mathcal{R}}
\newcommand{\zbf}{\bm{z}}
\newcommand{\pbf}{\bm{p}}
\newcommand{\pay}{\small{\text{\textsc{Payoff}}}}
\newcommand{\xbf}{\bm{x}}
\newcommand{\ybf}{\bm{y}}
\newcommand{\zi}[1]{\bm{z}^{(#1)}}
\newcommand{\yit}[2]{\bm{y}^{(#1)}_{#2}}

\newcommand{\domain}{\mathcal{D}}
\newcommand{\pspace}{{\Theta}}

\newcommand{\update}{\small{\text{\textsc{Local-update}}}}
\newcommand{\algspaceB}{\mathcal{X}}
\newcommand{\advspaceB}{\mathcal{Y}}

\newcommand{\targetsetB}{S}

\newcommand{\advfunB}{\textrm{AdvB}}

\newcommand{\ubf}{\bm{u}}
\newcommand{\constraint}{\mathcal{C}}
\newcommand{\opt}{\textsc{OPT}}

\def \his{\mathcal H}

\newcommand{\BlackwellAlg}{\textrm{AlgB}}
\newcommand{\BanditBlackwellAlg}{\textrm{AlgBB}}

\newcommand{\funcspace}{\mathcal{F}}

\newcommand{\param}{\bm{\theta}}
\newcommand{\paramdimension}{d_\text{param}}
\newcommand{\paramspace}{\Theta}
\newcommand{\payoffdimension}{d_\text{payoff}}

\newcommand{\unbiasedestimator}{\textrm{ExpS}}

\newcommand{\wbf}{\bm{w}}
\newcommand{\qbf}{\bm{q}}
\newcommand{\pibf}{\bm{\pi}}

\newcommand{\blackwellsequentialgame}{(\algspaceB, \advspaceB, \pbf)}
\newcommand{\banditblackwellsequentialgame}{(\algspaceB, \advspaceB, \pbf, \hat{\pbf})}

\usepackage{caption}
\usepackage{subcaption}
\newcommand{\offlinemeta}{\textsc{Offline-IG} (\constraint, \funcspace,\domain, \paramspace)}
\newcommand{\onlinemeta}{\textsc{Online-IG}(\constraint, \funcspace, 
\domain, \paramspace, \BlackwellAlg)}
\newcommand{\banditmeta}{\textsc{Bandit-IG}(\constraint, \funcspace, \domain,  \paramspace, \BanditBlackwellAlg)}
\newcommand{\banditmetaWoInput}{\textsc{Bandit-IG}}

\newcommand{\offlinemetatext}{\textsc{Offline-IG}}
\newcommand{\onlinemetatext}{\textsc{Online-IG}}
\newcommand{\banditmetatext}{\textsc{Bandit-IG}}

\definecolor{blue(munsell)}{rgb}{0.0, 0.5, 0.69}
\newcommand{\schange}[1]{\textcolor{black}{#1}\ignorespaces}

\newcommand{\Pp}{\mathbb{P}}
\newcommand{\tme}{t}
\newcommand{\KL}{\text{KL}}

\newcommand{\myqed}{\tag*{$\blacksquare$}}
\newcommand{\seqsubf}[1]{\textsf{f}_{#1}}

\newcommand{\alg}{\textrm{ALG}}
\newcommand{\adv}{\textrm{ADV}}

\newcommand{\lbf}{\mathbf{l}}
\newcommand{\wibf}{\mathbf{w}}
\newcommand{\xibf}{\mathbf{x}}
\newcommand{\yibf}{\mathbf{y}}
\newcommand{\Dbottom}{\underline{D}}
\newcommand{\Ber}{\text{Ber}}

\newcommand{\Pcal}{\mathcal{P}}

\newcommand{\direction}{{\bm{\rho}}}
\newcommand{\diampolytope}{R_{\Pcal}}
\newcommand{\hatf}{\widehat{f}}

%% file: tex/intro.tex
We study the problem of designing efficient no-regret -- also known as \emph{vanishing regret} -- online learning algorithms in complex real-world environments, where the underlying decision-making process is combinatorial in nature. In such environments, a decision-maker (learner) needs to experiment with exponentially many options whose rewards exhibit non-trivial and non-linear structures. Exploiting such structures to design efficient online learning algorithms is challenging as the underlying offline problems can indeed be NP-hard. Such offline problems can only admit approximation algorithms. Therefore, any efficient online learning algorithm can only hope to obtain vanishing regret with respect to an in-hindsight approximately optimal benchmark. This motivates our key research questions:

\begin{displayquote}
\emph{How can one transform existing approximation algorithms for NP-hard offline problems to vanishing regret learning algorithms for a wide range of combinatorial environments? Can we efficiently exploit the combinatorial reward structure to eliminate the necessity of experimenting with exponentially many arms?}
\end{displayquote}

To answer these questions, we consider an adversarial online learning setting. In every round $t$, the learner takes an action by choosing a (feasible) point $\zbf_t$ among possibly exponentially many choices, and receives a reward of $f_t(\zbf_t)$. The adversarially chosen reward function $f_t\in \funcspace$, which is unknown to the learner at the time of action, can be non-linear in action $\zbf_t$. 
We are interested in settings where the offline problem is NP-hard, and amenable to a $\gamma$-approximation algorithm, where $\gamma \in (0, 1)$.\footnote{Our framework can also be applied to polynomially solvable problems. In this case, the approximation factor is $\gamma=1$.
} 
In the offline problem, the reward function $f\in \funcspace$ is fully known, and the goal is to choose a feasible point $\zbf$ that maximizes the obtained reward $f(\zbf)$. 

We focus on the prevalent class of offline approximation algorithms with a greedy nature. Roughly speaking, such approximation algorithms build up a solution stage by stage, choosing the next stage that offers the most local improvement with respect to a metric. We require the greedy approximation algorithms to be robust to local errors in every stage; for details, see Definition \ref{def:robust-approx}. 
Several combinatorial problems, ranging from classic submodular maximization problems to more recently studied optimization problems related to market design and revenue management, admit such robust greedy approximation algorithms. 
For details, see Section \ref{sec:applications}. 
 
{Our goal here is to conduct  \emph{offline-to-online transformations}; that is, to design online learning algorithms whose performance (over time) is as good as the performance of their corresponding offline approximation algorithm.} 
The problem of offline-to-online transformation is studied by \cite{kalai2005efficient} and \cite{dudik2017oracle} when the learner can solve the offline problem efficiently. However, the approaches in these works fail when the learner only has access to an approximate solutions to the offline problem. This drawback is alleviated by \cite{kakade2009playing} who study the offline-to-online transformation when (i) the offline problem is NP-hard but amenable to approximation, and (ii) the reward function is linear in the learner's action. 
\cite{kakade2009playing} 
crucially uses the linearity of the reward function (see also \cite{garber2021efficient, hazan2018online}), and hence, their approach cannot be applied to our settings with nonlinear reward functions. We highlight that as shown by \cite{hazan2016computational}, for a general offline problem, there may not exist an efficient offline-to-online transformation, justifying our assumption on the type of approximation algorithms. 

We now summarize our main contributions. 
 
\textbf{A framework for offline-to-online transformations.} 
We design a unified framework to transform robust greedy approximation algorithms to efficient online learning algorithms when the reward functions are not necessarily linear. We consider two online learning settings: \emph{full information} and \emph{bandit}. In the full information setting, the learner observes function $f_t$ after taking action $\zbf_t$, and in the bandit setting, the learner only observes the obtained reward $f_t(\zbf_t)$.

For both settings, our proposed transformation relies on the celebrated Blackwell approachability theorem due to \cite{blackwell1956analog}. The Blackwell approachability theorem is concerned with a two-player repeated game with a vector payoff, and presents a strategy under which the time-averaged vector payoff approaches some target set $S$ that satisfies certain properties. As it is shown in \cite{abernethy2011blackwell}, there is a strong connection between Blackwell approachability and designing vanishing regret learning algorithms. In fact, for online linear optimization, they show that any strategy/algorithm for Blackwell approachability can be transformed to a vanishing regret learning algorithm and vice versa.

\textbf{Online learning algorithms using Blackwell strategies.} In this work, as one of our main contributions, we show that the transformation of Blackwell strategies to online vanishing regret algorithms is also possible for combinatorial non-linear learning settings whose underlying offline problem is NP-hard and admits a robust greedy $\gamma$-approximation algorithm. Specifically, we show that if the offline problem is \emph{Blackwell reducible} (see Definitions \ref{def:blackwell-reducible} and \ref{def:bandit-blackwell-reducible}), then we can design an online  learning algorithm with vanishing $\gamma$-regret (as in Definition \ref{def:gamma:regret})  by running a Blackwell algorithm for each stage (subproblem) of the offline greedy algorithm.  In every round, these Blackwell algorithms are run sequentially to build up the learner's action stage by stage. This allows the Blackwell algorithms to communicate with each other in a specific pattern dictated by the offline greedy algorithm. Thanks to such communication between Blackwell algorithms and the robustness of the offline greedy algorithm to local errors, the resulting online algorithm has a vanishing $\gamma$-regret. In fact, for the full information setting, we show that this transformation leads to an algorithm with $O(N \sqrt{T})$ $\gamma$-regret, where $N$ is the number of subproblems in the offline algorithm.\footnote{{Our regret bounds also depend on the diameter of vector payoff of the Blackwell games and their dimension; see Theorems \ref{thm:full-info-online-meta} and \ref{thm:bandit-blackwell-thm}. Further, in some applications we can show sub-linear dependency of the regret bound on the number of sub problems, which turns out to be crucial for transforming continuous optimization algorithms to their online variants; see Theorems  \ref{thm:DR-SM} and \ref{thm:bandit-DR-SM} in Appendix~\ref{subsec:DR-SM}.}}

The bandit setting turns out to be much trickier as the Blackwell algorithms cannot all obtain their desired feedback to update their course of actions over time. To circumvent this obstacle, we introduce a novel and customized bandit version of the Blackwell sequential game that we call \emph{bandit Blackwell}. In this version, the player/algorithm does not obtain any feedback on his payoff unless he agrees to pay a certain cost. When the player agrees to pay such a cost, an extra ``exploration'' will be done, and he obtains an unbiased estimator of his payoff. Surprisingly, we show that in the bandit Blackwell sequential games, getting a vanishing regret with respect to a combination of approachability and exploration cost minimization is feasible (Theorem \ref{thm:bandit-blackwell-thm}). We further give a tight lower bound on the rate of convergence for bandit Blackwell sequential games (Theorem \ref{thm:bandit-blackwell-lower-bound}). 

Leveraging our notions of bandit Blackwell sequential games and approachability, we present an offline-to-online transformation in which $N$ bandit Blackwell algorithms communicate with each other to build up a solution. To mimic the extra exploration step of bandit Blackwell games, we show how this communication can be interrupted in a controlled way when one of the bandit algorithms requests acquiring feedback. We also show how the required unbiased estimator of the vector payoff can be constructed. These pieces  give us  an online algorithm with $O(NT^{2/3})$ $\gamma$-regret.  

\textbf{Applications.} Finally, to demonstrate the generality and effectiveness of our framework, we apply our offline-to-online transformation to several problems at the intersection of revenue management, market design, and online optimization that have been proposed and studied in the literature. In particular, we consider problems of (i) optimizing product ranking, (ii) optimizing personalized reserve prices in second price auctions, and (iii) submodular maximization (SM) in discrete and continuous domains (see Table \ref{tab:results}). We show that in most cases, our transformations lead to new or improved regret bounds. 
In the following, we discuss our bounds in detail.

\vspace{2mm}
\noindent\textbf{\tikz\draw[black,fill=black] (0,0) circle (.5ex);~Product ranking optimization.} Online marketplaces have the opportunity of 
optimizing the ranking of displayed products in order to 
 improve revenue, shape the demand, and reduce users' search cost (see, for example, \cite{athey2011position, kempe2003maximizing, ursu2016power, aouad2015display},  \cite{golrezaei2018two}, and \cite{golrezaei2021learning}).
Inspired by this, we study the product ranking problem in the online adversarial setting.  
{In this problem, the platform needs to identify  a
ranking/permutation of n items across $n$ positions where items placed in top positions (positions with lower indices) get more visibility.  The goal of the platform is 
maximize its user engagement (also
known as market share), which is the probability that a consumer does not leave the platform without
taking a desired action.
To express user engagement as a function of the ranking over the products, we use the model proposed by \cite{asadpour2020ranking}, which is a generalization of the model presented by \cite{ferreira2019learning}.} 
Under this model, the offline ranking problem can be written as maximizing sequential submodular functions; see Section \ref{subsec:ranking} for the definition of these functions. By applying our framework to this problem, we get $O(n\sqrt{T\log n})$  and $O\big(n^{{5/3}}\left(\log n\right)^{1/3}T^{2/3}\big)$ $\frac{1}{2}$-regret in full information and bandit settings, respectively. We note that our work is the first one that studies the product ranking problem under the aforementioned model
in an online adversarial setting.\footnote{The offline PAC learning problem which resembles aspects of the online learning in the stochastic setting, is studied by \cite{ferreira2019learning} for a special case of our model. PAC stands for probably approximately correct.} 

\vspace{2mm}
\noindent\textbf{\tikz\draw[black,fill=black] (0,0) circle (.5ex);~Optimizing personalized reserve prices.} Second price auctions with reserve prices are prevalent in many marketplaces including online advertising markets, making them objects of both wide practical relevance and scientific interest (see, for example, \cite{HR09,cesa2014regret,beyhaghi2018improved,roughgarden2019minimizing, golrezaei2021dynamic}). 
We study the online problem of optimizing personalized reserve prices, where buyers' valuations are chosen adversarially in every round. { In the offline version of this problem, a seller wants to sell an item to one of $n$ bidders by running a second price auction with personalized reserve prices. Each
bidder $i$ has a private value for the item. The seller wishes to maximize his revenue by optimizing over bidders' reserve prices.}
By applying our framework to the offline greedy algorithm of \cite{roughgarden2019minimizing}, we achieve $O(n\sqrt{T\log T})$ $\frac{1}{2}-$regret in the full-information setting and $O(n^{3/5} T^{\schange{4/5}} (\log nT)^{1/3})$ $\frac{1}{2}-$regret in the bandit setting. Our results match the previous bound for the full-information setting by \cite{roughgarden2019minimizing} who apply a slight variant of the Follow-the-Perturbed-Leader algorithm of \cite{kalai2005efficient} every round for each bidder; the bandit setting had not been studied prior to our work.\footnote{In the special case with symmetric buyers and uniform reserve prices (also known as anonymous reserve auction, cf.~\cite{alaei2019optimal}), minimizing regret under stochastic bandit setting is studied in \cite{cesa2014regret}, in which they obtain $\tilde{O}(n\sqrt{T})$ regret bound. Here, the offline problem of finding the uniform optimal reserve can be solved exactly in polynomial time.} 

\vspace{2mm}
\noindent \textbf{\tikz\draw[black,fill=black] (0,0) circle (.5ex);~Submodular maximization problems.} Many 
optimization problems
that arise in the real world, including revenue management problems, can be expressed as maximizing a submodular function. The notion of submodularity is commonly used to describe the diminishing return property in discrete and continuous domains. Examples include the welfare maximization problem (e.g., \cite{dobzinski2006improved} and \cite{vondrak2008optimal}), capital budgeting with risk-averse investors (e.g., \cite{weingartner1967mathematical} and \cite{ahmed2011maximizing}), and the problem of maximizing influence through the network
(e.g., \cite{kempe2003maximizing}).


We apply our framework to the adversarial online submodular maximization problem. For the online problem of maximizing monotone set  submodular functions subject to cardinally constraints with size $k$, we transform the offline greedy algorithm by \cite{nemhauser1978analysis}, which is a $(1-1/e)-$approximation, to yield 
$O(k\sqrt{T\log n})$ $(1-1/e)-$regret in the online full-information setting, matching the bound by \cite{streeter2009online} who use a variation of the EXP3 algorithm. Furthermore, our framework gives $O(kn(\log n)^{1/3}T^{2/3})$ $(1-1/e)-$regret in the bandit setting, improving the previous bound of $O(k^2(n\log n)^{1/3}T^{2/3}(\log T)^2)$ $(1-1/e)-$regret by \cite{streeter2009online,streeter2007online} in the opaque feedback model, which is the limited feedback model that is analog to our bandit feedback model under exploration. See \Cref{sec:bandit-info} for more details.

For the online problem of maximizing non-monotone set submodular functions without any constraints, we transform a variation of the bi-greedy offline algorithm by \cite{buchbinder2018deterministic} using our framework and obtain $O(nT^{1/2})$ $\frac{1}{2}-$regret in the full-information setting, matching the previous bound by \cite{roughgarden2018optimal} who also take advantages of the bi-greedy offline algorithm of \cite{buchbinder2018deterministic}. Here, $n$ is the number of coordinates. For the bandit setting, our transformation yields 
 \schange{$O(nT^{2/3})$}~$\frac{1}{2}-$regret.
To the best of our knowledge, this is the first regret bound for the bandit setting of this challenging problem. 

{Switching to continuous submodular maximization settings, }for the online problem of maximizing non-monotone continuous submodular functions without any constraints, we transform a variation of the continuous bi-greedy algorithm by \cite{niazadeh2018optimal} and obtain $O(n \sqrt{T\log T})$ $\frac{1}{2}-$regret in the online full-information setting. For the bandit setting, we obtain $O(n\schange{T^{4/5}} (\log T)^{1/3})$ $\frac{1}{2}-$regret when the continuous submodular functions is weak-DR.\footnote{We omit the dependence on the Lipschitz constant here.} Our results for weak-DR submodular functions trivially yield results for strong-DR submodular functions. We highlight that 
the notion of weak-DR submodularity is equivalent to continuous submodularity and is easier to satisfy than strong-DR submodularity, which additionally requires coordinate-wise concavity; see the definition of weak-DR and strong-DR submodular functions in \Cref{subsec:usm} in the appendix.
Our work is the first one that designs online algorithms for weak-DR submodular functions. Furthermore, our bounds improve the previous bounds for strong-DR submodular functions by \cite{thang2019online}, which are $O(T^{5/6})$ $\frac{1}{4}-$regret and $O(T^{11/12})$ $\frac{1}{4}-$regret for the full-information and bandit settings, respectively. 


{ For the problem of maximizing monotone continuous strong-DR submodular functions over a downward closed  bounded convex set, by applying our framework to a variant of the Frank-Wolfe algorithm in \cite{bian2016guaranteed}, for the full information setting, we design an online algorithm  with $\bigO{\sqrt{Tn\log{n}}}$ ($1-1/e$)-regret. In terms of dependency on $T$, our regret bound matches the best regret bound in the literature by  \cite{chen2018online}, which is also obtained by an online learning algorithm based on the Frank-Wolfe idea.\footnote{{\cite{chen2018online}, however, considers maximizing monotone  continuous strong-DR submodular functions over a convex set which may not be downward closed. See also  \cite{thang2019online} for a work that builds on the Frank-Wolfe algorithm in \cite{chen2018online} for the non-monotone strong-DR submodular maximization in a downward closed convex set. They obtain $O(T^{3/4})$ $1/e$-regret for the full-information setting.}} For the bandit setting, we design an algorithm with $\bigO{n(\log n)^{1/6}T^{5/6}}$ $(1-1/e)$-regret, improving the previous bound in \cite{zhang2019online}, which is $\bigO{nT^{8/9}}$ for the same approximation factor. 
See Theorems \ref{thm:DR-SM} and \ref{thm:bandit-DR-SM} in Appendix \ref{subsec:DR-SM}.}

{\textbf{Experiments.} To demonstrate the practicality and ease-of-use of our framework, we evaluate our online learning algorithms for the product ranking and maximizing multiple reserves applications numerically. For both applications, our frameworks do better than the benchmark for both full-information and bandit settings on average. Furthermore, as expected, the full-information algorithm has smaller cumulative regret compared to the bandit algorithm. More details on the experiment is in \Cref{sec:numerics} in the appendix.}


\begin{table}[htb]
\smallskip
\caption{Our results for selective applications of our framework, compared to previously known results.}
\label{tab:results}
\smallskip
\smallskip
\renewcommand\arraystretch{1.3}
\renewcommand{\thefootnote}{\fnsymbol{footnote}}
\footnotesize
\begin{center}
\makebox[\textwidth][c]{ 
\scalebox{0.93}{
\begin{tabular}{ |c| c || c| c|| c|c|}
 \cline{3-6} \multicolumn{2}{c|}{} & \multicolumn{2}{c||}{Online Full-Information Setting} & \multicolumn{2}{c|}{Online Bandit Setting} \\ \hline 
 \multirow{2}{*}{Application}   & Approx   &\cellcolor{royalazure!17}Our $\gamma$-Regret & The Best  & Our \cellcolor{royalazure!17}$\gamma$-Regret & The Best \\
                  & Factor ($\gamma$) & \cellcolor{royalazure!17}Bound        & Prior Bound & \cellcolor{royalazure!17}Bound        & Prior Bound \\ \hline
 \multirow{2}{*}{Product Ranking Problem} & \multirow{2}{*}{$1/2$} & \cellcolor{royalazure!7}& \multirow{2}{*}{-} &\cellcolor{royalazure!7} & \multirow{2}{*}{-} \\
 & &\cellcolor{royalazure!7}\multirow{-2}{*}{$O(n\sqrt{T\log n})$} & &\cellcolor{royalazure!7}\multirow{-2}{*}{$O\left(n^{\schange{5/3}}\left(\log n\right)^{1/3}T^{2/3}\right)$} & \\ \hline
 \multirow{2}{*}{Reserve Price Optimization} & \multirow{2}{*}{$1/2$} &\cellcolor{royalazure!7}& \multirow{2}{*}{$O(n \sqrt{T \log T})$~\footnotemark[1]} &\cellcolor{royalazure!7}& \multirow{2}{*}{-} \\
 & &\cellcolor{royalazure!7}\multirow{-2}{*}{$O(n\sqrt{T \log T})$} & &\cellcolor{royalazure!7}\multirow{-2}{*}{$O\left(\schange{n^{3/5}T^{4/5}}\left(\log nT\right)^{1/3}\right)$} & \\ \hline
 Monotone Set SM & \multirow{2}{*}{$1-1/e$} &\cellcolor{royalazure!7} & \multirow{2}{*}{$O\left(k\sqrt{T\log n}\right)$~\footnotemark[2]} &\cellcolor{royalazure!7}& \multirow{2}{*}{$O\left(k^{2}(n\log n)^{1/3}T^{2/3}(\log T)^{2}\right)$~\footnotemark[2]}\\
 with Cardinality Constraints &&\cellcolor{royalazure!7}\multirow{-2}{*}{$O(k\sqrt{T\log n})$} & &\cellcolor{royalazure!7}\multirow{-2}{*}{$O\left(kn\schange{^{2/3}}(\log n)^{1/3}T^{2/3}\right)$} & \\ \hline
 
 Non-Monotone Set & \multirow{2}{*}{$1/2$} & \cellcolor{royalazure!7}& \multirow{2}{*}{$O(n\sqrt{T})$~\footnotemark[3]}&\cellcolor{royalazure!7} & \multirow{2}{*}{-} \\
 SM Functions & &\cellcolor{royalazure!7}\multirow{-2}{*}{$\schange{O(n\sqrt{T})}$} & &\cellcolor{royalazure!7}\multirow{-2}{*}{$\schange{O\left(nT^{2/3}\right)}$} &\\ \hline
 Non-monotone Continuous & \multirow{2}{*}{$1/2$} &\cellcolor{royalazure!7}& \multirow{2}{*}{$\gamma = 1/4$, $O(T^{5/6})$~\footnotemark[4]} &\cellcolor{royalazure!7} & \multirow{2}{*}{$\gamma = 1/4$, $O(T^{11/12})$~\footnotemark[4]}\\
 SM (Strong-DR) Functions & &\cellcolor{royalazure!7}\multirow{-2}{*}{$O(n \sqrt{T \log T})$} & &\cellcolor{royalazure!7}\multirow{-2}{*}{$O(n T^{\schange{4/5}} (\log T)^{{1/3}})$} & \\ \hline
 Non-monotone Continuous & \multirow{2}{*}{$1/2$} & \cellcolor{royalazure!7} & \multirow{2}{*}{-} &\cellcolor{royalazure!7}& \multirow{2}{*}{-} \\
 SM (Weak-DR) Functions & &\cellcolor{royalazure!7}\multirow{-2}{*}{$O(n \sqrt{T \log T})$}
 & &\cellcolor{royalazure!7}\multirow{-2}{*}{$O(n T^{\schange{4/5}} (\log T)^{{1/3}})$} &\\ \hline
{Monotone Cont. SM (Strong-DR)}  & \multirow{2}{*}{{$1-1/e$}} & \cellcolor{royalazure!7} & \multirow{2}{*}{{$\bigO{\sqrt{T}}$}~\footnotemark[5]} &\cellcolor{royalazure!7}& \multirow{2}{*}{{$O(nT^{8/9})$}~\footnotemark[6]} \\
 {in Downward Closed Convex Set}  & &\cellcolor{royalazure!7}\multirow{-2}{*}{{$O(\sqrt{Tn\log{n}})$}}
 & &\cellcolor{royalazure!7}\multirow{-2}{*}{{$O(n(\log n)^{1/6}T^{5/6})$}} &\\\hline
 \end{tabular}}
 } 

	\smallskip
 	
\noindent\par
 	\begin{minipage}{0.8\textwidth}
 		{
 			\center
 			\footnotesize
 			\footnotemark[1]~\cite{roughgarden2019minimizing}
 			\footnotemark[2]~\cite{streeter2009online};\quad 
 			 \footnotemark[3]~\cite{roughgarden2018optimal};\quad
 			 \footnotemark[4]~\cite{thang2019online};\quad
 			 \footnotemark[5]~\cite{chen2018online};
 			 
 			 \footnotemark[6]~\cite{zhang2019one};
 			
 		}
 	\end{minipage}
\end{center}
\end{table}

\subsection{Further Related Work}
\label{app:further}

\emph{Combinatorial learning.} 
Our work is related to the literature on online combinatorial learning. While in our work we study the design of efficient online learning algorithms for combinatorial problems whose loss function is not necessarily linear in the chosen action, the work on combinatorial learning mostly focuses on linear loss functions; see, for example, \cite{abernethy2008competing, uchiya2010algorithms, cesa2012combinatorial, audibert2014regret, chen2013combinatorial, combes2015combinatorial}, and \cite{zimmert2019beating}. 
This line of work examines both the full-information and bandit settings. 
The standard exponentially weighted average forecaster obtains a tight $\bigO{m\sqrt{T\log{\frac{d}{m}}}}$ regret in the full-information setting, where $m$ is the maximum $\ell_1$-norm of action vectors (\cite{audibert2014regret}). The state-of-the-art regret bound for the bandit setting is $\bigO{\sqrt{dm^3T\log \frac{d}{m}}}$, as reported in several papers (\cite{bubeck2012towards}, \cite{cesa2012combinatorial}, \cite{hazan2016volumetric}). 
Our framework achieves matching regret with respect to $T$ in the full-information setting without requiring the loss function to be linear. We get a worse regret (proportional to $T^{2/3}$) for the bandit setting to account for the non-linearity in loss functions.


\emph{Online adversarial submodular optimization. }
In the previous section, we briefly mentioned some of the work that are closely related to our results on maximizing submodular functions in an online adversarial setting. Here, we provide more details. \cite{chen2018projection,chen2019black} use method based on Frank-Wolfe to design vanishing-regret learning algorithms for maximizing monotone continuous strong-DR submodular functions with convex constraints. \cite{chen2018projection} (respectively \cite{chen2019black}) assume that the algorithm can access to $T^{1/2}$ exact (respectively $T^{2/3}$ stochastic) gradient evaluations in every round and design an algorithm whose $(1-1/e)$-regret is $O(\sqrt{T})$.\footnote{The dependency on the number of elements $n$ is not well specified in this work.} 
The results of \cite{chen2018projection,chen2019black} were later improved by \cite{zhang2019online} who design another Frank-Wolfe inspired learning algorithm that has access to one stochastic gradient in each round and obtains $O(T^{4/5})$ $(1-1/e)$-regret. \cite{zhang2019online} further present a learning algorithm in the bandit setting for the problem of maximizing monotone continuous strong-DR submodular functions subject to matroid constraints. Their algorithm obtain $O(T^{8/9})$ ($1-1/e$)-regret. To see how our framework partially improve these results, refer \Cref{tab:results} for a detailed comparison with our results related to monotone continuous submodular maximization subject to downward closed convex sets, and also non-monotone continuous submodular maximization with box constraints (also known as unconstrained). 

\emph{Online stochastic submodular optimization.} Designing learning algorithms for maximizing stochastic monotone continuous strong-DR submodular functions has been studied in \cite{hassani2017gradient, mokhtari2018stochastic, hassani2019stochastic}, and \cite{zhang2019one}. The best result for this setting is by \cite{zhang2019one} who obtain $O(\sqrt{T})$ $(1-1/e)$-regret using a stochastic variant of the Frank-Wolfe method. Their algorithm also implies the same regret bound for monotone set submodular maximization, which matches our regret bound for maximizing monotone set submodular function in the adversarial setting.

\paragraph{Blackwell approachability.} Several aspects of Blackwell sequential game, including the design of efficient algorithms for Blackwell game with various information feedback structures, and the alternative conditions for approachability, 
have been studied in the literature. In terms of feedback structures, the original Blackwell game develops efficient projection algorithm for games that return the adversary's moves on each round. \cite{mannor2011robust} develop simple and efficient algorithms for a variant of Blackwell game where on each round, the player only obtains a random signal whose distribution depends on the action of the player and the adversary (as opposed to the action of the adversary). This variant is called Blackwell approachability with partial monitoring, and is further studied in \cite{mannor2014set} and \cite{kwon2017online}. In terms of equivalent conditions for approachability, aside from the original halfspace-satisfiability condition for approachability in \cite{blackwell1956analog}, alternative conditions for approachability, including the response-satisfiability criteria that we use in this paper, can be found in \cite{lehrer2003approachability}, \cite{vieille1992weak}, \cite{spinat2002necessary}, and \cite{milman2006approachable}.

Blackwell approachability has also been proven to be a quintessential tool in various applications, as shown in \cite{even2009online} and \cite{mannor2006online}. However, most applications do not involve NP-hard combinatorial problems, and use the best-fixed action in hindsight (no approximation factor) as the benchmark for regret. Furthermore, they only create one Blackwell instance on each round. In contrast, we create multiple Blackwell instances on each round because the problems we consider have combinatorial nature and can only be solved efficiently in multiple stages. Furthermore, since we are solving NP-hard combinatorial problems with an intractable offline problem, we use a $\gamma$-approximation benchmark in our regret.

{\textbf{Organization.} In Section \ref{sec:prelim}, we present the offline optimization problem, adversarial online learning framework, and Blackwell sequential games. Section \ref{sec:offline_alg} presents the offline greedy approximation algorithm. 
In Sections \ref{sec:full-info} and \ref{sec:bandit-info}, we present our offline-to-online transformation in the full information and bandit settings, respectively. Section \ref{sec:applications} provides our regret bounds for the product ranking problem and optimizing reserve prices. Our regret bounds for maximizing unconstrained non-monotone submodular functions and maximizing monotone submodular functions over a  downward closed  bounded convex set are respectively presented in Sections \ref{subsec:usm} and \ref{subsec:DR-SM} in the appendix.  In \Cref{sec:numerics} in the appendix, we present our numerical studies.}

%% file: tex/prelim.tex
In this section, we  formulate our adversarial online learning framework for approximation algorithms. We then give an overview of Blackwell approachability~\citep{blackwell1956analog}, an important technical tool that we use in this paper. 

\subsection{Offline Optimization and Approximations}

Let $\funcspace$ be a space of functions defined over a (discrete or continuous) domain $\domain$.  Assume that $\funcspace$ is closed under addition, i.e., for any two functions $f_1, f_2 \in \funcspace$, we have  $f_1+f_2 \in \funcspace$. In the \emph{offline optimization problem}, the problem of interest is finding a point $\zbf^*\in\domain$ such that
\begin{equation}
\label{eq:offline-optimization}
  \zbf^*\in \argmax_{\zbf \in \constraint} f(\zbf)\,,
\end{equation}
where $f: \domain \to [0,1]$, which  belongs to $ \funcspace$, is the objective function, and $\constraint\subseteq \domain$ is the feasible region.\footnote{For maximization problems, which are the focus of this paper, we only need our functions to be upper bounded by a constant. However, for simplicity, we assume that our functions are upper bounded by one.}  We further denote the optimal objective value of problem~(\ref{eq:offline-optimization}) by $\opt$; that is, $\opt=\max_{\zbf \in \constraint} f(\zbf)$. We focus on maximization problems in this paper, but our techniques and results can easily be extended to minimization problems as well. 

We consider offline problems that are NP-hard to solve exactly, and at the same time are amenable to a $\gamma$-approximation algorithm for some constant $\gamma\in(0,1)$. 

\begin{definition}[$\gamma$-approximation offline algorithm]\label{def:offline} An offline algorithm $\mathcal{A}$ for problem~\eqref{eq:offline-optimization} is a  polynomial time $\gamma$-approximation algorithm if for every $f \in \funcspace$ returns a feasible (possibly randomized) point $\hat{\zbf} \in \constraint$ in polynomial time in the size of the algorithm's input such that
\begin{equation*}
    \expect{f(\hat{\zbf})}\geq \gamma\cdot  \opt \,.
\end{equation*}
Here, the expectation is with respect to the randomness in algorithm $\mathcal A$. The constant $\gamma\in(0,1)$ is referred to as the \emph{approximation factor} of algorithm $\mathcal A$.
\end{definition}




\subsection{Adversarial Online Learning and Approximations}

\paragraph{Framework.} In the adversarial online learning version of problem~\eqref{eq:offline-optimization}, there is a learner, denoted by $\alg$, who plays $T$ rounds of a sequential game against an adversary, denoted by $\adv$. In each round $t\in [T]$,  $\adv$ picks a function $f_t\in\funcspace$ and simultaneously $\alg$ takes an action by picking a feasible point $\zbf_t\in \constraint$. Then, $\alg$ obtains a reward equal to $f_t(\zbf_t)$ and receives a \emph{feedback} concerning this round. We highlight that unlike the offline optimization Problem~\eqref{eq:offline-optimization}, the unknown function $f_t$ is not observable to $\alg$ when it chooses action $\zbf_t$, and he only knows that $f_t$ belongs to $\funcspace$. Furthermore, $\alg$ picks his action at time $t$ only given the feedback of previous rounds $1,2,\ldots,t-1$, and in that sense, $\alg$ is an online learner. $\alg$'s goal is to pick points $\{\zbf_t\}_{t=1}^T$ given the feedback of each round to  maximize the accumulated reward $\sum_{t=1}^T f_t(\zbf_t)$ against a worst-case adversary $\adv$. In this paper, for the sake of brevity and simplicity, we limit our focus to worst-case oblivious adversaries, i.e.,  adversaries that pick the sequence $f_1,f_2,\ldots,f_T$ upfront. 

\paragraph{Feedback structures.} We consider two feedback structures: (i)~\emph{full information feedback}, where $\alg$ observes the entire function $f_t$ after choosing $\zbf_t$, and (ii)~\emph{bandit feedback}, where 
  $\alg$ only observes the quantity $f_t(\zbf_t)$ after choosing $\zbf_t$. Let $\phi_t$ be the feedback that $\alg$ receives after picking $\zbf_t$. Then,  $\alg$'s next action $\zbf_{t+1}$ is a function of the history $\his_{t}$, where $\his_{t} \triangleq\{(\zbf_1,\phi_1), \ldots, (\zbf_t, \phi_{t})\}$. More formally, any learning algorithm $\alg$ can be described as mappings $\{\pi^{(t)}_{\alg}\}_{t=1}^{T}$, where each $\pi^{(t)}_{\alg}$  maps the history $\his_{t-1}$ to action $\zbf_t$ for any $t\in [T]$. The mapping $\pi^{(t)}_{\alg}$ can be either deterministic or randomized.
   

\paragraph{Benchmarks and regret.}

We would like to design  polynomial-time online learning algorithms for offline problems that are NP-hard to solve exactly. Thus,  we use the adapted notion of  \emph{approximate regret} to quantify the performance of an online algorithm. This notion is  the  regret with respect to $\gamma$ fraction of the objective value at the best in-hindsight point. The  notion of $\gamma$-regret, which is formally defined below, is common in the literature, see, for example, \citealp{kakade2009playing}, \citealp{dudik2017oracle}, and \citealp{roughgarden2018optimal}. At a high level,  our goal is  to take a $\gamma$-approximation offline algorithm, and transform it to an online algorithm $\alg$ with a sublinear $\gamma\textrm{-regret}$.


\begin{definition}[$\gamma$-regret] \label{def:gamma:regret}Let $\sigma=\{(\zbf_t,f_t)\}_{t=1}^T$ be a sequence of strategies realized by  online learner $\alg$ and adversary $\adv$. Then, for any such $\sigma$ and  $\gamma\in(0,1)$, $\gamma\textrm{-regret}(\sigma)$ is defined as
$$
\gamma\textrm{-regret}\left(\sigma\right)\triangleq\gamma\cdot \underset{\zbf\in\constraint}{\max}~\sum_{t=1}^T f_t(\zbf)-\sum_{t=1}^Tf_t(\zbf_t)~.
$$
With a  slight abuse of the notation, we denote the worst-case expected approximate regret of $\alg$ against any (oblivious) adversary $\adv$ as follows:
$$
\gamma\textrm{-regret}\left(\alg \right)\triangleq \max_{\{f_t\}_{t=1}^T}\Big\{ \E[\gamma\textrm{-regret}\left(\sigma\right)]: \sigma=\{(\zbf_t,f_t)\}_{t=1}^T, \zbf_t\in\constraint=\textrm{$\alg$'s strategy at time $t\in[T]$}\Big\}\,,
$$
where the expectation is with respect to any randomness in $\alg$. 
\end{definition}


\subsection{Blackwell Sequential Games and Approachability}
\label{sec:blackwell}


To transform offline approximation algorithms to efficient online learning algorithms, we take advantage of \emph{Blackwell sequential games}. A Blackwell sequential game is a repeated two-player game characterized by a tuple $\blackwellsequentialgame$. In this repeated game, $\algspaceB$ and $\advspaceB$ are both compact convex sets representing the players' action spaces, and $\pbf: \algspaceB \times \advspaceB \to \mathbb{R}^d$ is a biaffine vector payoff function.\footnote{Function $\pbf(\cdot, \cdot)$ is biaffine if for any $\bm{x} \in \algspaceB$, $\pbf(\bm{x}, \cdot)$ is affine and for any $\bm{y} \in \advspaceB$, $\pbf(\cdot, \bm{y})$ is affine.} Moreover, parameter $d\in\mathbb{N}$ is known as the dimension of the Blackwell sequential game. The vector payoff function $\pbf$ is assumed to be known by both players. The game is played in $T$ rounds. Each round involves player~1 choosing an action $\bm{x}_t \in \algspaceB$ and player~2 choosing an action $\bm{y}_t \in \advspaceB$ simultaneously. Both actions may depend on the observed history $\left( (\bm{x}_1,\bm{y}_1), \cdots, (\bm{x}_{t-1}, \bm{y}_{t-1}) \right)$. This pair of actions produces the vector payoff $\pbf(\bm{x}_t, \bm{y}_t)$. The objective of player 1 is to ensure that the time-averaged payoff approaches a closed and convex target set $S \subseteq \mathbb{R}^d$,  and the objective of player 2 is to prevent this from happening.

\begin{definition}[Blackwell approachabilty]
\label{def:blackwell-approach}
In the Blackwell sequential game $\blackwellsequentialgame$, a target set $S$ is $g(T)$-approachable if there exists a player 1 strategy such that for every player 2's strategy, the resulting sequence of actions satisfies
  \begin{align*}
    d_{\infty} \left( \frac1T \sum_{t=1}^T \pbf(\bm{x}_t, \bm{y}_t), S \right) &\le g(T)\,,
  \end{align*}
  where for any vector ${\bf{w}}\in\mathbb{R}^d$ and set $S\subseteq\mathbb{R}^d$, $d_{\infty}({\bf{w}}, S) \triangleq \inf_{v \in S} \left\lVert\bf{w} - \bf{v} \right\rVert_{\infty}$ is the $\ell_{\infty}$-distance of vector $\bf{w}$ from set $S$.
\end{definition}
\medskip

In this paper, we focus on the $\ell_\infty$ norm rather than the usual $\ell_2$ norm since it is more suitable for our applications. Our bounds on the approachability term $g(T)$ will depend on the scale of the problem, and more formally on the diameter $\diameter{\pbf}$ of the payoff function $\pbf$, defined as
\begin{equation}\label{eq:diameter}
  \diameterdefinition{\pbf}\,.
\end{equation}

Ideally, player 1 aims to develop a strategy so that the term $g(T)$ in \Cref{def:blackwell-approach} converges to $0$ as $T$ converges to $+\infty$, and hence would be able to approach the target set $S$ asymptotically. However, not every closed and convex target set $S$ is approachable. To help with characterizing which sets are approachable, we additionally define the concept of \emph{response-satisfiablity}. 

\begin{definition}[Response-Satisfiable] \label{def:satisfy}
  A closed and convex target set $S$ is response-satisfiable in the Blackwell sequential game $\blackwellsequentialgame$ if for every player 2's action $\bm{y} \in \advspaceB$, there exists a player 1's action $\bm{x} \in \algspaceB$ such that the vector payoff falls into the target set, that is $\pbf(\bm{x}, \bm{y}) \in S$.
\end{definition}

Blackwell's landmark result~\citep{blackwell1956analog} is an equivalence of (asymptotic) approachability and response-satisfiablity.\footnote{There are other equivalent structural criteria for approachability similar to response-satisfiability; see \Cref{sec:apx-blackwellequivalent} in the appendix for a list of these conditions.} We extend this result in the following theorem.

\begin{restatable}{theorem}{blackwellthm}
\label{def:blackwell-thm}
  A  closed and convex target set $S$ is $O(\diameter{\pbf}\log(d)^{1/2}T^{-1/2})$-approachable in the Blackwell sequential game $\blackwellsequentialgame$ if and only if the set $S$ is response-satisfiable, where $\diameter{\pbf}$, defined in Equation  \eqref{eq:diameter}, is the $\ell_\infty$ diameter of the payoff function $\pbf$, and $d$ is the dimension of the game.
\end{restatable}

We present a detailed proof of \Cref{def:blackwell-thm} in \Cref{sec:apx-blackwell-proof} in the appendix, which is an adaptation of the original result of \cite{blackwell1956analog}. The main difference between the 
Blackwell's original result and Theorem \ref{def:blackwell-thm} is how the  distance between the average payoff and set $S$ is computed. While Blackwell uses norm 2, we apply norm infinity. To account for this difference, we use the equivalence between Blackwell approachability and online linear optimization~\citep{abernethy2011blackwell}. This equivalence allows us to 
 apply regret bounds for the latter problem that uses an arbitrary norm to find new bounds for the approachability problem. The regret bounds (on online linear optimization) can then be  obtained via using  Follow-the-Regularized-Leader~\citep{shalev2012online} or Online Mirror Descent~\citep{bubeck2015convex} algorithms.

 
 We finish this subsection by a few remarks regarding our treatment of the Blackwell approachability.

\medskip
\begin{remark}
\label{rem:polytime-blackwell}
As our goal is designing polynomial-time online learning algorithms, we further use algorithmic results in \citealp{even2009online}, and \cite{abernethy2011blackwell} due to the equivalence between Blackwell approachability and full information adversarial online linear optimization. These results provide a polynomial-time approachable online algorithm satisfying the bound in \Cref{def:blackwell-thm}, given access to a separation oracle for the closed and convex set $S$.\footnote{Given the separation oracle for convex set $S$, the running-time should be polynomial in $d$, $T$, and the number of bits required to encode $\algspaceB$. We are also considering a computational model where either the realized vector payoff is given as feedback at the end of each round, or the vector payoff function $\pbf$ can be evaluated efficiently at any given pair of actions $(\xbf,\ybf)$.} From this point on, when set $S$ is response-satisfiable, we assume access to such an online algorithm that uses a separation oracle for the convex set $S$ in a blackbox fashion.
\end{remark}
\medskip

\begin{remark}
Another upshot of the above line of research on the equivalence between Blackwell approachability and full information online linear optimization is that an algorithm for player $1$ to approach set $S$ might only have access to the realized vector payoffs $\left(\pbf(\xbf_1,\ybf_1),\ldots,\pbf(\xbf_{t-1},\ybf_{t-1})\right)$ in round $t$, rather than the entire history $\left((\xbf_1,\ybf_1),\ldots,(\xbf_{t-1},\ybf_{t-1})\right)$, and this is indeed without loss of generality for obtaining the optimal bound of \Cref{def:blackwell-thm}~\citep{abernethy2011blackwell}. We relax this assumption in our ``bandit Blackwell sequential game'', where we assume player 1 can only sometimes have access to an unbiased estimator of the realized vector payoff; see \Cref{sec:bandit-blackwell} for the definition and more details.
\end{remark}

%% file: tex/framework.tex
As stated earlier, we are interested in transforming  a $\gamma$-approximation algorithm for the offline problem~\eqref{eq:offline-optimization} to an online learning algorithm, so that the worst-case $\gamma\textrm{-regret}$ is sublinear in the number of rounds $T$. We consider a general class of algorithms for obtaining such an approximation guarantee, named \emph{Iterative Greedy (IG)} algorithms. In an algorithm in this class, roughly speaking, a sequence of locally optimal decisions with respect to a specific metric (which we elaborate on more later) leads to picking the final point. This point then provably provides an approximation guarantee with respect to the global optimal solution of problem~\eqref{eq:offline-optimization}.

  Formally, consider the following abstract skeleton. Suppose that  we have $N$ \emph{subproblems} indexed by $i\in [N]$. The algorithm starts from an initial feasible point $\zi{0}\in \constraint$. It then goes over the subproblems in the increasing order of their indices. The goal of each subproblem $i$ is to return a new feasible point $\zi{i}\in\constraint$ given the output of the previous subproblem, i.e., $\zi{i-1}$. The algorithm finishes by returning the point $\zi{N}$. Now, each subproblem $i$ performs two steps: 
  
\begin{enumerate}
    \item \underline{Local optimization}:  We associate a space of \emph{update parameters} $\paramspace\subseteq \mathbb{R}^{\paramdimension}$ to each subproblem. Given the previous point $\zi{i-1}$ and the objective function $f$, the goal of this step is to find a \emph{locally optimal} update parameter $\param^{(i)}\in\pspace$ that satisfies:
     $$\pay(\param^{(i)}, \zbf^{(i-1)}, f) \geq \mathbf{0}\,,$$
     where $\pay:\paramspace\times \domain\times\funcspace\rightarrow \mathbb{R}^{\payoffdimension}$ denotes the \emph{parameter vector payoff function}. 
    
    \item \underline{Local update}: Given the update parameter $\param^{(i)}$ and $\zi{i-1}$, this step returns the next point $$\zi{i}=\update(\param^{(i)},\zi{i-1})\in\constraint\,.$$
    Notably, we allow $\update:\paramspace\times\domain\rightarrow\Delta(\constraint)$ to incorporate randomness, and therefore $\zi{i}$ can be potentially a randomized point.
\end{enumerate}
The above procedure is summarized in \Cref{alg:offline-meta}.
\begin{remark}
To simplify the notation, we only consider symmetric subproblems in this section, i.e., all of the subproblems have the same update parameter spaces, local optimization steps, etc. In some of our applications in \Cref{sec:applications}, we need slightly different subproblems for different $i=1,\ldots,N$. Our method directly extends to that case by having index-dependent subproblems.
\end{remark}
\medskip
\begin{algorithm}
\caption{$\offlinemeta$}
\label{alg:offline-meta}
\textbf{Meta Input:} 
  Feasible region $\constraint$, function space $\funcspace$, defined over  domain $\domain$,  parameter space $\paramspace \subseteq \mathbb{R}^{\paramdimension}$~, and parameter vector payoff function $\pay: \paramspace \times \domain \times \mathcal{F} \to \mathbb{R}^{\payoffdimension}$.\\
\textbf{Input:} function $f \in \mathcal{F}$.\\
\textbf{Output:} feasible point $\zbf \in \constraint$. \\
\vspace{2mm}
Initialize $\zbf^{(0)}\in\constraint$;
\For{\textrm{subproblem} $i=1$ to $N$}{
 Choose update parameter $\param^{(i)}$ so that $\pay(\param^{(i)}, \zbf^{(i-1)}, f) \geq \mathbf{0}$\;
  Set $\zi{i}\leftarrow \update(\param^{(i)},\zbf^{(i-1)})$;
}
Return the final point $\zbf \leftarrow \zbf^{(N)}$.
\end{algorithm}

\begin{example}
\label{example:running-1}
As a simple running example, consider the problem of maximizing a monotone submodular set function $f:2^{[n]}\rightarrow [0,1]$ subject to the cardinality constraint $k$. A set function $f:2^{[n]}\rightarrow[0,1]$ is submodular if for all $S,T\subseteq [n]$,
$
    f(S\cup T)+f(S\cap T)\leq f(S)+f(T).
$ Maximizing a monotone submodular set function $f:2^{[n]}\rightarrow [0,1]$ subject to the cardinality constraint $k$ is an NP-hard optimization problem which admits  the classic $(1-\tfrac{1}{e}$)-approximation greedy algorithm \citep{nemhauser1978analysis}. {This algorithm starts from the empty set and picks elements greedily based on their marginal value to the current set, where the marginal value of adding element $j$ to set $S$ is $f(S\cup \{j\}) -f(S)$. That is, in each stage of this algorithm and given the  chosen set so far $S$, an element $j^*\in \arg\max_{j\in [n]} f(S\cup \{j\}) -f(S)$ with the highest marginal value is added to $S$.}
This problem is an example of problem~\eqref{eq:offline-optimization} where $\domain=\{0,1\}^n$, $\constraint=\{\zbf\in\{0,1\}^n:\zbf\cdot\mathbf{1}_n\leq k\}$ and $\funcspace$ is the space of all monotone submodular set functions. Here, $\mathbf{1}_n$ is the all-ones vector with size $n$. The greedy algorithm is an instance of Algorithm~\ref{alg:offline-meta} with $\paramspace=\Delta([n])$, which  is the set of all possible probability distributions over $n$ elements, and   $N=k$  subproblems, one for each iteration of the greedy algorithm. To describe each {subproblem}, for $\param\in\paramspace, \zbf\in\domain$, and $f\in\funcspace$,
$$
\forall j\in [n]:~~~\left[\pay(\param,\zbf,f)\right]_j=\param^T\ybf-[\ybf]_j\,,
$$
where $[\cdot]_j$ denotes the $j^{\textrm{th}}$ coordinate value of a vector and $\ybf\triangleq \left[f(\zbf\cup\{j\})-f(\zbf)\right]_{j\in[n]}$ is the marginal objective value  of adding element $j$ to $\zbf$. 
Moreover, $\update(\param,\zbf)$ samples an element $i^*\sim \param$, where $\param\in\Delta([n])$ is a probability distribution over $n$ elements, and returns $\zbf\cup\{i^*\}$. Note that $\pay(\param,\zbf,f)\geq \mathbf{0}$ guarantees $\param$ to only have positive mass on 
 elements with maximum marginal value with respect to the point $\zbf$.
\end{example}

\subsection{Approximation with Robustness to Local Errors}  

We focus on IG algorithms that ({i}) provide a worst-case multiplicative approximation guarantee for problem~\eqref{eq:offline-optimization}, and ({ii}) have a local optimization step that is robust to small errors, i.e., if we replace the locally optimal decisions with almost locally optimal ones, the final point still remains to be approximately optimal (with the same approximation factor), but up to a small additive error. The following definition formalizes this robustness notion. 

\begin{definition}[$(\gamma,\delta)$-robust approximation]
\label{def:robust-approx}
An instance of Algorithm~\ref{alg:offline-meta} is a $(\gamma,\delta)$-robust approximation algorithm for $\gamma\in(0,1)$ and $\delta>0$, if it satisfies the following properties:
\begin{enumerate}
    \item Algorithm~\ref{alg:offline-meta} is a $\gamma$-approximation offline algorithm as in Definition~\ref{def:offline},
    \item Supposed that  we replace $\param^{(i)}$ with $\tilde{\param}^{(i)}$ for every $\textrm{subproblem}~i=1,\ldots,N$. Then, if
    $$\forall j\in[\payoffdimension]:~~\left[\pay(\tilde\param^{(i)}, \zbf^{(i-1)}, f)\right]_j+\epsilon\geq 0\,, $$
    then we should have:
    $$\forall \hat\zbf\in\constraint:~~~\expect{f(\zbf)}\geq \gamma\cdot f(\hat\zbf)-\delta N\epsilon\,,$$
     where $\epsilon>0$ and $[\cdot]_j$ denotes the $j^{\textrm{th}}$ coordinate value of a vector.
\end{enumerate}
\end{definition}


For our purpose, we actually need a stronger version of this robustness property.  This  property essentially concerns multiple runs of the offline algorithm on a group of functions in $\funcspace$, i.e., $\{f_t\}_{t\in [T]}$, producing a sequence of feasible points $\zbf_t\in\constraint$ for $t\in [T]$, and then guarantees a robust approximation for the summation function, i.e., $\sum_{t\in [T]}f_t(z)$, against errors that are small on-average over these runs by the sequence $\{\zbf_t\}_{t=1}^T$. This property is satisfied in all of the applications that motivate our work, in particular in various set and continuous submodular maximization problems we study in Sections \ref{subsec:usm} and \ref{subsec:DR-SM} in the appendix, and in both reserve price optimization and product ranking problems in \Cref{sec:applications}.\medskip

\begin{definition}[Extended $(\gamma,\delta)$-robust approximation]
\label{def:extended-robust-approx}
An instance of Algorithm~\ref{alg:offline-meta} is an extended $(\gamma,\delta)$-robust approximation algorithm for $\gamma\in(0,1)$ and $\delta>0$, if for any sequence of functions $f_1,f_2,\ldots f_T\in\funcspace$ the following property is satisfied: 
\begin{itemize}
    \item {Consider a hypothetical variant  of Algorithm ~\ref{alg:offline-meta} that in every round $t$, instead of choosing the update parameter $\param^{(i)}_t$ for each subproblem $i\in [N]$ such that $\pay(\param^{(i)}_t, \zbf^{(i-1)}_t, f) \geq \mathbf{0}$, it chooses a ``noisy''  update parameter $\tilde\param^{(i)}_t$ such that the sequence of update parameters  $(\tilde\param^{(i)}_t)_{t\in [T]}$ satisfies the following condition 
    $$\forall j\in [\payoffdimension]:~~\left[\sum_{t=1}^T\pay(\tilde\param^{(i)}_t, \zbf_t^{(i-1)}, f_t)\right]_j+h(T)\geq 0\,. $$
    Then, we  have
    $$\forall \hat\zbf\in\constraint:~~~\sum_{t=1}^T\expect{f_t(\zbf_t)}\geq \gamma\cdot\sum_{t=1}^T f_t(\hat\zbf)-\delta Nh(T)~.$$
    Here,  $\zbf_t^{(i)}$ is the output of subproblem $i\in[N]$ for run $t\in[T]$ by applying the hypothetical variant of Algorithm \ref{alg:offline-meta} on $f_t$},  $h(\cdot):\mathbb{N}\rightarrow\mathbb{R}_+$, and $[\cdot]_j$ denotes the $j^{\textrm{th}}$ coordinate value of a vector.
\end{itemize}


\end{definition}
\medskip

When there is only one run of the function (i.e., $T=1$), the  extended $(\gamma,\delta)$-robust approximation guarantee boils down to the weaker $(\gamma,\delta)$-robust approximation guarantee in \Cref{def:robust-approx}.  We finish this section by revisiting our running example and demonstrating the (extended) robust approximation property in this example. 

\setcounter{example}{0}
\begin{example}[continued]
By digging deeper in the original analysis of the greedy algorithm~\citep{nemhauser1978analysis}, we  show that the greedy algorithm  satisfies the extended $(\gamma,\delta)$-robust approximation property  for $\gamma=1-\tfrac{1}{e}$ and $\delta=1$. 
\begin{proof}
Suppose that  $\zbf^*=\{a_1,\ldots,a_k\}$ is the optimal solution of the offline problem; that is, $\zbf^*=\argmax_{\zbf\in\{0,1\}^n:\zbf\cdot\mathbf{1}_n\leq k}\sum_{t=1}^T f_t(\zbf)$. ({We use binary indicator vectors and sets interchangably in this paper.}) Further, let $\zbf_t^{(i)}$ be the solution returned by the $i^{\textrm{th}}$ subproblem of the greedy algorithm when the objective function is $f_t$.  Then, for every $i\in [k]$,
\begin{align*}
    \sum_{t=1}^{T}f_t(\zbf^*)-\sum_{t=1}^T f_t(\zbf_t^{(i-1)})&\overset{(1)}{\le} \sum_{t=1}^{T}f_t(\zbf^* \cup \zbf_t^{(i-1)})-\sum_{t=1}^T f_t(\zbf_t^{(i-1)})\\
    &=\sum_{t=1}^{T}\sum_{j=1}^k \left(f_t(\zbf_t^{(i-1)}\cup\{a_1,\ldots,a_j\} )- f_t(\zbf_t^{(i-1)}\cup\{a_1,\ldots,a_{j-1}\})\right)\\
    &\overset{(2)}{\leq} \sum_{j=1}^k \sum_{t=1}^{T}\left(f_t(\zbf_t^{(i-1)}\cup\{a_j\} )- f_t(\zbf_t^{(i-1)})\right)\\
    &\overset{(3)}{=} \sum_{j=1}^k \sum_{t=1}^{T}\left(\langle{{\tilde\param}_t^{(i)}},\ybf_{t}^{(i-1)}\rangle-\left[\pay(\tilde\param_t^{(i)},\zbf_t^{(i-1)},f_t)\right]_{a_j}\right)\\
    &{=} \sum_{j=1}^k \sum_{t=1}^{T}\left(\sum_{j=1}^{n}[{\tilde\param}_t^{(i)}]_j\Big(f_t(\zbf_t^{(i-1)}\cup\{j\} )- f_t(\zbf_t^{(i-1)})\Big)-\left[\pay(\tilde\param_t^{(i)},\zbf_t^{(i-1)},f_t)\right]_{a_j}\right)\\
    &{=} k\cdot\sum_{t=1}^{T}\expect{f_t(\zbf_t^{(i)})-f_t(\zbf_t^{(i-1)})}-\sum_{j=1}^k \sum_{t=1}^{T}\left[\pay(\tilde\param_t^{(i)},\zbf_t^{(i-1)},f_t)\right]_{a_j}\\
    &\overset{(4)}{\leq} k\cdot\sum_{t=1}^{T}\expect{f_t(\zbf_t^{(i)})-f_t(\zbf_t^{(i-1)})}+kh(T)~,
\end{align*} 
where $\ybf_t^{(i)}\triangleq \left[f_t(\zbf_t^{(i)}\cup\{j\} )- f_t(\zbf_t^{(i)})\right]_{j\in [n]}$. In the above chain of inequalities, inequality~(1) holds because function $f_t$ is monotone, inequality~(2) holds due to submodularity of functions $\{f_t\}_{t=1}^T$, equality~(3) holds because of the definition of the payoff vector in Example~\ref{example:running-1}, and inequality~(4) holds because of the condition in Definition~\ref{def:extended-robust-approx}. By rearranging the terms and taking expectations, we have: 
$$
\sum_{t=1}^T\expect{f_t(\zbf^*)-f_t(\zbf^{(i)}_t)}\leq (1-\tfrac{1}{k})\sum_{t=1}^T\expect{f_t(\zbf^*)-f_t(\zbf^{(i-1)}_t)}+h(T)\,.
$$
By recursing  the above inequality for $i=1,\ldots,k$, and rearranging the terms, we finally have: 
$$
\sum_{t=1}^T\expect{f_t(\zbf_t)}\geq (1-(1-\tfrac{1}{k})^k)\sum_{t=1}^{T}f_t(\zbf^*)-h(T)\sum_{i=1}^{k}(1-\tfrac{1}{k})^{i-1}\geq (1-\tfrac{1}{e})\sum_{t=1}^{T}f_t(\zbf^*)-kh(T)~.
$$
\[\myqed\]
\end{proof}
\end{example}

Not all greedy algorithms have robust guarantees. Example~\ref{example:non-robust} of Section~\ref{sec:appendix-framework} in the appendix shows why, e.g., Dijkstra's algorithm for the shortest path problem, is not robust to local errors.

%% file: tex/fullinfo.tex
In this section, we show how to transform an offline IG  algorithm (\Cref{alg:offline-meta}) to an online learning algorithm with a small approximate regret whenever it (i) is an extended robust approximation algorithm (Definition~\ref{def:extended-robust-approx}), and (ii)  satisfies an extra condition that we call \emph{Blackwell reduciblity}. We first introduce this condition. Then, with the help of the Blackwell approachability  (Theorem~\ref{def:blackwell-thm}), we propose a meta full information online learning algorithm as our offline-to-online transformation.

\subsection{Blackwell Reduciblity}
The crux of our technique to transform an offline IG algorithm to an online learning algorithm is the possibility of reducing the local  optimization step of \Cref{alg:offline-meta} to an approachable instance of the Blackwell sequential game as in Section~\ref{sec:blackwell}. 

\begin{definition}[Blackwell reduciblity]
\label{def:blackwell-reducible}
An instance $\offlinemeta$ of Algorithm~\ref{alg:offline-meta} is Blackwell reducible if there exists an instance $\left(\algspaceB,\advspaceB,\pbf\right)$  of the Blackwell sequential game (with a biaffine vector payoff function $\pbf$) and a mapping $\advfunB:\domain \times \funcspace \to \advspaceB$ called \emph{synthetic Blackwell adversary function}, such that:

\begin{enumerate}
\item [1.] The player 1's action space $\mathcal X$ is equal to the parameter space $\paramspace$ in Algorithm~\ref{alg:offline-meta}; i.e.,   $\algspaceB=\paramspace$, and for any $\param\in\paramspace,\zbf\in\domain,f\in\funcspace$, we have $\pay(\param,\zbf,f)=\pbf\left(\param,\advfunB(\zbf,f)\right)$.

 \item [2.] The set $S\triangleq\{\ubf\in\mathbb{R}^{\payoffdimension}:[\ubf]_j\geq 0, j\in [\payoffdimension]\}$ is response-satisfiable (Definition~\ref{def:satisfy}).

 
\end{enumerate}

\end{definition}
\medskip

\setcounter{example}{0}
\begin{example}[continued]
The greedy algorithm of \cite{nemhauser1978analysis} is Blackwell reducible. Consider an instance $\left(\algspaceB,\advspaceB,\pbf\right)$ of Blackwell where $\algspaceB=\paramspace=\Delta([n])$ and $\advspaceB=[0,1]^n$. The synthetic Blackwell adversary function is $\advfunB(\zbf,f) = \left[f(\zbf\cup\{j\})-f(\zbf)\right]_{j\in [n]}$, and the biaffine Blackwell vector payoff function is $\pbf(\param,\bm{y}) = \param^T\bm{y}\mathbf{1}_n - \bm{y}$.\footnote{Note that  $\advspaceB=[0,1]^n$ because  $f:2^{[n]}\rightarrow [0,1]$  is monotone non-decreasing.} Recall that   $\mathbf{1}_n$ is all-ones $n$-dimensional vector. Furthermore, set $\targetsetB$ is response-satisfiable because for every player 2's action $\bm{y}\in\advspaceB,$ playing $\param = e_{j^*}$ with $j^*=\underset{j\in [n]}{\textrm{argmax}}~\bm{y}_j$ implies that $\pbf(\param,\ybf)\geq\bm{0}.$
\end{example}

\subsection{Offline-to-Online Transformation with Full Information Feedback}
If the offline algorithm (\Cref{alg:offline-meta}) is Blackwell reducible, then one can think of the following approach to transform it into an online learning algorithm: associate an instance of the Blackwell sequential game to each subproblem $i$ following the Blackwell reducibility, and then running $N$ parallel online approachable algorithms for these Blackwell instances to find a sequence of assignments of the update parameter of each subproblem $i$ over time. We further need to show how to synchronize these parallel runs through a proper communication between them, so as to construct a sequence of feasible solutions $\zbf_1,\ldots,\zbf_T$ guaranteeing a small approximate regret.

 \subsubsection{Overview of the Algorithm}
 
   Recall that our goal in the offline problem is to solve the optimization problem $\max_{\zbf\in  \constraint } f(\zbf)$, where $f\in \funcspace$. The offline problem admits a polynomial time IG $\gamma$-approximation algorithm, $\offlinemeta$, presented in  Algorithm~\ref{alg:offline-meta}. This algorithm solves $N$ subproblems sequentially, building the solution step by step. In step/subproblem $i$ of this algorithm, we first update parameters $\boldsymbol \theta^{(i)}\in \paramspace \subseteq \mathbb{R}^{\paramdimension}$ using the previous point $\zbf^{(i-1)}$, and then return the next point $\zbf^{(i)}$ to feed to the next subproblem. The algorithm finishes by returning the final point $\zbf^{(N)}$.
 
 As stated earlier, we assume that the offline problem is Blackwell reducible; that is, we can define the Blackwell instance  $\left(\algspaceB,\advspaceB,\pbf\right)$ and synthetic Blackwell adversary function $\advfunB:\domain\times\funcspace\rightarrow \advspaceB$ that satisfy the conditions in Definition~\ref{def:blackwell-reducible}. 
 Although this definition might seem technical, verifying it for many offline algorithms is indeed straightforward; see Sections \ref{sec:applications}, \ref{subsec:usm}, and \ref{subsec:DR-SM}. 
 
 For the online version of the above offline algorithm, the meta input is feasible region $\constraint$, function space $\funcspace$, which is defined over domain $\domain$, and  parameter space $\paramspace \subseteq \mathbb{R}^{\paramdimension}$. We further consider having access to an online Blackwell algorithm $\BlackwellAlg$, player 1's strategy in the above Blackwell sequential game
$\left(\algspaceB,\advspaceB,\pbf\right)$, where such algorithm (i) ensures that the distance between the average vector payoff $\frac1T \sum_{t=1}^T \pbf(\bm{x}_t, \bm{y}_t)$ and set $S$ goes to zero with rate  $g(T)=O\left(\diameter{\pbf}\sqrt{\frac{\log(d)}{T}}\right)$ against any adversarial player 2's strategy (see Theorem \ref{def:blackwell-thm}), and (ii) can be implemented in polynomial time having access to a separation oracle for the convex set $S$. As stated earlier, the existence of such an algorithm follows from the work of \citealp{even2009online} and \cite{abernethy2011blackwell}; see \Cref{rem:polytime-blackwell}. We consider $N$ parallel copies of this algorithm, one for each subproblem $i\in[N]$. It is also important to note that in most of our applications, set $S$ is the positive orthant, for which a polynomial time separation oracle exists.\footnote{{For the application of maximizing monotone Strong-DR submodular functions over downward closed bounded convex sets, presented in Section \ref{subsec:DR-SM}, the target set $S$ is not a positive orthant, yet it is convex and admits a polynomial separation oracle.}} Our algorithm that takes advantage of  $N$ parallel copies of the online Blackwell algorithm  is summarized in \Cref{alg:full-info-backbone}.


Let $\BlackwellAlg^{(i)}$ be the copy of the above online Blackwell algorithm associated to subproblem $i\in [N]$. This copy handles the local optimization step of subproblem $i$ in the $\offlinemeta$ in every round $t\in [T]$ without knowing function $f_t$. Consider the decision-making process of this online algorithm in round $t$. The inputs prior to this round are all the update parameters of the subproblem $i$ in the first $t-1$ rounds, i.e., $\param^{(i)}_1,\ldots,\param^{(i)}_{t-1}$, and the realized vector payoffs of the first $t-1$ rounds against player 2 in the Blackwell sequential game associated to subproblem $i$, i.e., $\pbf(\param^{(i)}_1,\yit{i}{1}),\ldots,\pbf(\param^{(i)}_{t-1},\yit{i}{{t-1}})$. We consider a particular player 2 for this Blackwell sequential game. More explicitly, the synthetic adversary function $\advfunB$, which is part of our reduction, plays the role of player 2 in any round $t$, i.e., $\yit{i}{t}=\advfunB(\zi{i-1}_t,f_t)$. Given the input prior to time $t$, $\BlackwellAlg^{(i)}$ returns the new update parameter $\param^{(i)}_t$.\footnote{{ Note that the adversary's action  in round $t$  for subproblem $i$ is $\advfunB(z_t^{(i-1)},f_t)$, where $z_t^{(i-1)}$ is the decision made by the first $i-1$ subproblems in round $t$.   Here, both $f_t$ and $z_t^{(i-1)}$ can be viewed as the signals used by the adversary in round $t$ to determine his strategy.  Using these signals in the adversary's strategy (i.e.,  $\advfunB(z_t^{(i-1)},f_t)$) is allowed. This is because conditioned on the history of plays (and past signals) in previous rounds (i.e., $f_{\tau}$ and $z_{\tau}^{(i-1)}$ for any $\tau< t$), as well as the feedback that $\BlackwellAlg^{(i)}$  receives in this round (i.e., $\textsc{Payoff}(\theta_t^{i},z_{t}^{(i-1)}, f_t)$), both $f_t$ and $z_t^{(i-1)}$ will be independent of $\BlackwellAlg^{(i)}$'s action in future rounds $t'>t$.}}

After the online Blackwell algorithm $\BlackwellAlg^{(i)}$ returns the update parameter $\param_t^{(i)}$, we return the point $\zbf_t^{(i)}$ by calling the  $\update$ function in the offline algorithm, i.e., we set  $\zi{i}_t$ to $ \update(\param^{(i)}_t, \zi{i-1}_t)$. Observe that 
the point returned by the subproblem $i$, i.e., $\zbf_t^{(i)}$,  depends on the point returned by the previous subproblem $\zbf_t^{(i-1)}$. This highlights that while each online Blackwell algorithm is responsible for one subproblem, they communicate with each other to build the final solution, where this communication is structured by the offline algorithm through the $\update$ function. After obtaining the point $\zbf_t^{(i)}$, we move to subproblem $i+1$.

Finally note that simulating the actions of our particular player 2 to determine the realized vector payoffs of each round, and computing/sending this feedback at the end of each round to $\BlackwellAlg^{(i)}$ (as player 1) in a computationally efficient manner, require the following:
\begin{itemize}
    \item Knowing the point $\zbf_t^{(i-1)}$ picked by subproblem $i-1$ at time $t$: This is possible as we go over our subproblems in the order $i=1,\ldots,N$ in each round $t$.
    \item  Knowing the function $f_t$: This is possible because here we study the full information feedback structure, where under this  structure we have access to $f_t$ after we choose point $\zbf_t=\zbf_t^{(N)}$.
    \item Being able to compute the realized vector payoff $\pbf\left(\param^{(i)}_{t},\advfunB(\zi{i-1}_t,f_t)\right)$ efficiently given $\param^{(i)}_{t}$, $f_t$, and $\zi{i-1}_t$. This is possible as this quantity is equal to $\pay(\param^{(i)}_{t},\zi{i-1}_t,f_t)$, which can be evaluated in polynomial time as $\offlinemeta$ is a polynomial time algorithm.
\end{itemize}





\begin{algorithm}[ht]
\SetAlgoLined
\textbf{Meta Input:} Feasible region $\constraint$, function space $\funcspace$, defined over domain $\domain$, and  parameter space $\paramspace \subseteq \mathbb{R}^{\paramdimension}$. \\
\vspace{-2mm}

\textbf{Offline algorithm and reduction gadgets:} An instance $\offlinemeta$ of Algorithm~\ref{alg:offline-meta}, the Blackwell instance $\left(\algspaceB,\advspaceB,\pbf\right)$ and synthetic Blackwell adversary function $\advfunB:\domain\times\funcspace\rightarrow \advspaceB$ as this offline algorithm is Blackwell reducible (Definition~\ref{def:blackwell-reducible})\,.\\
\vspace{-2mm}

\textbf{Input:} Number of rounds $T$; access to a Blackwell online algorithm $\BlackwellAlg$\,.

\textbf{Output:} Points $\zbf_1,\zbf_2,\ldots,\zbf_T\in\constraint$\,.

\vspace{2mm}
Initialize $N$ parallel instances $\{\BlackwellAlg^{(i)}\}_{i=1}^{N}$ of the online algorithm $\BlackwellAlg$\,;

\For{\textrm{round} $t=1$ to $T$}{
Initialize $\zi{0}_t\in\constraint$\,;

\For{\textrm{subproblem} $i=1$ to $N$}{
Choose update parameter $\param^{(i)}_t$ by querying  online algorithm $\BlackwellAlg^{(i)}$ given the update parameters and vector payoffs prior to round $t$ in the Blackwell sequential game of subproblem $i$, that is, $\param^{(i)}_t\leftarrow\BlackwellAlg^{(i)}\left(\param^{(i)}_1,\ldots,\param^{(i)}_{t-1},\pbf(\param^{(i)}_{1},\yit{i}{1}),\ldots,\pbf(\param^{(i)}_{t-1},\yit{i}{t-1})\right)$ \;

Set $\zi{i}_t \leftarrow \update(\param^{(i)}_t, \zi{i-1}_t) \in \constraint$\,;

}
Play the final point $\zbf_t\leftarrow\zi{N}_t$\,;

\vspace{2mm}
< \emph{Full information feedback: adversary reveals function $f_t\in\funcspace$} >\,;
\vspace{2mm}

\For{$i=1$ to $N$}{
  Give feedback $\pbf(\param^{(i)}_{t},\yit{i}{t})\leftarrow\pay(\param^{(i)}_{t},\zi{i-1}_t,f_t)$ to the Blackwell Algorithm $\BlackwellAlg^{(i)}$ (as the vector payoff of round $t$ against player 2)\,;
// \small{\emph{Note that $\yit{i}{t}=\advfunB(\zi{i-1}_t,f_t)$ for player 2 implicitly, although we \underline{do not} need to evaluate $\advfunB$ to compute this action explicitly.}}
   
}

}
\caption{Full-information Online Learning Meta-algorithm ($\onlinemetatext$)}
\label{alg:full-info-backbone}
\end{algorithm}

\subsubsection{Regret Analysis}
The following theorem, which  bounds the regret of our algorithm, is the main result of this section.

\begin{theorem}[Full information offline-to-online transformation]
  \label{thm:full-info-online-meta}
  Suppose that an instance of the algorithm $\offlinemeta$ for the offline problem~\eqref{eq:offline-optimization} satisfies the following properties:
  \begin{itemize}[noitemsep]
    \item It is an extended $(\gamma,\delta)$-robust approximation for $\gamma\in(0,1)$ and $\delta\in\mathbb{R}_+$, as in Definition~\ref{def:extended-robust-approx}.
    \item It is Blackwell reducible, that is, we can define the Blackwell sequential game $\left(\algspaceB,\advspaceB,\pbf\right)$ and synthetic Blackwell adversary function $\advfunB:\domain\times\funcspace\rightarrow \advspaceB$ that satisfy the conditions in Definition~\ref{def:blackwell-reducible}.
  \end{itemize}
  Consider the full-information adversarial online learning version of the problem~\eqref{eq:offline-optimization}, and let $\BlackwellAlg$ be a polynomial time Blackwell algorithm for $\left(\algspaceB,\advspaceB,\pbf\right)$ as in \Cref{rem:polytime-blackwell}.
  Then, for this online problem, $\onlinemeta$ runs in polynomial time and  satisfies the following $\gamma$-regret bound:
  $$
  \gamma\textrm{-regret}\left(\onlinemeta \right)\leq O\left(\diameter{\pbf} N\delta\sqrt{\log(\payoffdimension)T}\right)\,,
  $$
  where $N$ is the number of subproblems, $\payoffdimension$ is the dimension of vector payoffs, and $\diameter{\pbf}$, defined in Equation \eqref{eq:diameter}, is the $\ell_\infty$-diameter of the vector payoff space.
\end{theorem}

 
\proof{\emph{Proof of \Cref{thm:full-info-online-meta}.}} Consider a subproblem $i\in[N]$. Let $S$ be the $\payoffdimension$-dimensional positive orthant; see the Blackwell reducibility definition and its associated approachable set $S$  in Definition~\ref{def:blackwell-reducible}. Because $S$ is response-satisfiable and projection onto $S$ can be done in polynomial-time, there exists a polynomial-time online algorithm $\BlackwellAlg$ (with $N$ parallel copies $\{\BlackwellAlg^{(i)}\}_{i=1}^N$) that guarantees Blackwell approachability  for the Blackwell instance corresponding to subproblem $i$ with $g(T)=O\left(\diameter{\pbf}\sqrt{\frac{\log(\payoffdimension)}{T}}\right)$, based on \Cref{def:blackwell-thm}. Therefore, we have: 
\begin{equation*}
    d_\infty\left(\frac{1}{T}\sum_{t=1}^T \pbf\left(\param_t^{(i)},\advfunB(\zbf_t^{(i-1)},f_t)\right),S\right)=d_\infty\left(\frac{1}{T}\sum_{t=1}^T \pbf\left(\param_t^{(i)},\mathbf{y}^{(i)}_t\right),S\right)\leq g(T)\,.
\end{equation*}
Because the target set $S$ is the positive orthant, we have 
\begin{align*}
d_\infty\left(\frac{1}{T}\sum_{t=1}^T \pbf\left(\param_t^{(i)},\advfunB(\zbf_t^{i-1},f_t)\right),S\right) \leq g(T)
\Longleftrightarrow \forall j:\left[\sum_{t=1}^T \pbf\left(\param_t^{(i)},\advfunB(\zbf_t^{i-1},f_t)\right)\right]_j \geq -T g(T)
\end{align*}
Because of Blackwell reduciblity, $
\pay\left(\param_t^{(i)},\zbf_t^{(i-1)},f_t\right)=\pbf\left(\param_t^{(i)},\advfunB(\zbf_t^{(i-1)},f_t)\right)$. Therefore,
\begin{equation}
\label{eq:regret-robust}
\forall j\in[\payoffdimension]:~~~\left[\sum_{t=1}^T\pay\left(\param_t^{(i)},\zbf_t^{(i-1)},f_t\right)\right]_j\geq -Tg(T)\,.
\end{equation}
Finally, because Algorithm~\ref{alg:offline-meta} is an extended $(\gamma,\delta)$-robust approximation (see Definition  \ref{def:extended-robust-approx}), from Equation \eqref{eq:regret-robust}, we have:
 $$\sum_{t=1}^T\expect{f_t(\zbf_t)}\geq \gamma\cdot\sum_{t=1}^T f_t(\zbf^*)-\delta NT g(T)= \gamma\cdot\sum_{t=1}^T f_t(\zbf^*)-O\left(\delta N \diameter{\pbf} \sqrt{\log(\payoffdimension)T}\right)\,,$$
which finishes the proof. Here,   $\zbf^*$ is the optimal in-hindsight feasible solution, i.e., $\zbf^*=\underset{\zbf\in\constraint}{\textrm{argmax}}\sum_{t=1}^T f_t(\zbf)\,.$ 
\[\myqed\]
\endproof

We finish this section by reviewing our running example (Example~\ref{example:running-1}) and mentioning the regret bound we get as a direct corollary of Theorem~\ref{thm:full-info-online-meta}. 
\setcounter{example}{0}
\begin{example}[continued]
The greedy algorithm in \cite{nemhauser1978analysis} is an extended $(1-\tfrac{1}{e},1)$-robust approximation algorithm and Blackwell reducible. It has $N=k$ subproblems, the $\ell_\infty$ diameter of the payoff space is $D=1$, and the dimension of vector payoffs is $d=n$. Therefore, by invoking Algorithm~\ref{alg:full-info-backbone} given any Blackwell  algorithm satisfying the approachability bound in Theorem~\ref{def:blackwell-thm}, we obtain the following bound:
$$
  \left(1-\dfrac{1}{e}\right)\textrm{-Regret}\left(\textrm{Algorithm~\ref{alg:full-info-backbone}} \right)\leq O(k\sqrt{\log(n)T})\,,
  $$
  which exactly matches the bound known in \cite{streeter2009online} for the same problem.
\end{example}

%% file: tex/bandit.tex
So far, we presented a framework to transform an offline iterative greedy algorithm to its online counterpart under the full information feedback structure. While the full information setting provides the theoretical foundations for the rest of our results, from an application point of view, it is less motivated. In almost all applications of our framework in revenue management and online decision making (e.g., product ranking problem and reserve price optimization), assuming the learner has full information feedback is rather a strong assumption.

In this section, we seek to relax this assumption, and try to understand if our framework can be extended to the more challenging bandit feedback structure setting. Under the bandit feedback structure, at the end of each round $t$, the learner faces an additional challenge: he only has access to $f_t(\zbf_t)$, rather than the entire function $f_t$ like in the full information setting. Such a feedback structure prevents the online Blackwell algorithms $\BlackwellAlg{}^{(i)}$ to receive the feedback they require.

To overcome this challenge, we first consider a stylized bandit variation of the sequential Blackwell game. We characterize a new notion of approachability that we call \emph{bandit Blackwell approachability} and provide an algorithm achieving the information-theoretic tight approachability bound for this problem.  This algorithm uses an algorithm for the full information version of the Blackwell sequential game in a blackbox fashion.

We then introduce the extra ingredient that is needed for our bandit transformation, which is the possibility of creating an unbiased estimator for the vector payoff of the Blackwell games associated with different subproblems. Putting all these pieces together, we propose a bandit online learning algorithm with the help of our bandit Blackwell approachability. We highlight that this approach essentially uses the unbiased estimators to obtain bandit-style feedback for the online learning problems of each subproblem, leading to an efficient overall bandit learning algorithm with a sublinear $\gamma$-regret.

\subsection{Bandit Blackwell Sequential Games and Approachability}
\label{sec:bandit-blackwell}
In the bandit online learning version of problem~\eqref{eq:offline-optimization}, an online algorithm can only see the value of the function at the particular point that is picked in that round. Therefore, in our transformation, multiple online Blackwell algorithms  compete over a single piece of information in order to estimate the vector payoffs, where  estimating the vector payoff of a  Blackwell algorithm    can  be typically done by taking a costly ``exploration'' move, tailored to that algorithm.

With the goal of properly modeling this paradigm at a lower level, we propose the notion of a \emph{bandit Blackwell sequential game}, characterized by the extended tuple $\banditblackwellsequentialgame$. In this variant, player 1 makes an additional decision in each round: whether to \emph{explore} or not. Only if player 1 chooses to explore in round $t$, do they receive the unbiased estimator $\hat{\pbf}(\bm{x}_t, \bm{y}_t)$ whose expectation is the vector payoff for that round $\pbf(\bm{x}_t, \bm{y}_t)$. \schange{However, player 1 is punished by an additive cost $D(\pbf)$. } If player 1 refrains from exploration, they neither receive any feedback nor any punishment. Player 1's new goal is to minimize the distance from the time-averaged payoff to the target set $S$ plus their time-averaged exploration penalty.

\begin{definition}[Bandit Blackwell approachability]
\label{def:bandit-blackwell-approach}
  A closed convex target set $S$ is $g(T)$-bandit-approachable in the bandit Blackwell sequential game $\banditblackwellsequentialgame$ if there exists a bandit player 's strategy such that for every player 2's strategy, the resulting sequence of actions satisfy
  \begin{align*}
    \standarddistance{\frac1T \sum_{t=1}^T \pbf(\bm{x}_t, \bm{y}_t)}{S} + \E \left[ \frac1T \schange{D(\pbf)} \cdot (\text{\# explore}) \right] &\le g(T)\,,
  \end{align*}
  where (\# explore) is the number of exploration rounds.
\end{definition}

We prove the following extension of Blackwell's approachability theorem. Interestingly, the bound in \Cref{thm:bandit-blackwell-thm} in terms of the dependency on the number of rounds $T$ is information-theoretically tight. See \Cref{app:bandit-blackwell-lowerbound} for details.

\begin{restatable}{theorem}{banditblackwellthm}\label{thm:bandit-blackwell-thm}
  A closed convex set $S$ is \schange{$\bigO{{D(\pbf)}^{1/3}\diameter{\hat{\pbf}}^{2/3}(\log d)^{1/3}T^{-1/3}}$}-bandit-approachable in the bandit Blackwell sequential game $\banditblackwellsequentialgame$ if and only if $S$ is response-satisfiable in the Blackwell game $\blackwellsequentialgame$. In particular, when $S$ is response satisfiable, the online algorithm $\BanditBlackwellAlg$ (\Cref{alg:bandit-blackwell}) achieves this approachability bound in polynomial time, given access to a separation oracle for $S$.
\end{restatable}

\noindent\emph{{Proof sketch of \Cref{thm:bandit-blackwell-thm}.}} To see the only if direction of the first part of the theorem, bandit Blackwell approachability implies Blackwell approachability. Specifically, if 
$$\standarddistance{\frac1T \sum_{t=1}^T \pbf(\bm{x}_t, \bm{y}_t)}{S} + \E \left[ \frac1T \schange{C} \cdot (\text{\# explore}) \right] \le \schange{\bigO{{D(\pbf)}^{1/3}\diameter{\hat{\pbf}}^{2/3}(\log d)^{1/3}T^{-1/3}}}\,,$$ then we must have $\standarddistance{\frac1T \sum_{t=1}^T \pbf(\bm{x}_t, \bm{y}_t)}{S}\le \schange{\bigO{{D(\pbf)}^{1/3}\diameter{\hat{\pbf}}^{2/3}(\log d)^{1/3}T^{-1/3}}}$, and hence this $\ell_\infty$-distance is vanishing as $T\rightarrow{+\infty}$.
This, in turn,  implies that the target set $S$ is response satisfiable (see \Cref{def:blackwell-thm}). {Note that while \Cref{def:blackwell-thm} is stated for a specific $g(T)$, the only if direction of this theorem holds for any vanishing approachability bound~\citep{blackwell1956analog}.}

To see the if direction and the second part of the theorem, we consider a simple algorithm that uses a (full information) Blackwell algorithm $\BlackwellAlg$ as a blackbox. We pick an algorithm $\BlackwellAlg$ that satisfies the approachability bound of \Cref{def:blackwell-thm}, and can obtain this bound in polynomial time given a separation oracle for $S$; see \Cref{rem:polytime-blackwell}.  At the beginning of each round, our bandit algorithm plays the last suggested action by $\BlackwellAlg$. It then explores  randomly with probability $q$ by flipping an independent coin. Based on the outcome of the coin, it either updates the state of $\BlackwellAlg$ using the unbiased payoff feedback it gets (exploration) and queries $\BlackwellAlg$ for suggesting a new action to follow, or decides not to explore with probability $1-q$ and refrains the state of $\BlackwellAlg$. These steps are summarized in \Cref{alg:bandit-blackwell}.

As for the running time, the above algorithm will run in polynomial time given a separation oracle for $S$ based on \Cref{rem:polytime-blackwell}. As for the approachability bound, at a high level, if we imagine that unbiased payoffs are the actual payoffs in the Blackwell game, then the expected distance of time-averaged unbiased vector payoff from $S$ is roughly equal to the same quantity for only rounds that we explore. There are $qT$ such rounds in expectation. Therefore, the expected distance is upper bounded by $O(\diameter{\hat\pbf}\left(\log(d)\right)^{1/2}q^{-1/2}T^{-1/2})$ due to the approachability of $\BlackwellAlg$ for this imaginary Blackwell sequential game (Theorem~\ref{def:blackwell-thm}). Also, the algorithm gets penalized on average by $\schange{O(D(\pbf)q)}$ due to exploring. Taking expectation to replace unbiased estimators with the actual payoffs and balancing the two terms in regret by setting $q=\schange{{D(\pbf)}^{-2/3}\diameter{\hat{\pbf}}^{2/3}(\log d)^{1/3}T^{-1/3}}$ gives the final bound. See \Cref{sec:appendix-bandit} in the appendix for a detailed proof with a more involved argument.
$\blacksquare$

\medskip
\medskip
\begin{algorithm}
\caption{Bandit Blackwell Online Algorithm ($\BanditBlackwellAlg$)}
\label{alg:bandit-blackwell}
\textbf{Meta Input:}
  Parameter $q\in[0,1]$, bandit Blackwell sequential game $\banditblackwellsequentialgame$~.\\
\textbf{Input:} Number of rounds $T$, blackbox access to full information online algorithm $\BlackwellAlg$ for the Blackwell sequential game $\blackwellsequentialgame$, achieving approachability bound of \Cref{def:blackwell-thm}~.\\
\textbf{Output:} Actions
$\{\xbf_t\}_{t\in [T]}$ and  binary signals $\{\pi_t\}_{t\in [T]}$, where 
$\xbf_t\in\algspaceB$, and  $\pi_t\in\{\textsc{Yes},\textsc{No}\}$ for any $t\in [T]$.\\
\vspace{2mm}
Initialize $\xbf_{\textrm{new}}$ by sending the initial query to $\BlackwellAlg$\,;
\For{\textrm{round} $t=1$ to $T$}{
Play the action $\xbf_t\leftarrow \xbf_{\textrm{new}}$\,;
Set $\pi_t$ to be $\textsc{YES}$ with probability $q$, and $\textsc{No}$ with probability $1-q$\,;
\If {$\pi_t=\textsc{Yes}$}{
Obtain $\hat\pbf({\xbf_{t}},\ybf_t)$ and send $\hat\pbf({\xbf_{t}},\ybf_t)/q$ as feedback to $\BlackwellAlg$\,;
// {\small{\emph{$\BlackwellAlg$ gets a new feedback in each exploration round, i.e., round $t$ where $\pi_t=\textsc{Yes}$.}}}

Update $\xbf_{\textrm{new}}$ by querying $\BlackwellAlg$ given the actions and realized unbiased estimator vector payoffs in exploration rounds prior to round $t+1$, i.e., $\xbf_{\textrm{new}}\leftarrow\BlackwellAlg\left(\left\{{(\xbf_\tau,\hat\pbf(\xbf_\tau,\ybf_\tau):\tau\leq t, \pi_\tau=\textsc{Yes}}\right\}\right)$\,;

}
 
}

\end{algorithm}

\begin{remark}
Our notion of bandit Blackwell approachability and the algorithm that achieves the tight bound (\Cref{alg:bandit-blackwell}) bear some resemblance to the $\epsilon$-greedy algorithm in the classic bandit setting, where in every round of this algorithm, we decide whether or not to explore, and when we explore in a round we assume we suffer from the maximum possible regret in this round. 
\end{remark}

\begin{remark}
The vanilla version of $\BanditBlackwellAlg$ needs to tune exploration probability $q$ based on the horizon $T$ to obtain the bound in \Cref{thm:bandit-blackwell-thm}. However, by using the standard \emph{doubling trick} in online learning (e.g., see \cite{bubeck2015convex}) in a blackbox fashion, one can boost \Cref{alg:bandit-blackwell} to work for unknown but bounded $T$: the new algorithm starts with a guess for horizon (e.g., $T=1$) and sets $q$ according to this guess. Each time it reaches the guessed horizon, it doubles its guess, and restarts by tuning a new value for $q$ and initializing again. The doubling trick is a well-known idea in the online learning literature that can be traced back to the classic work of \cite{auer2002nonstochastic}. We refer the reader to aforementioned work,  and omit the details here for brevity. 
\end{remark}

\subsection{Offline-to-online Transformation under Bandit Feedback}
Similar to our full information offline-to-online transformation, which gave us algorithm $\onlinemetatext$ in \Cref{sec:full-info}, 
we transform an offline IG algorithm to a bandit online learning algorithm by associating an instance of  the bandit Blackwell sequential game to each subproblem $i\in [N]$ of the offline algorithm. That is, we
crucially  rely  on a reduction from the local optimization step of each subproblem in \Cref{alg:offline-meta} to an approachable instance of the bandit Blackwell sequential game as in Definition~\ref{def:bandit-blackwell-approach}. Such a reduction is possible if the offline algorithm is \emph{Bandit Blackwell reducible};  see the following definition.

\begin{definition}[Bandit Blackwell Reducibility]
\label{def:bandit-blackwell-reducible}
An instance $\offlinemeta$ of Algorithm~\ref{alg:offline-meta} is bandit Blackwell reducible if there is an instance $\left(\algspaceB,\advspaceB,\pbf,\hat \pbf\right)$  of  bandit Blackwell sequential game (Section~\ref{sec:bandit-blackwell}) and an \emph{exploration sampling device} $\unbiasedestimator: \paramspace\times \domain \to  \Delta\left(\mathbb{R}^{\payoffdimension}\times \constraint\right)$, such that:
\begin{enumerate}
  \item [1.] $\offlinemeta$ is Blackwell reducible as in Definition~\ref{def:blackwell-reducible}, using the Blackwell sequential game $\left(\algspaceB,\advspaceB,\pbf\right)$ (with biaffine $\pbf$) and the synthetic Blackwell adversary function $\advfunB$.
\item [2.] If $\ybf=\advfunB(\zbf,f)$ for some  $f\in\funcspace,\zbf\in\domain$, then $\hat\pbf(\param,\ybf)=f(\zbf_{\textrm{exp}})\wbf_{\textrm{exp}}$ for all $\param\in\paramspace$, where $\left(\wbf_{\textrm{exp}},\zbf_{\textrm{exp}}\right)\sim \unbiasedestimator(\param,\zbf)$. Otherwise,  $\hat{\pbf}(\param,\ybf) = \pbf(\param,\ybf)$.
  
  \item [3.] The above $\hat\pbf$ is an unbiased estimator for the actual vector payoff, i.e., for all $\param\in\paramspace,\ybf\in\advspaceB:\E[\hat{\pbf}(\param,\ybf)] = \pbf(\param,\ybf)$.
  \item [4.] The exploration sampling device $\unbiasedestimator(\param,\zbf)$ returns its samples $(\wbf_{\textrm{exp}},\zbf_{\textrm{exp}})$  in polynomial time.
\end{enumerate}
\end{definition}
\medskip

To better understand the bandit Blackwell reducibility, we revisit our running example. 

\medskip
\setcounter{example}{0}
\begin{example}[continued]
The greedy algorithm of \cite{nemhauser1978analysis}  is also bandit Blackwell reducible. As stated in Section~\ref{sec:full-info}, this algorithm is Blackwell reducible. Recall that in this example, the biaffine Blackwell payoff is $\pbf(\param,\bm{y}) = \param^T\bm{y}\mathbf{1}_n - \bm{y}$, where $\mathbf{1}_n$ is all ones $n$-dimensional vector. We will construct an exploration sampling device $\unbiasedestimator$ that returns $(\wbf_{\textrm{exp}},\zbf_{\textrm{exp}})$ such that if $\forall \param\in\paramspace$, we have $\ybf=\advfunB(\zbf,f)$ for some  $f\in\funcspace,\zbf\in\domain$, we set 
$\hat\pbf(\param,\ybf)=f(\zbf_{\textrm{exp}})\wbf_{\textrm{exp}}$ and we must have $\E[\hat{\pbf}(\param,\ybf)] = \pbf(\param,\ybf)$.
The exploration sampling device $\unbiasedestimator$ works as follows. Given a point $\zbf \in \constraint$ (which represents a set of elements) and parameter $\param \in \paramspace$, it draws $j\sim\textrm{Uniform}\{1,\ldots,n\}$ and returns (i) $\wbf_{\textrm{exp}}=n \left(\param_j \mathbf{1}_n - \mathbf{e}_j\right)$, (ii) $\zbf_{\textrm{exp}} = \zbf \cup \{j\}$. 
Now, $\hat\pbf$ is an unbiased estimator of $\pbf$, because: 
\begin{align*}
\expect{\hat\pbf(\param,\advfunB(\zbf,f))}&=\expect{f(\zbf_{\textrm{exp}})\wbf_{\textrm{exp}}} \\
&=\expect{n \left(\param_jf(\zbf \cup \{j\}) \mathbf{1}_n - f(\zbf \cup \{j\})\mathbf{e}_j\right)}\\
&=
\mathbf{1}_n \sum_{j\in [n]} \param_j f(\zbf \cup \{j\} )-[f(\zbf\cup \{1\}), \ldots,f(\zbf\cup \{n\}) ]^T\\
&=\mathbf{1}_n \sum_{j\in [n]} \param_j \left(f(\zbf \cup \{j\} )-f(\zbf)\right)-[f(\zbf\cup \{1\}), \ldots,f(\zbf\cup \{n\}) ]^T + f(\zbf)\mathbf{1}_n\\
&=\param^T\ybf\mathbf{1}_n-\ybf=\pbf(\param,\advfunB(\zbf,f))\,,
\end{align*}
where $\ybf\triangleq \left[f(\zbf\cup\{j\})-f(\zbf)\right]_{j=1,2,\ldots,n}=\advfunB(\zbf,f)$. Here, the fourth equation holds because $\sum_{j\in [n]}\param_j= 1$. Observe that the exploration sampling device $\unbiasedestimator$ has an intuitive interpretation, at every round, it randomly picks one of the elements $j\in [n]$, and  evaluates the marginal benefit of adding element $j$ to $\zbf$. 
\end{example}

\subsubsection{Overview of the Algorithm}
When the offline algorithm (\Cref{alg:offline-meta}) is bandit Blackwell reducible (\Cref{def:bandit-blackwell-reducible}), we can employ a similar offline-to-online transformation mentioned in Section \ref{sec:full-info}. However, instead of associating an instance of the Blackwell game to each subproblem, we associate an instance of the bandit Blackwell game. To obtain unbiased estimators for the vector payoffs of these bandit Blackwell instances, we rely on the exploration sampling devices that are promised by Definition~\ref{def:bandit-blackwell-reducible}.
This sampling device allows us to strike a  balance between  exploration  and exploitation in all of the online bandit Blackwell games. We formalize this transformation of the offline algorithm to an online bandit algorithm called $\banditmetatext$ in \Cref{alg:bandit-meta}. 
 
 Suppose that the offline algorithm  $\offlinemeta$ is given. For the particular bandit Blackwell sequential game $\left(\algspaceB,\advspaceB,\pbf,\hat \pbf\right)$ coming from \Cref{def:bandit-blackwell-reducible}, we use $\BanditBlackwellAlg$ (Algorithm \ref{alg:bandit-blackwell}) to determine the strategy of player 1. Such an online bandit Blackwell algorithm as player 1 ensures that the distance between the average vector payoff $\frac1T \sum_{t=1}^T \pbf(\bm{x}_t, \bm{y}_t)$ and set $S$ plus the exploration penalty 
goes to zero with rate  $g(T)=\schange{O({D(\pbf)}^{1/3}{D(\hat{\pbf})}^{2/3}(\log d)^{1/3}T^{-1/3})}$; see Theorem \ref{thm:bandit-blackwell-thm}. 

 We dedicate a copy of the above algorithm  $\BanditBlackwellAlg^{(i)}$ to  each subproblem $i\in [N]$. We query algorithms $\BanditBlackwellAlg^{(i)}$ in the increasing order of their index $i$. Consider the online bandit Blackwell algorithm $\BanditBlackwellAlg^{(i)}$, and assume that in round $t$, we query this algorithm. 
The algorithm returns two outputs: the update parameter $\param_t^{(i)}$ and a binary signal  $\pi_t^{(i)}\in \{\textsc{Yes},\textsc{No}\}$.  
If $\pi_t^{(i)}= \textsc{Yes}$, the algorithm  explores: it samples  $(\bm{w}^{(i)}_{t,\textrm{exp}}, \zbf^{(i)}_{t,\textrm{exp}})$ from the exploration sampling device $\unbiasedestimator( \param^{(i)}_t,\zbf_t^{(i-1)})$.
Note that the exploration sampling device uses the update parameter $\param_t^{(i)}$ and the point returned by the previous subproblem $\zbf_t^{(i-1)}$. This indeed allows the subproblems to communicate with each other 
during exploration. The algorithm then plays $\zbf_t =\zbf^{(i)}_{t,\textrm{exp}}$ and provides the payoff vector feedback $\hat\pbf_t^{(i)} = f_t(\zbf_t) \bm{w}^{(i)}_{t,\textrm{exp}}$ to $\BanditBlackwellAlg^{(i)}$. This feedback is only used by the online bandit Blackwell algorithm  $\BanditBlackwellAlg^{(i)}$, not the rest of $N-1$  bandit Blackwell algorithms. We highlight that if $\BanditBlackwellAlg^{(i)}$ decides to explore in round $t$, the rest of 
bandit Blackwell algorithms will not be queried. 
Finally, if  $\pi_t^{(i)}= \textsc{No}$, the algorithm  exploits: it returns point $\zi{i}_t= \update(\param^{(i)}_t,\zbf^{(i-1)}_t)$. Again observe that during  exploitation, subproblem $i$ also communicates with subproblem $i-1$ through  using $\zbf^{(i-1)}_t$.

\subsubsection{Regret Analysis}
Theorem \ref{thm:banditILO} bounds the regret of  the $\banditmetatext$ algorithm. The proof is deferred to \Cref{apx:banditmeta-proof} in the appendix.

\begin{algorithm}[tbh]
\caption{Bandit Online Learning Meta-algorithm ($\banditmetatext$)}
\label{alg:bandit-meta}
\textbf{Meta Input:} Feasible region $\constraint$~, function space $\funcspace$, defined over domain $\domain$, parameter space $\paramspace \subseteq \mathbb{R}^{\paramdimension}$~. \\

\textbf{Offline algorithm and reduction gadgets:} An instance $\offlinemeta$ of Algorithm~\ref{alg:offline-meta};  this algorithm is bandit Blackwell reducible as in Definition~\ref{def:bandit-blackwell-reducible}, using the bandit Blackwell instance  $\left(\algspaceB,\advspaceB,\pbf,\hat\pbf\right)$ and exploration sampling device  $\unbiasedestimator:\paramspace\times \domain \times \to  \Delta\left(\mathbb{R}^{\payoffdimension}\times \constraint\right)$.\\
\vspace{-2mm}

\textbf{Input:} Number of rounds $T$; access to a bandit Blackwell online algorithm $\BanditBlackwellAlg$~.

\textbf{Output:} Points $\zbf_1, \zbf_2, \ldots, \zbf_T \in \constraint$.

\vspace{2mm}
Initialize $N$ parallel instances $\{\BanditBlackwellAlg^{(i)}\}_{i=1}^{N}$ of the online algorithm $\BanditBlackwellAlg$\,;
\For{round $t=1$ to $T$}{
  Initialize $\zbf^{(0)}_t \in\constraint$\,;
  \For{subproblem $i=1$ to $N$}{
    Choose the update parameter $\param^{(i)}_t \in \paramspace$ and exploration signal $\pi_t^{(i)} \in \{ \textsc{Yes},\textsc{No}\}$ by 
    querying online algorithm $\BanditBlackwellAlg^{(i)}$ given the update parameters and vector payoffs $\hat{\pbf}$ of exploration rounds prior to round $t$ in the bandit Blackwell sequential game of subproblem $i$, that is $\left(\param^{(i)}_t,\pi_t^{(i)}\right)\leftarrow\BanditBlackwellAlg^{(i)}\left(\param^{(i)}_1,\ldots,\param^{(i)}_{t-1}, \{\hat\pbf(\param^{(i)}_\tau,\ybf^{(i)}_\tau)\}_{\tau\leq t-1:\pi_\tau^{(i)}=\textsc{Yes}}\right)$\,;
    \If{$\pi^{(i)}_t = \textsc{Yes}$,} {
      Sample $(\bm{w}^{(i)}_{t,\textrm{exp}}, \zbf^{(i)}_{t,\textrm{exp}})$ from the exploration sampling device $\unbiasedestimator( \param^{(i)}_t,\zbf_t^{(i-1)})$\,;
      Play the exploration  point $\zbf_t \leftarrow \zbf^{(i)}_{t,\textrm{exp}}$\,;
     
      \vspace{2mm}
      < \emph{Bandit information feedback: observe  $f_t(\zbf_t)$} >\,;
      \vspace{2mm}
      
      Give payoff vector feedback $\hat\pbf_t^{(i)} = f_t(\zbf_t) \cdot \bm{w}^{(i)}_{t,\textrm{exp}}$ to $\BanditBlackwellAlg^{(i)}$\,;
      Skip immediately to the beginning of the next round $t+1$\,;
    }
    Set $\zi{i}_t\leftarrow \update(\param^{(i)}_t,\zbf^{(i-1)}_t)$\,;
  }
  Play the final point $\zbf_t \leftarrow \zbf^{(N)}_t$ and receive bandit feedback $f_t(\zbf_t)$, and  ignore it.
}
\end{algorithm}
\begin{theorem}[Bandit information offline-to-online transformation]
  \label{thm:banditILO}

  Suppose that an instance of $\offlinemeta$ for the offline problem~\eqref{eq:offline-optimization} satisfies the following properties:
  \begin{itemize}[noitemsep]
    \item It is an extended $(\gamma,\delta)$-robust approximation for $\gamma\in(0,1)$ and $\delta>0$, as in Definition~\ref{def:extended-robust-approx}.
    \item It is bandit Blackwell reducible; that is, we can define the bandit Blackwell sequential game $\banditblackwellsequentialgame$ and exploration sampling device  $\unbiasedestimator:\paramspace\times \domain \to  \Delta\left(\mathbb{R}^{\payoffdimension}\times \constraint\right)$ that satisfy the contions in \Cref{def:bandit-blackwell-reducible}. 
    
  \end{itemize}
    Consider the bandit-information adversarial online learning version of problem~\eqref{eq:offline-optimization}, and let $\BanditBlackwellAlg$ be a polynomial time bandit Blackwell algorithm for $\banditblackwellsequentialgame$ as in \Cref{thm:bandit-blackwell-thm}. 
  Then, for this online problem, $\banditmeta$  runs in polynomial time and  satisfies the following $\gamma$-regret bound:
  \begin{equation*}
    \gamma\textrm{-regret}\left(\banditmeta \right)\leq \schange{\bigO{{D(\pbf)}^{1/3}{D(\hat{\pbf})}^{2/3}N\delta\left(\log(\payoffdimension)\right)^{1/3} T^{2/3}}}\,,
  \end{equation*} 
  where $N$ is the number of subproblems and $\payoffdimension$ is the dimension of vector payoffs.
\end{theorem}
\medskip 

We finish this section by wrapping up our running example (Example~\ref{example:running-1}) and mentioning the bandit regret bound we get as a direct corollary of Theorem~\ref{thm:banditILO}.

\setcounter{example}{0}
\begin{example}[finished]
The greedy algorithm in \cite{nemhauser1978analysis} satisfies $(1-\tfrac{1}{e},1)$-robust approximation and is bandit Blackwell reducible. It has $N=k$ subproblems and $\ell_\infty$ diameter of $\hat\pbf$ is $D(\hat\pbf)=O(n)$. Therefore, by invoking Algorithm~\ref{alg:bandit-meta} given any bandit Blackwell  algorithm satisfying the approachability bound in Theorem~\ref{thm:bandit-blackwell-thm}, we obtain the following bandit regret bound:
$$
  \left(1-\dfrac{1}{e}\right)\textrm{-regret}\left(\textrm{Algorithm~\ref{alg:bandit-meta}} \right)\leq O(kn\schange{^{2/3}}(\log n)^{1/3}T^{2/3})~,
  $$ 
  which in turn, by noting that $k$ can be as large as $n$, gives us an immediate improvement over regret bound of $O\left(k^{2}(n\log n)^{1/3}T^{2/3}(\log T)^{2}\right)$ in \cite{streeter2007online, streeter2009online}. 
\end{example}

%% file: tex/applications.tex
\section{Applications to Revenue Management and Combinatorial Optimization}
\label{sec:applications}

 We have already showed how to fit monotone submodular maximization into our framework through \Cref{example:running-1}. In this section, we apply our framework to two other selected problems: product ranking through sequential submodular maximization and  personalized reserve price optimization in second-price auction. (See Section \ref{subsec:usm} in the appendix to see how to apply our framework to the problem of maximizing  non-monotone continuous weak-DR submodular functions, in which obtaining any sub-linear approximate regret has been an open problem for a while. {See also \Cref{subsec:DR-SM} to see how to apply our framework to the problem of maximizing monotone continuous strong-DR submodular functions subject to downward closed bounded convex sets.)} Our framework results in improved/new regret bounds in all mentioned applications for both full-information and bandit settings. 
 

\subsection{Application to Product Ranking and Sequential Submodular Maximization}
\label{subsec:ranking}
\paragraph{Problem definition.} In the \textsc{Product Ranking Problem}, a platform aims to characterize a ranking of $n$ items, where a ranking is a permutation $\pibf$ over the items. Here, items on positions with lower indices have more visibility. The goal of the platform is to maximize its user engagement (also known as market share), which is the probability that a consumer does not leave the platform without taking a desired action. This action can be a click, purchase, or even installing an application. 

For the sake of presentation, assume that the desired action is clicking on an item. We consider the model proposed by \cite{asadpour2020ranking}, which is inspired by an earlier model proposed in \cite{ferreira2019learning}. In this model, a consumer $u$ is characterized by a patience level $\theta_u$ together with a monotone non-decreasing submodular set function $\kappa_u:2^{[n]}\xrightarrow{}[0,1]$. A consumer of type $(\theta_u,\kappa_u)$, when offered a ranked list of products $\pibf = ([\pibf]_1,[\pibf]_2,\ldots,[\pibf]_n),$ inspects the first $\theta_u$ products and clicks with probability $\kappa_u\left(\left\{[\pibf]_1,\ldots,[\pibf]_{\theta_u}\right\}\right)$. The platform knows the distribution $\mathcal{G}$ from which $u$ is selected. The goal is to pick a permutation $\pibf$ maximizing the probability of click $$\E_{u\sim\mathcal{G}}\left[\kappa_u\left(\left\{[\pibf]_1,\ldots,[\pibf]_{\theta_u}\right\}\right)\right]~.$$
For a wide range of choice models in the literature, the probability of a purchase from an offered set $S$ can be described using a monotone submodular function $\kappa_u$. This includes multinomial logit, nested logit, and paired combinatorial logit models. See \cite{kok2008assortment} for details on these models.


\paragraph{Product ranking problem as sequential submodular maximization.} A slight reformulation of the above model casts the product ranking problem as a special case of a class of optimization problems over permutations called sequential submodular maximization~\citep{asadpour2020ranking}. We define the sequential submodular maximization problem as follows.\footnote{{Our notion of sequential submodular functions is different from the notion of sequential submodular functions studied in \cite{tschiatschek2017selecting,mitrovic2018submodularity}. However, under all these notions, the goal is to return a sequence (permutation), rather than a set.} } Given a sequence of monotone submodular set functions $\left\{\seqsubf{1}(\cdot),\ldots,\seqsubf{n}(\cdot)\right\}$, and a sequence of non-negative weights $\bm{\lambda} =(\lambda_1,\ldots,\lambda_n),$ we aim to find a ranking $\pi$ that maximizes
\begin{equation*}
    \sum_{i=1}^n\lambda_i\seqsubf{i}\left(\left\{[\pibf]_1,\ldots,[\pibf]_i\right\}\right),
\end{equation*}
where $[\pibf]_i$ denotes the item on the $i^{\text{th}}$ position of ranking $\pi.$ In the aforementioned choice model, for all $i\in[n]$, we have $\seqsubf{i}(S) \triangleq \E_{u\sim\mathcal{G}}\left[\kappa_u(S)|\theta_u =i\right],$ representing the probability of clicks functions, and $\lambda_i\triangleq\Pp_{u\sim\mathcal{G}}\left(\theta_u=i\right),$ representing the probability that a consumer has patience level $i$. The probability that a consumer clicks on at least one product when offered a ranked set of products $\pibf$ is then
\begin{equation*}
    f(\pibf) \triangleq \lambda_1\seqsubf{1}\left(\{[\pibf]_1\}\right)+\lambda_2\seqsubf{2}\left(\{[\pibf]_1,[\pibf]_2\}\right)+\ldots+\lambda_n\seqsubf{n}\left(\{[\pibf]_1,\ldots,[\pibf]_n\}\right),
\end{equation*}
where $\seqsubf{i}$'s are monotone submodular functions and $\lambda_i$'s are non-negative. To simplify the analysis, notice that while $f$ is a function of a set of ranked/ordered items, $\seqsubf{i}$ is a function of a set that has at most $i$ items for each $i\in[n].$

\paragraph{Online problem.} In the offline setting, the platform knows $\mathcal{G},$ which translates to knowing the probability of clicks $\left\{\seqsubf{1}(\cdot),\ldots,\seqsubf{n}(\cdot)\right\}$ and the probability distribution of the patience level $\bm{\lambda} =(\lambda_1,\ldots,\lambda_n)$. We study the online user-engagement-maximization ranking problem where on every round $t,$ a distribution over patience level $\bm{\lambda}_t$ and the expected probability of click function $f_t,$ which is made of $\left\{\seqsubf{{t,1}}(\cdot),\ldots,\seqsubf{{t,n}}\right\}$, are chosen adversarially. The platform, whose goal is to maximize its user-engagement, chooses a ranking $\pi_t$ without observing $\bm{\lambda}_t$ and $f_t$. After choosing the ranking, the platform observes the function $f_t$ in the full-information setting. In the bandit setting, the platform \emph{only} observes whether or not the consumer clicks on at least one item, but not which item was clicked on. To the best of our knowledge, the online adversarial version of this problem has not been studied before, neither under full information nor bandit setting.


\paragraph{Offline algorithm.} In this paper,  we  focus on a greedy algorithm that achieves $\tfrac12-$approximation, and transform it into an online adversarial learning algorithm. Our offline algorithm is presented in \Cref{alg:rank-off}. The input to this algorithm is a sequential submodular function $f:\Pi\rightarrow[0,1]$, where $\Pi$ is the set of ranking permutations of $n$ items, i.e., $\Pi = \{0,1,\ldots,n\}^n$. In this case, having $[\pibf]_i = 0$ represents putting no item at position $i$ and multiple positions can display the same item for simplicity. In this problem, both the domain and the feasible region are $\domain = \constraint = \Pi.$ Let $S_i$ denote the  set of subsets of $[n]$ that consist of at most $i$ items, i.e., $S_i =  \{S\subseteq[n]~|~|S|\leq i\}$. We have $f(\pibf) = \sum_{j=1}^n\lambda_j\seqsubf{j}(\{[\pibf]_1,\ldots,[\pibf]_j\})$ where $\seqsubf{i}$ is a monotone submodular function that takes an element of $S_i$ as input and returns a probability in $[0,1],$ i.e., $\seqsubf{i}:S_i\rightarrow[0,1]$. \Cref{alg:rank-off}, taken from \cite{asadpour2020ranking} and \cite{ferreira2019learning}, is a greedy algorithm with $n$ subproblems, where each subproblem corresponds to a position on the ranking. The algorithm starts filling up the positions from the top and for each position $i,$ it chooses the item that has the highest marginal probability of click. The update $\pibf^{(i)}\leftarrow\pibf^{(i-1)} + z^{(i)}\mathbf{e}_i$ represents the action of adding item $z^{(i)}$ to position $i.$


    \begin{algorithm}
    \caption{Greedy for Sequential Submodular Maximization \citep{asadpour2020ranking}}
    \label{alg:rank-off}
      \textbf{Input:} A sequential submodular function $f$, which can be represented using a sequence of monotone submodular functions $\{\seqsubf{i}(\cdot)\}_{i\in[n]}$ and a sequence of non-negative weights $\bm{\lambda}$.\\
      \textbf{Output:} Ranking $\pibf\in\Pi$.\\
      Set initial ranking $\pibf^{(0)} \leftarrow \mathbf{0}_n$.\\
      \For{position $i=1, 2, \ldots, n$}{
        \vspace{1mm}
        \underline{Local Optimization Step}\\
        Choose $z^{(i)} \in \argmax_{z \in [n]} \sum_{j=i}^n\lambda_j\seqsubf{j}\left(\{[\pibf^{(i-1)}]_1,\ldots,[\pibf^{(i-1)}]_{i-1},z\}\right) - $\\ \qquad\qquad\qquad$ \sum_{j=i}^n\lambda_j\seqsubf{j}(\{[\pibf^{(i-1)}]_1,\ldots,[\pibf^{(i-1)}]_{i-1}\})$.\\
        \underline{Local Update Step}\\
        Set $\pibf^{(i)} \leftarrow \pibf^{(i)} + z^{(i)}\mathbf{e}_i$.
      }
      \Return{$\pibf\leftarrow\pibf^{(n)}$.}
    \end{algorithm}

We cast \Cref{alg:rank-off} as an instance of $\offlinemetatext$ (\Cref{alg:offline-meta}). The parameter space is $\paramspace = \Delta([n])$ and $\paramdimension = n$. Moreover, in subproblem $i$ the algorithm picks the distribution $\param^{(i)}$ over items so that the resulting vector payoff lands in set $S$. In this language, set $S$ is the $n$-dimensional positive orthant and the vector payoff function is:
\begin{equation*}
    \forall j\in[n]:~~~\left[\pay(\param^{(i)},\pibf^{(i-1)},f)\right]_j=\param^T\ybf^{(i)}-[\ybf^{(i)}]_j~~,
\end{equation*}
where 
\begin{align*}
    \ybf^{(i)}&\triangleq\left[\sum_{a=i}^n\lambda_a\seqsubf{a}(\{[\pibf^{(i-1)}]_1,\ldots,[\pibf^{(i-1)}]_{i-1},j\})-\sum_{a=i}^n\lambda_a\seqsubf{a}(\{[\pibf^{(i-1)}]_1,\ldots,[\pibf^{(i-1)}]_{i-1}\})\right]_{j\in[n]}\\
    &=\big[f(\pibf^{(i-1)} + j\mathbf{e}_i)-f(\pibf^{(i-1)})\big]_{j\in[n]},
\end{align*}
is the marginal objective value of putting item $j$ on the $i^{\text{th}}$ position. Note that any $\param^{(i)}$ for which the vector payoff is in $S$ is indeed a distribution over items $z^{(i)}$ such that $z^{(i)}\in  \argmax_{z \in [n]} \sum_{j=i}^n\lambda_j\seqsubf{j}\left(\{[\pibf^{(i-1)}]_1,\ldots,[\pibf^{(i-1)}]_{i-1},z\}\right) - \sum_{j=i}^n\lambda_j\seqsubf{j}(\{[\pibf^{(i-1)}]_1,\ldots,[\pibf^{(i-1)}]_{i-1}\})$.

\begin{theorem}[Online learning for sequential submodular maximization]
\label{thm:seqsub}
 Let $n$ be the number of items. For the problem of maximizing sequential submodular functions in the online full information setting, there exists a learning algorithm that obtains $\bigO{n\sqrt{T\log{n}}}$ $\frac12$-regret, where $T$ is the number of rounds. Furthermore, in the online bandit setting, there exists a learning algorithm that obtains $\bigO{n^{\schange{5/3}}\left(\log{n}\right)^{1/3}T^{2/3}}$ $\tfrac12$-regret, where $T$ is the number of rounds. For this problem, the benchmark in the regret bounds is $\frac{1}{2}\max_{\pibf\in \Pi}\sum_{t=1}^T f_t(\pibf)$.
\end{theorem}

\Cref{thm:seqsub} and the following corollary are proved in \Cref{apx:ranking} using our offline-to-online transformations presented in Sections  \ref{sec:full-info} and \ref{sec:bandit-info}.

\begin{corollary}[Online learning for product ranking]
\label{cor:rank-general}
  Let $n$ be the number of items. For the problem of product ranking optimization to maximize user engagement using the model from \cite{asadpour2020ranking}, there exists a learning algorithm that obtains $\bigO{n\sqrt{T\log{n}}}$ $\frac{1}{2}-$regret in the full-information setting and $\bigO{n^{\schange{5/3}}\left(\log{n}\right)^{1/3}T^{2/3}}$ $\frac{1}{2}-$regret in the bandit setting, where $T$ is the number of consumers.\footnote{As an implication, the same problem for the consumer choice model from \cite{ferreira2019learning} also has a learning algorithm that obtains $\bigO{n\sqrt{T\log{n}}}$ $\frac{1}{2}-$regret in the full information setting and $\bigO{n^{\schange{5/3}}\left(\log{n}\right)^{1/3}T^{2/3}}$ $\frac{1}{2}-$regret in the bandit setting. For both problems, the benchmark in the regret bounds is $\frac{1}{2}\max_{\pibf\in \Pi}\sum_{t=1}^T f_t(\pibf)$.}
\end{corollary}

{In Section \ref{sec:numerics:product}, we numerically   evaluate of our online algorithms for the product ranking problem  under both full information and bandit settings.}


\subsection{Application to Maximizing Multiple Reserves in Second Price Auction}
\label{subsec:mmr}

\paragraph{Problem definition.} In the \textsc{Maximizing Multiple Reserves} (MMR)  problem~\citep{roughgarden2019minimizing, derakhshan2019lp}, a seller wants to sell one item to $n$ bidders to maximize her revenue. Each bidder $i$ has a private value $v_i$ for the item. The seller runs a second-price auction with personalized reserves $\mathbf{r}$; the winner is the bidder with the highest bid/valuation among the bidders whose bids exceed their reserve prices.\footnote{Since second price auctions are truthful, we will use bids and valuations interchangeably. } The winner pays the minimum bid she needed to win, which is the maximum between their reserve price and the second-highest bid that cleared its reserve price. 
\paragraph{Online problem.} We are interested in the seller's problem in the online full information and bandit settings. In both settings, each round $t \in [T]$ involves the seller choosing a set of reserves $\rbf_t$ and the adversary choosing a valuation profile $\vibf_t$. In the online full information setting, the seller observes the valuation profile and gets credit for the resulting revenue. In the online bandit setting, the seller observes just the resulting revenue and does not observe the bidders' valuations or even the identity of the winner. The seller's goal is to minimize the difference between his average revenue and the best average revenue in hindsight for a fixed set of reserves $\mathbf{r}^*$. To the best of our knowledge, the bandit setting has not been studied in the literature. However, the full information setting of this problem is studied by \cite{roughgarden2019minimizing}, where they present a learning algorithm  with $\bigO{n \sqrt{T \log T}}$ $\tfrac12$-regret, which we improve upon.

\paragraph{Offline non-batch vs batch problem.}
We start with formulating the offline non-batch problem. Let $\feasiblereserves = \{\feasiblereserve_1, \feasiblereserve_2, \ldots , \feasiblereserve_{m}\}$ be the set of feasible reserve prices, where $|\feasiblereserves| = m$, and they are sorted: $0 = \feasiblereserve_1< \feasiblereserve_2 < \cdots < \feasiblereserve_m$.  For the offline (non-batch) problem, let $f: \feasiblereserves^n \times [0,1]^n \to [0, 1]$ be the seller's revenue function: $f(\mathbf{r}, \mathbf{v}) = \max\{\indexintovector{\vibf}{\hat{j}}, \indexintovector{\rbf}{j^*}\} $ for some $\mathbf{v}$. Here, $j^*$ and $\hat{j}$ are the highest and second-highest bidders among those who cleared their bids, with ties broken arbitrarily: $
    j^*     \in \argmax_{j \in [n]: \indexintovector{\vibf}{j} \ge \indexintovector{\rbf}{j}} \{\indexintovector{\vibf}{j}\}$ and $
    \hat{j} \in \argmax_{j \in [n]: \indexintovector{\vibf}{j} \ge \indexintovector{\rbf}{j}, j \ne j^*} \{\indexintovector{\vibf}{j}\}$. 
    If no bidder clears their reserve, then we say $[\rbf]_{j^*}$ and $\indexintovector{\vibf}{\hat{j}}$ are both zero. Similarly, if only one bidder clears their reserve, we say $\indexintovector{\vibf}{\hat j}$ is zero. Moreover, $\funcspace$ is the space of all such revenue functions: $\funcspace = \{f \mid \exists \vibf \in [0, 1]^n \text{ such that } f(\mathbf{r},  \mathbf{v}) = \max\{\indexintovector{\vibf}{\hat{j}}, \indexintovector{\rbf}{j^*}\}$.
  In the offline (non-batch) problem, the goal is to solve $\max_{\mathbf{r}\in \mathcal R^n} f(\rbf, \vibf)$ for an input valuation profile $\vibf \in [0,1]^n$. In the optimization problem, the domain $\domain$ we consider and the feasible region $\constraint$ are both $\feasiblereserves^n$.

    The aforementioned offline problem can be solved efficiently. Note that in the offline problem, the seller who has access to the valuations of the bidders in one auction needs to optimize personalized reserve prices. It is then obvious that in the offline setting, the best action is to set reserve prices of all the bidders to zero, except the bidder with the highest bid; for this bidder, his reserve price is set to his valuation. 
    Then, one may wonder why for the online version  of this  offline (non-batch) problem, which  is not even NP-hard, we characterize  $\tfrac12$-regret, rather than $1$-regret. 
The reason is that \cite{roughgarden2019minimizing} show that the full information online setting is at least as hard as the offline batch problem, which is APX-hard. In the offline batch problem, the seller has access to the valuation profiles in $m$ auctions and would like to determine a single vector of reserve prices $\mathbf{r}$ that maximizes revenue across all the $m$ auctions. Considering the hardness of the offline batch problem, to solve the offline (non-batch) problem, we use a slight variation of the algorithm of \cite{roughgarden2019minimizing}. This variation is stated in \Cref{alg:mmr-single}, which obtains a $\tfrac12$ fraction of the optimal revenue similar to the original algorithm. See \Cref{apx:mmr-discussion} for a discussion on the major differences between the two. 

 \begin{algorithm}[htb]
    \caption{Greedy Algorithm for Discretized MMR \citep{roughgarden2019minimizing}}
    \label{alg:mmr-single}
      \textbf{Input:} Valuation profile $\vibf$.\\
      \textbf{Output:} Reserve prices $\rbf \in \feasiblereserves^n$.\\
      Set initial reserves $\mathbf{r}^{(0)} \leftarrow \mathbf{0}_n$.\\
      \For{bidder $i=1, 2, \ldots, n$}{
        Define \emph{revenue-from-reserves} function $q^{(i)}: \feasiblereserves \to [0, 1]$ as $q^{(i)}(r)$ equals $r$ if $i$ has the highest valuation (ties broken arbitrarily) and $r \in [\indexintovector{\vibf}{i'}, \indexintovector{\vibf}{i}]$ where $i'$ has the second-highest valuation, and $0$ otherwise.\\
        {\underline{Local Optimization Step}}\\
        Choose $z^{(i)} \in \argmax_{r \in \rcal} q^{(i)}(r)$.  // {\small{\emph{In this case $\param^{(i)} \in \setofdistributions(\feasiblereserves)$ is the distribution that always returns $z^{(i)}$.}}}\\
        {\underline{Local Update Step}}\\
        Set $\rbf^{(i)} \leftarrow \rbf^{(i-1)} + z^{(i)} \mathbf{e}_i$.
      }
      \Return{$\rbf \sim \text{Uniform}\{\mathbf{0}_n, \rbf^{(n)}\}$.}
    \end{algorithm}

\paragraph{Offline algorithm.} We now briefly discuss \Cref{alg:mmr-single} and show how to cast it as an instance of $\offlinemetatext$ (\Cref{alg:offline-meta}). This greedy algorithm has $n$ subproblems, where in each subproblem, reserve price of a bidder $i$ is set using our \emph{revenue-from-reserves} function $q$. At the end, the algorithm  randomly returns either the all-zeros reserve vector $\mathbf{0}_n$ or the crafted reserve vector denoted by $\rbf^{(n)}$, where the former yields revenue equal to the second-highest valuation and the latter yields revenue of at least $q^{(j^*)}(z^{(j^*)})$; see the definition of $q(\cdot)$ in the algorithm.  By definition of the revenue function, the optimal reserves obtain their revenue via one of these two cases, i.e.,
    \begin{align*}
      f(\rbf^*, \vibf) &\le \max \{ \indexintovector{\vibf}{\hat{j}}, q^{(j^*)}(\indexintovector{\rbf}{j^*}) \} 
                            \le \indexintovector{\vibf}{\hat{j}} + q^{(j^*)}(\indexintovector{\rbf}{j^*}) \le f(\mathbf{0}_n, \vibf) + f(\rbf^{(n)}, \vibf) = 2 \E \left[ f(\rbf, \vibf) \right],
    \end{align*}
    where the expectation is taken with respect to the randomness in the algorithm. This implies that  our algorithm is indeed a $\tfrac12$-approximation. Stated in the language of our \Cref{alg:offline-meta}, our local updates manage to guarantee that $\pay(\param^{(i)}, \rbf^{(i-1)}, \vibf)$ is in the positive orthant, where the (asymmetric) vector payoff function $\pay$ returns an $m$-dimensional point whose $j^{th}$ coordinate value is the expected difference between the expected value of picking a reserve according to $\param^{(i)}$ and that of picking $\feasiblereserve_j$:
  $
      \indexintovector{\pay(\param^{(i)}, \rbf^{(i-1)}, \vibf)}{j} \triangleq
      \E_{z' \sim \param^{(i)}} \left[ q^{(i)}(z') - q^{(i)}(\feasiblereserve_j) \right].\footnote{Note that here, $\pay(\param^{(i)}, \rbf^{(i-1)}, \vibf)$ is not a function of $\rbf^{(i-1)}$.}
$ 
The following theorem shows that using our framework, the greedy \Cref{alg:mmr-single}  can be transformed to polynomial-time online learning algorithms under both full information and bandit feedback structures. {See Section \ref{sec:numerics:reserve} for numerical evaluation of our online algorithms for reserve price optimization  under both full information and bandit settings.}

\begin{theorem}[Online learning for maximizing multiple reserves]
\label{thm:mmr}
  Let $\feasiblereserves = \{\feasiblereserve_1, \ldots , \feasiblereserve_{m}\}$ be the set of possible reserve prices and $n$ be the number of bidders. Assume that the maximum valuation is normalized to one. Then, for the problem of maximizing personalized reserve prices in the online full information setting, there exists a learning algorithm that obtains $\bigO{n T^{1/2} \log^{1/2} m}$ $\tfrac12$-regret, where $T$ is the number of auctions. Furthermore, in the online bandit setting, there exists a learning algorithm that obtains $\bigO{n m\schange{^{2/3}} T^{2/3} \log^{1/3} m}$ $\tfrac12$-regret, where $T$ is the number of auctions. Here, the benchmark in the regret bounds is $\frac{1}{2}\max_{\rbf\in \feasiblereserves^n}\sum_{t=1}^T f(\rbf, \vibf_t)$.
\end{theorem}

The proof of Theorem \ref{thm:mmr} is presented in \Cref{apx:mmr-main-proof}. 
The following corollary considers a stronger benchmark than the one we considered earlier. This benchmark allows the reserve prices to be any number in $[0,1]^n$, i.e., the regret is computed against  $\frac{1}{2}\max_{\rbf\in [0,1]^n}\sum_{t=1}^T f(\rbf, \vibf_t)$, rather than $\frac{1}{2}\max_{\rbf\in \feasiblereserves^n}\sum_{t=1}^T f(\rbf, \vibf_t)$. 
See \Cref{apx:mmr-cor} in the appendix for proof. 

\begin{corollary}
\label{cor:mmr}
  Let $n$ be the number of bidders. Assume that the maximum valuation is normalized to one. Then, for the problem of maximizing personalized reserve prices in the online full information setting, there exists a learning algorithm that obtains $\bigO{n T^{1/2} \log^{1/2} T}$ $\tfrac12$-regret, where $T$ is the number of auctions. Furthermore, in the online bandit setting, there exists a learning algorithm that obtains $\schange{\bigO{n^{3/5}T^{4/5}\log^{1/3}(nT)}}$ $\tfrac12$-regret. Here, the benchmark in the regret bounds is $\frac{1}{2}\max_{\rbf\in [0,1]^n}\sum_{t=1}^T f(\rbf, \vibf_t)$.
\end{corollary}

%% file: tex/numerics.tex
{
\section{Numerical Studies}\label{sec:numerics}

In this section, we evaluate our two algorithms $\onlinemetatext$ (Algorithm \ref{alg:full-info-backbone}) and  $\banditmetaWoInput$ (Algorithm \ref{alg:bandit-meta}) for two of the applications  we presented in Section \ref{sec:applications}: (i) product ranking problem and (ii) maximizing multiple reserves in second price auction. Recall that $\onlinemetatext$  is defined/designed for the full information setting while $\banditmetaWoInput$ is defined/designed for the bandit setting.

\subsection{Numerical Studies for Product Ranking Problem} \label{sec:numerics:product}
\textbf{Simulation setting.} In our setting, we have $n = 10$ items where $\mathcal{N} = \{1,2,\ldots,10\}$ is the universe of items, $T = 20,000$ customers (rounds), and $k=2$ types appearing equally likely. We consider multinomial logit (MNL)  choice models for customers. Given this choice modeling, a type $u$ is characterized by a patience level $\theta_u$ and a set of weights over items $\mathbf{w}_u\in\mathbb{R}^n$ such that given a ranked list of products $\bm{\pi} = ([\bm{\pi}]_1,[\bm{\pi}]_2,\ldots,[\bm{\pi}]_{n})$, they inspect the first $\theta_u$ products and clicks with probability $\kappa_u(\{[\bm{\pi}]_1,[\bm{\pi}]_2,\ldots,[\bm{\pi}]_{\theta_u}\}) = \dfrac{w_{u,[\bm{\pi}]_1}+ w_{u,[\bm{\pi}]_2}+\ldots+w_{u,[\bm{\pi}]_{\theta_u}}}{1+w_{u,[\bm{\pi}]_1}+ w_{u,[\bm{\pi}]_2}+\ldots+w_{u,
[\bm{\pi}]_{\theta_u}}}$, following the MNL model. Let $f(\bm{\pi})$ denote the expected market share when ranking $\bm{\pi}$ is offered.\footnote{Here, we have $f(\bm{\pi}) = \frac{1}{2}\sum_{u=1}^2\kappa_u(\{[\bm{\pi}]_1,[\bm{\pi}]_2,\ldots,[\bm{\pi}]_{\theta_u}\})$ as both types are equally likely.}

We change the parameters $(\theta_{u,t},\mathbf{w}_{u,t})_{u=1,2}$ every 500 rounds (one episode), and keep them constant in the same episode. At the end of an even episode, we set  patience levels $\theta_{1,t}=\theta_{2,t}=10$, while at the end of an odd episode, we set both of them to be $5$. Furthermore, at the end of an even episode, we choose $3$ items uniformly at random from $\{1,2,\ldots,5\}$ for both types, and put equal weight of $1/3$,  on each of the chosen items, and zero weights on the remaining items. On the other hand, at the end of an odd episode, we choose $3$ items uniformly at random from $\{6,\ldots,10\}$, then put equal weights on the chosen items, and $0$ for the remaining items. This setting allows us to model  the adversarial nature of the environment. 

\textbf{Implementation of $\onlinemetatext$ algorithm.} We replace {\BlackwellAlg } with the Multiplicative Weights or Hedge algorithm with the learning rate of $\epsilon_t = \sqrt{\frac{1}{t}}$ with $n= 10$ arms (representing items). The Hedge algorithm is a no-regret adversarial learning algorithm that keeps a weight over the different arms at each time step, and updates those according to the observed feedback. Recall that in the product ranking problem, subproblem $i$ corresponds to determining the item for the $i^{\text{th}}$ position in the ranking. Then, the rewards we give to the Hedge algorithm that corresponds to subproblem $i$ is the marginal market share of adding item $j$ to the top $(i-1)$ items, for all $j\in \mathcal{N}$. Specifically, suppose that in round $t$, the top $(i-1)$ items are $([\bm{\pi}]_1,[\bm{\pi}]_2,\ldots,[\bm{\pi}]_{i-1})$. Then, the set of rewards we feed to hedge algorithm in our setting is 
$
    f\left(\left\{[\bm{\pi}]_1,[\bm{\pi}]_2,\ldots,[\bm{\pi}]_{i-1},j\right\}\right) - f\left(\left\{[\bm{\pi}]_1,[\bm{\pi}]_2,\ldots,[\bm{\pi}]_{i-1}\right\}\right)
$
for $j=1,2,\ldots,10$.

\textbf{Implementation of  $\banditmetaWoInput$ algorithm.} For $\banditmetaWoInput$, we modify the Hedge algorithm to implement $\BanditBlackwellAlg$, and we call this the Hedge-bandit algorithm. The Hedge-bandit algorithm has exploration and exploitation rounds and only updates its weights in the exploration rounds. Following $\BanditBlackwellAlg$, the Hedge-bandit algorithm explores with probability $q = \left(\frac{n}{T}\right)^{1/3}$; see Algorithm \ref{alg:bandit-blackwell} for details. In each round $t\in [T]$ and  for each subproblem $i\in [10]$, the algorithm decides to explore with probability $q$, and continues to the next subproblem with probability $(1-q)$. Suppose that in round $t$, the first $(i-1)$ subproblems do not enter the exploration round, and the ranking so far is $\bm{\pi}^{(i-1)} = ([\bm{\pi}]_1,[\bm{\pi}]_2,\ldots,[\bm{\pi}]_{i-1})$. For position $i$, in the exploration round, we choose an item $j\in\mathcal{N}$ uniformly at random, then with probability $0.5$, the algorithm outputs the ranking $\bm{\pi}^{(i-1)}+ j\mathbf{e}_i = ([\bm{\pi}]_1,[\bm{\pi}]_2,\ldots,[\bm{\pi}]_{i-1},j)$, observes its market share $f\left(\bm{\pi}^{(i-1)}+ j\mathbf{e}_i\right)$, and feeds $2n/qf\left(\bm{\pi}^{(i-1)}+ j\mathbf{e}_i\right)$ back to the Hedge-bandit algorithm. Furthermore, with probability $0.5$, it outputs the ranking $\bm{\pi}^{(i-1)}= ([\bm{\pi}]_1,[\bm{\pi}]_2,\ldots,[\bm{\pi}]_{i-1})$, observes its market share $f\left(\bm{\pi}^{(i-1)}\right)$, and feeds $-2n/qf\left(\bm{\pi}^{(i-1)}\right)$ to the Hedge-bandit algorithm. If we consider an $n$-dimensional vector with the resulting randomized feedback in its $j^{\textrm{th}}$ coordinate and zero elsewhere, this vector is an unbiased estimate of the marginal market share vector
$
   [ f\left(\bm{\pi}^{(i-1)}+ j\mathbf{e}_i\right) - f\left(\bm{\pi}^{(i-1)}\right)]_{j\in[n]}.
$
This is based on the feedback stated in Equation \eqref{eq:sampling_device_product_ranking}, $n(\theta_j\mathbf{1}_n-\mathbf{e}_j)f(\bm{\pi}^{(i-1)}+j\mathbf{e}_i)$, which also incorporates $\bm{\theta}$, the current weight over possible items for subproblem $i$.




\begin{figure}
     \centering
     \begin{subfigure}[b]{\textwidth}
         \centering
         \includegraphics[width=\textwidth]{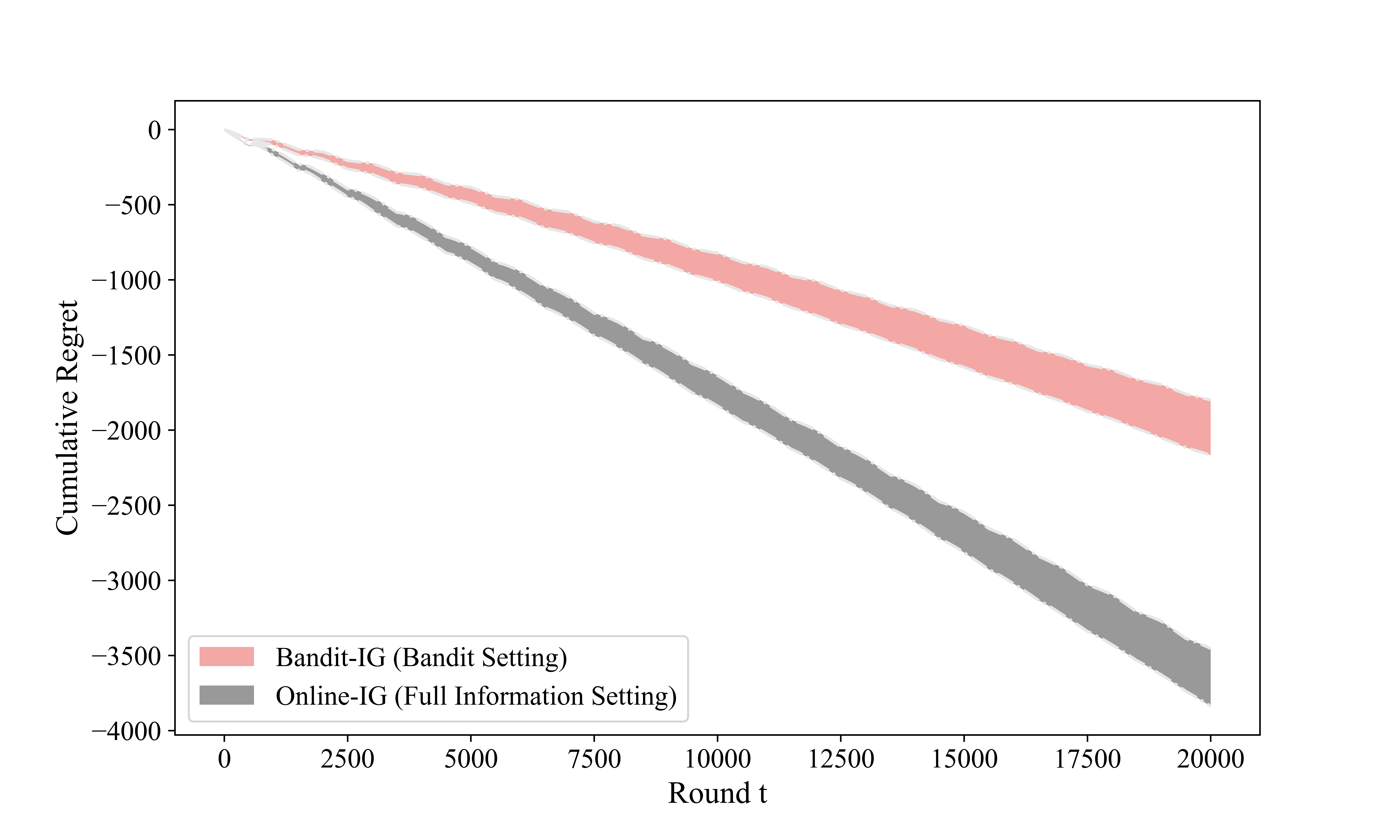}
         \caption{Product Ranking \label{fig:regret_BW_ranking}}
     \end{subfigure}
     \hfill
     \begin{subfigure}[b]{\textwidth}
         \centering
         \includegraphics[width=\textwidth]{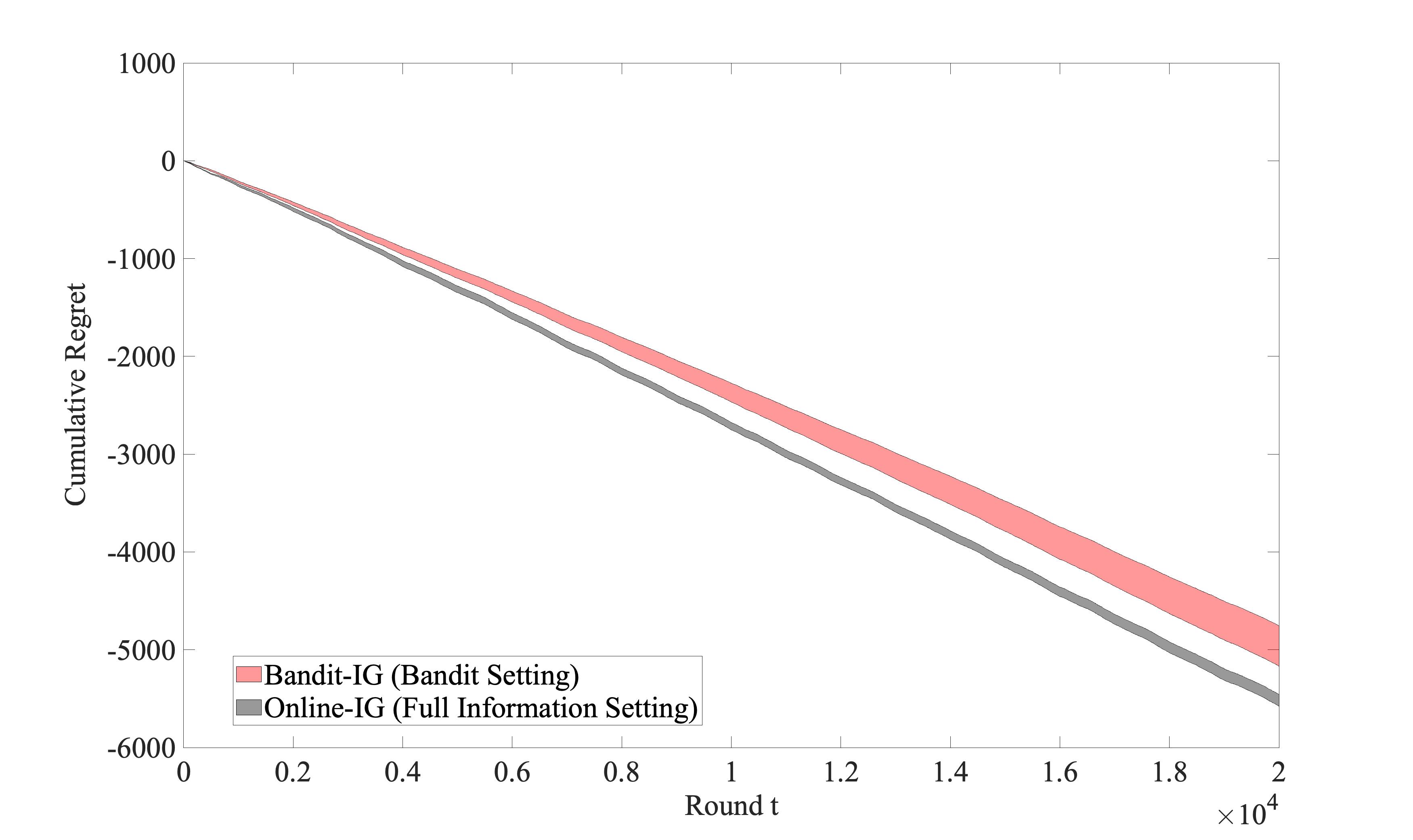}
         \caption{Reserve Price Optimization \label{fig:regret_bw_reserve}}
     \end{subfigure}
     \caption{Average cumulative  regret of \onlinemetatext{} and  $\banditmetaWoInput$ over time in the product ranking and reserve price optimization  problems. The width of the curves is equal to two times the standard error of the regret bounds across $50$ problem instances.}
     \end{figure}
\textbf{Regret.} 
Figure \ref{fig:regret_BW_ranking} shows the cumulative regret of $\onlinemetatext$ and $\banditmetaWoInput$ algorithms over time. In computing the regret of our algorithms, we use  $1/2$ of the optimal market share as the regret benchmark.  Note that the optimal market share  is the highest possible market share among all rankings assuming that we know all the parameters (weights and patience windows for all types) beforehand.\footnote{The optimal ranking is obtained by enumerating all possible rankings.} We choose the factor $1/2$ because the greedy algorithm that we based our online algorithm on is a $1/2$-approximation of the optimal value.  In Figure \ref{fig:regret_BW_ranking}, we consider $50$ instances where each problem has a unique set of parameters (i.e. $(\theta_{u,t},\mathbf{w}_{u,t})_{u\in[2],t\in[T]})$. For each instance, we take the average performance of both $\onlinemetatext$ and $\banditmetaWoInput$ over $10$ runs. Then, we calculate the regret of these algorithms using $1/2$ of the optimal values as the benchmark. As we observe from Figure \ref{fig:regret_BW_ranking}, $\banditmetaWoInput$ has a bigger regret than that of $\onlinemetatext$. In addition, the regret of both algorithms are negative, implying that they do better than half of the optimal market share.  

\subsection{Numerical Studies for Maximizing Multiple Reserves in Second Price Auctions} \label{sec:numerics:reserve}

\textbf{Simulation setting.} We consider a setting with $n= 3$ buyers and $T= 20,000$ auctions. In each round $t$, to generate 
the value of each buyer $i\in [3]$, denoted by $v_{i,t}$, we first generate an independent random  variable $\hat v_{i,t}$ form  a log-normal distribution with parameters $(\mu_{i,t}, \sigma_{i,t})$, where $\mu_{i,t} \in [0,0.5]$ and $\sigma_{i,t}\in [0, 0.2]$. We then set $v_{i,t} = \hat v_{i,t}\cdot x_{i,t} $, where $x_{i,t}\in \{0,1\}$. Here, $x_{i,t}$, which  is independently  drawn from a Bernoulli distribution with success probability of $p_{i,t}$,  models the fact that not all the buyers participate in all the auctions. In our setting,  parameters $(\mu_{i,t}, \sigma_{i,t}, p_{i,t})_{i\in [3]}$ remain fixed during an {episode} with the length of $500$ rounds, but keep changing across episodes, modeling the adversarial nature of the environment. In addition, within odd episodes, we set  $\mu_{i,t}$ and $\sigma_{i,t}$ to $\mu_{i,1}$ and $\sigma_{i,1}$, and similarly, within even episodes, we set  $\mu_{i,t}$ and $\sigma_{i,t}$ to $\mu_{i,2}$ and $\sigma_{i,2}$. Here,  we independently draw $\mu_{i,j}$ and $\sigma_{i,j}$, $i\in [n]$ and $j\in [2]$,   from a uniform distribution in the range of $[0,0.5]$  and $[0, 0.2]$, respectively.  Similarly, at the end of each episode, to set $p_{i,t}$, $i\in [n]$, we  draw a uniform random variable in the range of $[0,1]$. (Note that $p_{i,t}$ does not follow the alternating episodic pattern that $\mu_{i,t}$ and $\sigma_{i,t}$ follow.) Finally, for reserve prices, we assume that reserve prices belongs to set $\mathcal R = \{\rho_1, \ldots, \rho_{20}\}$, where $\rho_{1} = 0.1$, $\rho_{20}$  is the maximum buyers' value across all the auctions, and $\frac{\rho_i}{\rho_{i+1}}$, $i\in [19]$,  is equal to some constant $k$.

\textbf{Implementation of $\onlinemetatext$ and $\banditmetaWoInput$ algorithms.} To implement $\onlinemetatext$, we again use  the Hedge algorithm with the learning rate of $\epsilon_t = \sqrt{\frac{10}{t}}$ and the number of arms of $20$. (Recall the number of  feasible reserve prices is equal to $20$.)  At the end of each round, for every subproblem $i\in [3]$, we pass $q^{(i)} (r)$, $r\in \feasiblereserves$ to the Hedge algorithm as the reward of reserve price $r$, where $\feasiblereserves$ is the set of feasible reserve prices. (See the definition of $q^{(i)} (r)$ is Algorithm \ref{alg:mmr-single}.) 
For $\banditmetaWoInput$, we again rely on the Hedge algorithm to implement $\BanditBlackwellAlg$. 
The Hedge-bandit algorithm explores with probability  $q = (2\frac{n}{T})^{1/3}$. We now explain how the Hedge-bandit algorithm performs during an exploration round. Suppose that in some round $t$, we would like to decide about the reserve price of some buyer $i$ and the first $i-1$ buyers (subproblems) did not enter an exploration round. If buyer/subproblem $i$ enters an exploration round,  given the sampling device presented in  \Cref{apx:mmr-main-proof} (see Equation \eqref{eq:sample_device_reserve}), we first choose one of the reserve prices in $\feasiblereserves$ uniformly at random. Let us denote this reserve price by $\hat \rho$. Then, with probability $1/2$, we set the reserve price of all the buyers (including buyer $i$) equal to $\hat \rho$. Otherwise, we set the reserve price of all the buyers, except buyer $i$, equal to $\hat \rho$. The reserve price of buyer $i$ is then set to $0$.  In the former case, we pass $2m f(\hat \rho \mathbf{1}_n, \vibf_t)$ to the Hedge-bandit algorithm as the reward of arm $\hat \rho$. The reward of other arms is set to zero. Here,  $m=20$ is the number of feasible reserve prices and  $f(\hat \rho \mathbf{1}_n, \vibf_t)$ is the revenue of the second price auction with reserve prices $\hat \rho \mathbf{1}_n$ when buyers' values are $\vibf_t$. In the latter case, we pass $-2m f(\hat \rho (\mathbf{1}_n-\mathbf{e}_i), \vibf_t)$ to the Hedge-bandit algorithm as the reward of arm $\hat \rho$. The reward of other arms is set to zero.\footnote{Considering the fact that all the subproblems  need the value of $f(r\mathbf{1}_n, \vibf_t)$, for some $r\in \mathcal R$,  during their exploration rounds, in our implementation, we allow different subproblems to learn from each other. In particular, when subproblem  $i$ explores, with probability $1/(n+1)$, we choose reserve price of $\hat \rho\mathbf{1}_n$ and with probability $n/(n+1)$, we choose reserve price of $\hat \rho (\mathbf{1}_n-\mathbf{e}_i)$, where  $\hat \rho$ is randomly chosen reserve price by subproblem $i$. Then, when $\hat \rho\mathbf{1}_n$ is chosen, we update the weights of all the Hedge-bandit algorithms, not only those of subproblem $i$. To do so, we pass the reward of $(n+1)\cdot m\cdot  f(\hat \rho \mathbf{1}_n, \vibf_t)$ to the Hedge-bandit algorithms. When reserve price of  $\hat \rho (\mathbf{1}_n-\mathbf{e}_i)$ is chosen, we only update the Hedge-bandit algorithm of subproblem $i$ by passing the reward of $-\frac{n+1}{n}\cdot m \cdot f(\hat \rho (\mathbf{1}_n-\mathbf{e}_i), \vibf_t)$.   }



\textbf{Regret.} 
Figure \ref{fig:regret_bw_reserve} shows the cumulative regret of $\onlinemetatext$ and $\banditmetaWoInput$ over time.  We use $1/2$ of the optimal revenue as our benchmark in our regret calculation, where we compute the optimal revenue by enumerating over all possible vector of reserve prices. (Recall that our online  algorithms are transformation of the 1/2-approximate  greedy algorithm of  \cite{roughgarden2019minimizing}.) In Figure \ref{fig:regret_bw_reserve}, we consider $50$ problem instances where each problem instance has a unique set of parameters (i.e., $(\mu_{i,t}, \sigma_{i,t}, p_{i,t})_{i\in [n], t\in [T]}$)), as well as, auction bids/values generated from these parameters. For each problem instance, we run $\onlinemetatext$ and $\banditmetaWoInput$ $10$ times and compute their average performance. 
We observe that $\onlinemetatext$ outperforms $\banditmetaWoInput$, as expected. Further, similar to our results for the product ranking problem, the regret of both algorithms is negative.} 

%% file: tex/appendix.tex
\section{Proofs and Remarks of \texorpdfstring{\Cref{sec:blackwell}}{}}
\subsection{Equivalent criteria for approachability}
\label{sec:apx-blackwellequivalent}
Interestingly, there are other structural conditions that are equivalent to approachability. For example, the original proof of the Blackwell approachability theorem~\citep{blackwell1956analog} uses a condition called ``halfspace-satisfiability''. The following proposition summarizes all the known equivalences. 
\begin{proposition}[Satisfiable/Halfspace-Satisfiable/Response-Satisfiable~\citep{abernethy2011blackwell}] \label{def:satisfy-app}
The following conditions are all equivalent to the approachability condition (Definition~\ref{def:blackwell-approach}):
\begin{enumerate}
    \item A target set $S$ is satisfiable in the Blackwell sequential game $\blackwellsequentialgame$ if there exists a player 1's action $\bm{x} \in \algspaceB$ such that for every player 2's action $\bm{y} \in \advspaceB$, the vector payoff falls into the target set, that is $\pbf(\bm{x}, \bm{y}) \in S$.
    \item  A target set $S$ is halfspace-satisfiable in the Blackwell sequential game $\blackwellsequentialgame$ if for every halfspace $H \supseteq S$, $H$ is satisfiable.
    \item  A target set $S$ is response-satisfiable in the Blackwell sequential game $\blackwellsequentialgame$ if for every player 2's action $\bm{y} \in \advspaceB$, there exists a player 1's action $\bm{x} \in \algspaceB$ such that the vector payoff falls into the target set, that is $\pbf(\bm{x}, \bm{y}) \in S$.
\end{enumerate}
\end{proposition}

\subsection{Proof of \texorpdfstring{\Cref{def:blackwell-thm}}{}}
\label{sec:apx-blackwell-proof}
\proof{\emph{Proof.}} The proof for the only if direction relies on the fact that the $\ell_\infty$-distance between the average payoff and $S$ is vanishing as $T\rightarrow{}+\infty$ since $S$ is $o(1)$-approachable. Suppose that $S$ is not response satisfiable, then there exists player 2's action $\ybf_0\in\advspaceB$ such that for every player 1's action $\xbf\in\algspaceB,$ the payoff $\pbf(\xbf,\ybf_0)$ is not in $S$. Consider the set $U := \{\pbf(\xbf,\ybf_0):\xbf\in\algspaceB\}$. Because the payoff $\pbf$ is biaffine and $\algspaceB$ is convex and compact, so is $U,$ hence $\underset{\mathbf{u}\in U}{\inf} d_\infty(\mathbf{u},S) = d_\infty(\pbf(\underline{\xbf},\ybf_0),S)$ for some $\underline{\xbf}\in\algspaceB$. As $\pbf(\underline{\xbf},\ybf_0)\notin S,$ $\beta = d_\infty(\pbf(\underline{\xbf},\ybf_0),S)>0.$ When player 2 always plays $\ybf_0,$ we know that the $\ell_\infty$ distance between the average payoff and $S$ should converge to zero as $S$ is $o(1)$-approachable. At the same time, it is at least $\beta$, a contradiction.

To prove the if direction, we first show a reduction from Blackwell approachability to Online Linear Optimization (OLO) by showing that we can upper bound the $\ell_\infty$ distance between the average payoff and the target set in a Blackwell approachability problem with the regret of the corresponding OLO instance. Then, we bound the regret of the OLO problem from above in terms of the $\ell_\infty$ norm of the payoff $D(\pbf)$ (because of our desired bound), the number of rounds $T,$ and the dimension of the payoff function $d$. We assume that $S$ is a cone throughout the proof, which is not an issue because we can always lift a convex set to a cone in one dimension higher while not perturbing the distances by more than a factor of $2.$

\paragraph{Blackwell approachability reduces to OLO.} In an OLO problem, a player is given a compact convex decision set $\mathcal{K}\subset\mathbb{R}^d,$ and have to decide on a sequence of actions $\wibf_1,\wibf_2,\ldots,\wibf_T\in\mathcal{K}$. In round $t,$ after the player decides on an action $\wibf_t,$ Nature reveals a loss vector $\lbf_t$ and the player pays $\langle\lbf_t,\wibf_t\rangle.$ The player observes the loss vector $\lbf_t$ in each round (full-information setting) and aims to minimize his cost. We want to construct a learning algorithm $\mathcal{L}$, such that, for any sequence of loss vectors $\lbf_1,\lbf_2,\ldots,\lbf_T\in\mathbb{R}^d$, the algorithm outputs $\wibf_1,\wibf_2,\ldots,\wibf_T\in\mathcal{K}$ that attains a small regret, i.e. $\sum_{t=1}^T\langle\lbf_t,\wibf_t\rangle - \underset{\wibf\in\mathcal{K}}{\min}\sum_{t=1}^T\langle\lbf_t,\wibf\rangle \leq o(T).$ \cite{abernethy2011blackwell} show that we can efficiently obtain an algorithm for a Blackwell approachability problem from an algorithm for its corresponding OLO problem. Specifically, we have the following lemma.
\begin{lemma}(\cite{abernethy2011blackwell})
\label{lemma:BW-OLO}
Given a Blackwell instance $\blackwellsequentialgame$, and a cone $S$ such that $S$ is response-satisfiable, we can construct an OLO problem with $\mathcal{K} = S^\circ\cap B_2(1)$\footnote{$S^\circ$ is the polar cone of $S$, i.e. $S^\circ :=\{\mathbf{s}\in\mathbb{R}^d:\langle\mathbf{s},\mathbf{x}\rangle\leq 0\text{ for all }\mathbf{x}\in S\}$, and $B_2(1)$ is a Euclidian ball with radius $1$, i.e. $B_2(1) = \{\wibf\in\mathbb{R}^d:\lVert\wibf\rVert_2\leq1\}$.} and $\lbf_t = -\pbf(\xbf_t,\ybf_t)$ for all $t,$ such that, if the OLO learning algorithm returns $\wibf_t$ in round $t,$ we can convert it into $\xbf_t\in\mathcal{X}$ where
\begin{equation*}
    d_2\left(\frac1T\sum_{t=1}^T\pbf(\xbf_t,\ybf_t),S\right) \leq \frac1T\left(\sum_{t=1}^T\langle\lbf_t,\mathbf{w}_t\rangle - \underset{\wibf\in\mathcal{K}}{\min}\sum_{t=1}^T\langle\lbf_t,\wibf\rangle\right).
\end{equation*}
\end{lemma}
\proof{\emph{Proof of \Cref{lemma:BW-OLO}.}}
This lemma was proved in \cite{abernethy2011blackwell}, but we include the proof here for completion. Notice that, for any $\xibf\in\mathbb{R}^d$ and convex cone $S\subset\mathbb{R}^d,$ the distance from $\xibf$ to $S$ can be written as
\begin{equation}
\label{eq:distance}
    d_\infty(\xibf,S) = \underset{\wibf\in S^\circ,\lVert\wibf\rVert\leq 1}{\max}~ \langle\wibf,\xibf\rangle
\end{equation}
because
\begin{equation*}
    d_\infty(\xibf,S)= \lVert\xibf - \pi_S(\xibf)\rVert_2\geq \lVert\wibf\rVert\lVert\xibf-\pi_S(\xibf)\rVert\geq\langle\wibf,\xibf-\pi_S(\xibf)\rangle\geq\langle\wibf,\xibf\rangle,
\end{equation*}
where $\pi_S(\xibf)$ denotes the projection of $\xibf$ onto $S$, and when $\wibf = \frac{\xibf - \pi_S(\xibf)}{\lVert\xibf - \pi_S(\xibf)\rVert_2},$ we have equality, i.e. $\langle\wibf,\xibf\rangle = d_2(\xibf,S).$ To construct a mapping from the output of the OLO algorithm $\wibf_t$ to $\xbf_t$ for the Blackwell game, we utilize the halfspace oracle for the Blackwell problem (see \Cref{def:satisfy-app}). Specifically, we pick $\xbf_t$ such that $\pbf(\xbf_t,\ybf_t)\in H_{\wibf_t}$ for all $\ybf\in\advspaceB$, where $H_{\wibf_t} = \{\xbf:\langle\wibf_t,\xbf\rangle\leq 0\}$ is a halfspace that contains $S$ ($H_{\wibf_t}$ contains $S$ because its normal, $\wibf_t$, is in $S^\circ$). This gives us the following guarantee
\begin{align}
d_2\left(\frac1T\sum_{t=1}^T\pbf(\xbf_t,\ybf_t),S\right) 
&\overset{(1)}{=} \underset{\wibf\in\mathcal{K}}{\max}\left\langle\frac1T\sum_{t=1}^T\pbf(\xbf_t,\ybf_t),\wibf\right\rangle = \frac1T\underset{\wibf\in\mathcal{K}}{\max}\left(-\sum_{t=1}^T\langle\lbf_t,\wibf\rangle\right)\label{eq:unbiased}\\
&\overset{(2)}{\leq}\frac1T\left(\sum_{t=1}^T\langle -\pbf(\xbf_t,\ybf_t),\wibf_t\rangle - \underset{\wibf\in\mathcal{K}}{\min}\sum_{t=1}^T\langle\lbf_t,\wibf\rangle\right)\nonumber\\
&\overset{(3)}{=} \frac1T\left(\sum_{t=1}^T\langle\lbf_t,\wibf_t\rangle - \underset{\wibf\in\mathcal{K}}{\min}\sum_{t=1}^T\langle\lbf_t,\wibf\rangle\right).\nonumber
\end{align}

Here, Equality~$(1)$ follows from \Cref{eq:distance}, Inequality~$(2)$ holds because $\pbf(\xbf_t,\ybf_t)\in H_{\wibf_t},$ and Equality~$(3)$ holds from our definition of $\lbf_t.$

\[\myqed\]
\endproof{}
As a corollary, since for any $\xibf\in\mathbb{R}^d$ and $S\subseteq\mathbb{R}^d,$ the $\ell_\infty$ distance is always less than equal to the $\ell_2$ distance, i.e., $d_\infty(\xibf,S)\leq d_2(\xibf,S),$ we obtain
\begin{equation*}
    d_\infty\left(\frac1T\sum_{t=1}^T\pbf(\xbf_t,\ybf_t),S\right) \leq d_2\left(\frac1T\sum_{t=1}^T\pbf(\xbf_t,\ybf_t),S\right) \leq \frac1T\left(\sum_{t=1}^T\langle\lbf_t,\mathbf{w}_t\rangle - \underset{\wibf\in\mathcal{K}}{\min}\sum_{t=1}^T\langle\lbf_t,\wibf\rangle\right).
\end{equation*}

\paragraph{OLO regret upper-bound with Follow-the-Regularized-Leader algorithm.} To obtain the upper-bound on the regret of an OLO problems in terms of the $\ell_\infty$ norm of its losses, we apply the Follow-the-Regularized-Leader algorithm with a $\mu$-strongly convex\footnote{A $\mu$-strongly convex function $f$ with respect to the $\ell_q$ norm is a differentiable function that satisfies $f(\yibf)\geq f(\xibf)+ \nabla\left(f(\xibf)^T\right)(\yibf-\xibf) + \tfrac\mu2\lVert\yibf-\xibf\rVert_q^2$ for some $\mu>0.$ } regularizer with respect to the $\ell_1$ norm. We use a regularizer with respect to the $\ell_1$ norm, the dual of the $\ell_\infty$ norm, because of the bound structure of the algorithm as stated in \Cref{lemma:OLO}. We elaborate further in the following lemmas.

\begin{lemma}(\cite{shalev2012online})
\label{lemma:OLO}
Consider an OLO problem on a convex and compact decision space $\mathcal{K}\subseteq\mathbb{R}^d.$ Applying Follow-the-Regularized-Leader algorithm with a regularizer $R,$ where $R:\mathbb{R}^d\rightarrow\mathbb{R}$ is a $\mu$-strongly convex function with respect to some norm $\lVert\cdot\rVert$ for $\mu>0$, implies
\begin{equation*}
    \frac1T\left(\sum_{t=1}^T\langle\lbf_t,\mathbf{w}_t\rangle - \underset{\wibf\in\mathcal{K}}{\min}\sum_{t=1}^T\langle\lbf_t,\wibf\rangle\right) \leq\bigO{CB^{1/2}\mu^{-1/2}T^{-1/2}},
\end{equation*}
where $B>0$ upper bounds the function $R,$ $C>0$ upper bounds the dual norm of the loss vector $\lVert\lbf\rVert_*$\footnote{If $\lVert\cdot\rVert$ is a norm in $\mathbb{R}^d,$ the dual norm $\lVert\cdot\rVert_*$ of $\lVert\cdot\rVert$ is defined as $\lVert\wibf\rVert_* = \sup\{\wibf^T\xibf|~\lVert\xibf\rVert\leq 1\}$.}, and $T$ is the number of rounds.
\end{lemma}

\begin{lemma}(\cite{shalev2007online})
\label{lemma:funR}
For $q\in (1,2),$ the function $f:\mathbb{R}^d\rightarrow\mathbb{R}$ defined as $f(\xibf) = \frac1{2(q-1)}\lVert\xibf\rVert_q^2$ is strongly convex with respect to the $\ell_q$ norm over $\mathbb{R}^d$. Recall that the $\ell_q$ norm is defined as $\lVert\xibf\rVert_q = \left(x_1^q+x_w^q+\ldots+x_d^q\right)^{1/q}$ for $\xibf\in\mathbb{R}^d.$
\end{lemma}

To get a bound from \Cref{lemma:OLO} that depends on the upper-bound of the $\ell_\infty$ norm of the loss vectors, we want a regularizer $R$ that is $\mu$-strongly convex w.r.t the $\ell_1$ norm for some $\mu >0$ (to be determined later). However, the function from \Cref{lemma:funR} does not work for $q = 1.$ To solve this, we set $q$ to be greater than $1$, $q = \frac{\log(d)}{\log(d)-1}$ in particular, then bound the $\ell_q$ norm from above using the $\ell_1$ norm. Specifically, setting $R(\xibf) = \frac{1}{2}\lVert\xibf\rVert_q^2$ and $q =\frac{\log(d)}{\log(d)-1}$ we have
\begin{equation*}
\begin{aligned}
R(\xibf) &= (q-1)\cdot \frac{1}{2(q-1)}\lVert\xibf\rVert_q^2 \\
&\overset{(1)}{\geq} (q-1)\frac{1}{2(q-1)}\lVert\xibf\rVert_q^2 + (q-1)\nabla\left( \frac{1}{2(q-1)}\lVert\xibf\rVert_q^2\right)^T(\yibf-\xibf) +(q-1)\cdot\frac{1}{2}\cdot\lVert\yibf-\xibf\rVert_q^2\\
&\overset{(2)}{\geq} (q-1)\frac{1}{2(q-1)}\lVert\xibf\rVert_q^2 + (q-1)\nabla\left( \frac{1}{2(q-1)}\lVert\xibf\rVert_q^2\right)^T(\yibf-\xibf) +(q-1)\cdot\frac{1}{2}\cdot\frac{\lVert\yibf-\xibf\rVert_1^2}{3}\\
&= R(\xibf) + R(\xibf)^T(\yibf-\xibf) + \frac{(q-1)/3}{2}\lVert\yibf-\xibf\rVert_1^2,
\end{aligned}    
\end{equation*}
where $\mu = \frac{q-1}{3} =\frac{\frac{\log(d)}{\log(d)-1}-1}{3} = \frac{1}{3\log(d)}.$ So, the function $R$ is $\frac{1}{3\log(d)}$-strongly convex with respect to the $\ell_1$ norm. Furthermore, Inequality~$(1)$ follows from \Cref{lemma:funR} and Inequality~$(2)$ holds because $\lVert\wibf\rVert_1/3\leq\lVert\wibf\rVert_q$ for any $\wibf\in\mathbb{R}^d.$

Consequently, by constructing an OLO problem with $\mathcal{K} = S^\circ\cap B_{2}(1)$ and $\lbf_t = -p(\xbf_t,\ybf_t)$ on each round, applying the Follow-the-Regularized-Leader algorithm with regularizer $R(\wibf) = \frac12\lVert\wibf\rVert_q^2$ to the OLO problem, and converting $\wibf_t$ to $\xibf_t$ in each round, we obtain
\begin{equation*}
\begin{aligned}
d_\infty\left(\frac1T\sum_{t=1}^T\pbf(\xbf_t,\ybf_t),S\right) \leq \frac1T\left(\sum_{t=1}^T\langle\lbf_t,\mathbf{w}_t\rangle - \underset{\wibf\in\mathcal{K}}{\min}\sum_{t=1}^T\langle\lbf_t,\wibf\rangle\right)\leq\bigO{D(\pbf)\log(d)^{1/2}T^{-1/2}}.
\end{aligned}
\end{equation*}
The last inequality follows from applying \Cref{lemma:OLO} to the OLO problem corresponding to the Blackwell game with $B=1$ and $C = D(\pbf)$. Notice that for any $\wibf\in\mathcal{K},$ $\lVert\wibf\rVert_2\leq 1 = B$ because we set $\mathcal{K} = S^{\circ}\cap B_2(1)$ in \Cref{lemma:BW-OLO}. Furthermore, since we set $\lbf_t = \pbf(\xbf_t,\ybf_t),$ we have $\lVert\lbf_t\rVert_\infty = \lVert\pbf(\xbf_t,\ybf_t)\rVert_\infty\leq D(\pbf)$ by definition.
\[\myqed\]
\endproof{}

\section{Proofs and Remarks of Section~\ref{sec:framework}}
\label{sec:appendix-framework}

\begin{example}[Non-Robust Greedy Algorithm]
\label{example:non-robust}
  In the shortest path tree problem, we are given an undirected graph $G = (V, E)$ along with a root node $u$ and edge weights $\{w_{uv}\}_{(u, v) \in E}$. We want to compute a spanning tree of $G$ such that for all vertices $v \in V$, the distance to the root in the tree, $\text{dist}_T(u, v)$, equals the distance to the root in the original graph, $\text{dist}_G(u, v)$. This problem can be solved by a greedy algorithm which runs Dijkstra's algorithm from $u$ and then for each node $v \not = u$ chooses a parent $p \in \text{neighborhood}(v)$ with the smallest $w_{vp} + \text{dist}_G(p, u)$. Suppose that we want to solve the online problem where the $G$ and $u$ are fixed over all rounds but the edge weights are chosen by an adversary.
  
  This can be translated into the language of our meta-algorithm as follows. The feasible region is to choose a parent for every non-root vertex ($\constraint = \prod_{v \in V \setminus \{u\}} \text{neighborhood}(v)$). The adversary's function space is to choose (bounded) weights ($\funcspace \cong (0, 1]^E$), and the cost of a chosen set of edges that we aim to minimize is the average distance from a random vertex to $u.$ For each of our $|V|$ subroutines, the parameter space is to choose a distribution for the parent vertex ($\paramspace = \Delta(\text{neighborhood}(v))$)\footnote{Our framework is amenable to the parameter space depending on the iteration $i$.}. The (one-dimensional) payoff vector is the shortest path from $v$ to $u$ through that parent $p$ ($\pay(\param, \zbf, \{w_{uv}\}) = \E_{p\sim\param}\left[w_{vp} + \text{dist}_G(p, u)\right]$), where $\param$ is the probability of choosing a vertex among $v$'s neighbors as parent.
  
  Managing to perfectly minimize the one-dimensional payoff vector at each iteration results in a shortest path tree and therefore the best possible objective value. However, if the local choices deviate from their optimal values, then we can create cycles which result in infinite objective value.
  
  For example, consider the clique on $V = \{1,2,3\}$, where we want a shortest path tree to the root node $u = 1$. When the weights are $w_{12} = 0.25, w_{13} = 1.0, w_{23} = 0.5$, then it would be best for node three to first take edge $(2, 3)$. If we swap the role of nodes two and three, $w_{12} = 1.0, w_{13} = 0.25, w_{23}$, then it would be best for node two to first take edge $(2, 3)$. When our subroutines for nodes two and three make simultaneous decisions without actually seeing the input, they could easily both choose edge $(2, 3)$, yielding an invalid shortest path tree and making it impossible to get from either node to the root. This global issue can't be expressed as local utilities, so the algorithm is not robust in the sense that is needed to apply our framework.
\end{example}


\section{Proofs and Remarks of Section~\ref{sec:bandit-info}}

\subsection{Proof of \texorpdfstring{\Cref{thm:bandit-blackwell-thm}}{}}
\label{sec:appendix-bandit}

In this section, we complete the proof of \Cref{thm:bandit-blackwell-thm}, which is restated below for convenience.

\banditblackwellthm*

\proof{\emph{Proof.}}
  The only if direction is proved in the sketch. To prove the if direction and second part of Theorem~\ref{thm:bandit-blackwell-thm}, we propose an algorithm $\BanditBlackwellAlg{}$ that is parameterized by an exploration  probability $q \in (0, 1]$ (\Cref{alg:bandit-blackwell}). We later choose $q$ to be $\schange{{D(\pbf)}^{-2/3}\diameter{\hat{\pbf}}^{2/3}(\log d)^{1/3}}T^{-1/3}$ to balance terms in our regret upper-bound. In each round $t$, this algorithm outputs a move $\bm{x}_t \in \advspaceB$ as well as whether to explore  $\pi_t \in \{\textsc{Yes},\textsc{No}\}$, and then receives an unbiased estimate of the resulting payoff $\hat{\pbf}(\bm{x}_t, \bm{y}_t)$ based on both players' actions if it picks to explore. It also maintains a (full-information) Blackwell algorithm $\BlackwellAlg{}$. In each round $t = 1, 2, ..., T$, our algorithm follows the last suggested action by $\BlackwellAlg{}$ to generate a move $\xbf_t$. Note that this move will be exactly the same as the previous round, if the algorithm chose to not explore in the previous round. Our algorithm then decides to either explore with probability $q$ or not explore with probability $1 - q$. If it explores, then it receives an unbiased estimator, $\hat{\pbf}$ of the current vector payoff $\pbf(\xbf_t,\ybf_t)$, and passes a scaled version $\hat{\pbf} / q$ on to $\BlackwellAlg{}$. If it does not explore, then it rewinds the state of $\BlackwellAlg{}$ to the beginning of the current round. Our goal here is to show that under  algorithm $\BanditBlackwellAlg{}$,  $\standarddistance{\frac1T \sum_{t=1}^T {\pbf}(\bm{x}_t, \bm{y}_t)}{S}$ plus  exploring penalty term $\E \left[ \tfrac1T \schange{D(\pbf)} \cdot  \text{(\# explore}) \right]$ is $\bigO{\schange{{D(\pbf)}^{1/3}\diameter{\hat{\pbf}}^{2/3}(\log d)^{1/3}}T^{-1/3}}$. 
  
  We start with bounding the first term $\standarddistance{\frac1T \sum_{t=1}^T {\pbf}(\bm{x}_t, \bm{y}_t)}{S}$ as a function of $\hat{\Pi}_T$, which is  the time-averaged rescaled estimated payoffs from the rounds that we explore in $1, 2, \ldots, T$: 
  \begin{align*}
    \hat{\Pi}_T &\triangleq \frac1T \sum_{t=1}^T \frac1q \hat{\pbf}(\bm{x}_t, \bm{y}_t) \indicator{explore in round $t$}\,.
  \end{align*}
  Specifically, we have 
    \begin{align*}
    \standarddistance{\frac1T \sum_{t=1}^T {\pbf}(\bm{x}_t, \bm{y}_t)}{S}
      &= \standarddistance{\frac1T \sum_{t=1}^T \E \left[ \frac1q \hat{\pbf}(\bm{x}_t, \bm{y}_t) \indicator{explore in round $t$} \right]}{S} \\
      &\le \E \left[ \standarddistance{\hat{\Pi}_T}{S} \right]\,,
  \end{align*}
  where the equality follows 
because $\E\left[ \frac1q \hat{\pbf}(\bm{x}_t, \bm{y}_t) \indicator{explore in round $t$}\right]=\frac{q}{q}\E\left[ \hat{\pbf}(\bm{x}_t, \bm{y}_t)\right]=\pbf(\xbf_t,\ybf_t)$ as $\hat \pbf$ is an unbiased estimator for $\pbf$, and the inequality is obtained  by applying Jensen's inequality to the convex $\ell_\infty$ distance function. 
We next show that if we explore with probability  $q,$ $\E \left[ \standarddistance{\hat{\Pi}_T}{S} \right]\le \bigO{\diameter{\hat{\pbf}}\log(d)^{1/2} (qT)^{-1/2}}$. 
Then, observe that the exploring penalty term $\E \left[ \tfrac1T \schange{D(\pbf)} \cdot  \text{(\# explore}) \right]$ equals $\schange{D(\pbf)} q$. Our choice of exploring  probability $q = \schange{{D(\pbf)}^{-2/3}\diameter{\hat{\pbf}}^{2/3}}\log(d)^{1/3}T^{-1/3}$ makes the two terms equal to $O(\schange{{D(\pbf)}^{1/3}\diameter{\hat{\pbf}}^{2/3}(\log d)^{1/3}} T^{-1/3})$, and gives us the desired bound.

   To see why  $\E \left[ \standarddistance{\hat{\Pi}_T}{S} \right]\le \bigO{\diameter{\hat{\pbf}}\log(d)^{1/2} (qT)^{-1/2}}$ when we explore with probability $q$, let $M$ be a random variable equal to the number of rounds we explore and $(\tau_1, \tau_2, \cdots, \tau_M)$ be the rounds that we explore. Note that $M \sim \text{Binomial}(T, q)$. By applying the law of total expectation, we have:
   $$
    \E \left[ \standarddistance{\hat{\Pi}_T}{S} \right]=\sum_{m=0}^T\E \left[ \standarddistance{\hat{\Pi}_T}{S} \;\middle|\; M=m\right]\Pr\left[M=m\right]~.
   $$
   
   We provide an upper bound on each term in the above summation separately. First, we handle the $M=m =0 $ case by noting $\hat{\Pi}_T=\mathbf{0}$, hence the distance from $S$ is bounded by $\diameter{\hat{\pbf}}$ in this case. Moreover, this event occurs with probability $(1 - q)^T$. \schange{See that
   \begin{align*}
       (1-q)^T &= (1-q)^{q^{-1}\cdot q\cdot T}\leq (1/e)^{qT}\le \bigO{(qT)^{-1/2}},
   \end{align*}
  and therefore $\E \left[ \standarddistance{\hat{\Pi}_T}{S} \;\middle|\; M=0\right]\Pr\left[M=0\right]\leq \bigO{\diameter{\hat{\pbf}}\log(d)^{1/2}{(qT)}^{-1/2}}$}.
   
  Now fix some $M=m \not = 0$. Assuming that $S$ is a cone (we can always lift the convex set $S$ to a cone in one dimension higher as shown in \cite{abernethy2011blackwell}), our full-information Blackwell algorithm $\BlackwellAlg{}$, who receives ``fake payoffs'' $\{\hat\pbf(\xbf_{\tau_i},\ybf_{\tau_i})\}_{i=1}^{m}$ with a diameter of $\frac{1}{q} \diameter{\hat\pbf}$,  guarantees that:
  \begin{align}
    \E \left[ \standarddistance{\hat{\Pi}_T}{S} \;\middle|\; M = m \right]
      &= \E \left[ \frac{M}{T} \cdot \standarddistance{\frac{T}{M} \hat{\Pi}_T}{S} \;\middle|\; M = m \right] \nonumber\\
      &= \frac{m}{T} \E \left[ \standarddistance{\frac1M \sum_{i=1}^M \frac1q \hat{\pbf}(\bm{x}_{\tau_i}, \bm{y}_{\tau_i})}{S} \;\middle|\; M = m \right]\nonumber \\
      &\le \frac{m}{T} \cdot \bigO{\frac{1}{q} \diameter{\hat{\pbf}}\log(d)^{1/2} m^{-1/2}}\nonumber \\
      &= \bigO{\frac{1}{qT} \diameter{\hat{\pbf}}\log(d)^{1/2} m^{1/2}}\label{eq:final-bound-q}\,,
  \end{align} where the expectation is taken w.r.t.  the randomness  in $\hat \pbf$ and the inequality holds because $S$ is response-satisfiable in the Blackwell game $\blackwellsequentialgame$ and $\hat \pbf$ is an unbiased estimator of $\pbf$. To be more clear why the above inequality holds, note that set $S$ is response-satisfiable in the Blackwell game $\blackwellsequentialgame$, and is not necessarily response-satisfiable if we replace $\pbf$ with $\hat\pbf$. However, by (i) following exactly the same steps as in proof of \Cref{def:blackwell-thm} (\Cref{sec:apx-blackwell-proof} in the appendix) to reduce Blackwell approachability to online linear optimization for rounds $\{\tau_i\}_{i=1}^M$, (ii) plugging $\hat\pbf$ as the vector payoff of each round and using $\lbf_i=-\hat\pbf(\xbf_{\tau_i},\ybf_{\tau_i})$ as the loss function in the online linear optimization, and then (iii) using the fact that $\hat \pbf$ is an unbiased estimator for $\pbf$ and $S$ is response-satisfiable w.r.t. payoffs $\pbf$, we can obtain exactly the same approachability bound in expectation as if $S$ was response-satisfiable w.r.t. payoffs $\hat\pbf$. To see this, consider the chain of inequalities~\eqref{eq:unbiased} in the proof of \Cref{def:blackwell-thm} in \Cref{sec:apx-blackwell-proof}, tailored to our problem, and take an expectation w.r.t. the randomness in $\hat\pbf$. We have:
\begin{align*}
 \E \left[ \standarddistance{\frac1M \sum_{i=1}^M  \hat{\pbf}(\bm{x}_{\tau_i}, \bm{y}_{\tau_i})}{S} \;\middle|\; \{\tau_i\}_{i=1}^M \right]&\leq \E\left[\underset{\wibf\in\mathcal{K}}{\max}\left\langle\frac1M\sum_{i=1}^M\hat\pbf(\xbf_{\tau_i},\ybf_{\tau_i}),\wibf\right\rangle\bigg|\{\tau_i\}_{i=1}^M\right]\\
 &= \E\left[\frac1M\underset{\wibf\in\mathcal{K}}{\max}\left(-\sum_{i=1}^M\langle\lbf_i,\wibf\rangle\right)\bigg|\{\tau_i\}_{i=1}^M\right]\\
&\overset{(2)}{\leq}\frac1M\left(\sum_{i=1}^M\langle -\pbf(\xbf_{\tau_i},\ybf_{\tau_i}),\wibf_i\rangle -\E\left[ \underset{\wibf\in\mathcal{K}}{\min}\sum_{i=1}^M\langle\lbf_i,\wibf\rangle\bigg|\{\tau_i\}_{i=1}^M\right]\right)\nonumber\\
&\overset{(3)}{=} \frac1M\E\left[\sum_{i=1}^M\langle\lbf_i,\wibf_i\rangle - \underset{\wibf\in\mathcal{K}}{\min}\sum_{i=1}^M\langle\lbf_i,\wibf\rangle\bigg| \{\tau_i\}_{i=1}^M\right].\nonumber
\end{align*}
This time, Inequality~(2) holds as before, because $S$ is response-satisfiabe w.r.t. payoffs $\pbf$ (and hence half-space satisfiable when using $\wibf_t$ as the normal of the half-space), but Equality~(3) holds because: $$-\langle\pbf(\xbf_{\tau_i},\ybf_{\tau_i}),\wibf_i\rangle=-\E\left[\langle\hat\pbf(\xbf_{\tau_i},\ybf_{\tau_i}),\wibf_i\rangle|\{\tau_i\}_{i=1}^M\right]=-\E\left[\langle\lbf_i,\wibf_i\rangle|\{\tau_i\}_{i=1}^M\right]$$
Note that expectation is conditioned on $\{\tau_i\}_{i=1}^M$, but only we use a universal upper-bound on the last term (regret of online linear optimization) that is a function of $M$, so we can change the conditioning on only $M$.



  We now use the bound in \eqref{eq:final-bound-q} for $q=\schange{{D(\pbf)}^{-2/3}\diameter{\hat{\pbf}}^{2/3}}\log(d)^{1/3}T^{-1/3}$, and then use Jensen's inequality applied to the (concave) square-root function.
  \begin{align*}
    \E \left[ \standarddistance{\hat{\Pi}_T}{S} + \frac1T \schange{D(\pbf)}\cdot (\text{\# explore}) \right]
      &= \E_{m \sim \text{Binomial}(T, q)} \left[ \E \left[ \standarddistance{\hat{\Pi}_T}{S} \;\middle|\; M = m \right] \right] + \bigO{\schange{{D(\pbf)}q}}\\
      &\le \E_{m \sim \text{Binomial}(T, q)} \left[ \bigO{\frac{1}{qT} \diameter{\hat{\pbf}}\log(d)^{1/2} m^{1/2}} \right] + \bigO{\schange{D(\pbf)q}} \\
      &\le \bigO{\frac{1}{qT} \diameter{\hat{\pbf}}\log(d)^{1/2} (Tq)^{1/2}} + \bigO{\schange{D(\pbf)q}} \\
      &= \bigO{\diameter{\hat{\pbf}}\log(d)^{1/2} (qT)^{-1/2}} + \bigO{\schange{D(\pbf)q}} \\
      &= \bigO{\schange{{D(\pbf)}^{1/3}D(\hat{\pbf})^{2/3}(\log d)^{1/3}T^{-1/3}}}\,.
  \end{align*}

The last inequality is the desired result.
\[\myqed\]
\endproof{}

\subsection{Proof of \texorpdfstring{\Cref{thm:banditILO}}{}}
\label{apx:banditmeta-proof}
\proof{\emph{Proof.}}
The function $\hat\pbf$ is an unbiased estimator for $\pbf$ (due to the bandit Blackwell reduciblity), so $(\algspaceB,\advspaceB,\pbf,\hat\pbf)$ is a valid instance of the bandit Blackwell sequential game. Moreover, our target set $S$ is the $\payoffdimension$-dimensional positive orthant. Therefore, there exists a polynomial-time separation oracle for set $S$. Set $S$ is also response-satisfiable (due to bandit Blackwell reduciblity). \schange{Thus, there exists a polynomial time online algorithm $\BanditBlackwellAlg{}$ that guarantees the bandit approachability upper-bound $\bigO{{D(\pbf)}^{1/3}{D(\hat{\pbf})}^{2/3}\left(\log(\payoffdimension)\right)^{1/3} T^{2/3}}$, established in \Cref{thm:bandit-blackwell-thm}, for each of the bandit Blackwell instances corresponding to the $N$ different subproblems.}

Consider a subproblem $i \in [N]$.  Note that $\BanditBlackwellAlg^{(i)}$ is not invoked in all rounds $[T]$, but rather a subset $\mathcal{T}_i \subseteq [T]$ depending on when its fellow Blackwell bandit algorithms, i.e., $\BanditBlackwellAlg^{(i')}$, $i'\in [i-1]$,  decide to explore. Note that $\mathcal{T}_i$ is a random set, and only depends on realizations of binary signals $\{\pi^{(i')}_t\}_{i'\in[i-1],t\in[T]}$. Fix a particular realization of set $\mathcal{T}_i$. By using the upper-bound in \Cref{thm:bandit-blackwell-thm} for each of the terms in the LHS of the bound (i.e., the distance of the average payoff vector from set $S$ and expected number of explorations) separately,  we have
  \begin{align*}
    \standarddistance{\frac{1}{|\mathcal{T}_i|} \sum_{t \in \mathcal{T}_i} \pbf\left(
      \param_t^{(i)}, \advfunB(\zbf_t^{(i-1)},f_t)
    \right)}{S} &\le
    \schange{\bigO{{D(\pbf)}^{1/3}{D(\hat{\pbf})}^{2/3}\left(\log(\payoffdimension)\right)^{1/3} \lvert \mathcal{T}_i\rvert^{-1/3}}}\,.
  \end{align*}
  Moreover, let ${\mathcal{T}}_i^+$, ${\mathcal{T}}_i^-$, $M_i$ be the rounds where $\BanditBlackwellAlg{}^{(i)}$ explores, the rounds where $\BanditBlackwellAlg{}^{(i)}$ exploits, and the number of rounds that $\BanditBlackwellAlg{}^{(i)}$ explores respectively, i.e. $M_i = |{\mathcal{T}}_i^+|$.  Then, by our choice of $q= {{D(\pbf)}^{-2/3}\diameter{\hat{\pbf}}^{2/3}(\log d)^{1/3}T^{-1/3}}$, we have 
  \begin{align*}
      \E \left[ \frac{1}{|\mathcal{T}_i|} M_i \bigg| \mathcal{T}_i \right]
     & = {D(\pbf)}^{-2/3}{D(\hat{\pbf})}^{2/3}\left(\log(\payoffdimension)\right)^{1/3}\lvert {T}\rvert^{-1/3}\\
      &\leq \schange{{{D(\pbf)}^{-2/3}{D(\hat{\pbf})}^{2/3}\left(\log(\payoffdimension)\right)^{1/3}\lvert \mathcal{T}_i\rvert^{-1/3}}}~,
  \end{align*} 
Let $\mathcal{T}^-$ be the set of rounds where no $\BanditBlackwellAlg{}^{(i)}$ explored and $\mathcal{T}^+$ be the set of rounds where some $\BanditBlackwellAlg{}^{(i)}$ explored. Note also that $\mathcal{T}^- \subseteq \mathcal{T}_i$, simply because if no algorithm explores, $\BanditBlackwellAlg^{(i)}$ will be invoked. Notice that for the rounds where $\BanditBlackwellAlg{}^{(i)}$ is invoked and not exploring, but there exists some subroutine $j>i$ that explores, no feedback was given to $\BanditBlackwellAlg{}^{(i)}$, so we can think of it as if $\BanditBlackwellAlg{}^{(i)}$ is not being invoked. Hence, combining this with the fact that the set $S$ is the positive orthant, we have:
  \begin{align*}
    \forall j\in [n]:~~ \left[ \sum_{t \in (\mathcal{T}_i\setminus\mathcal{T}^+)\cup\mathcal{T}_i^+} \pbf\left(
      \param_t^{(i)}, \advfunB(\zbf_t^{(i-1)},f_t)
    \right) \right]_j \ge -\schange{\bigO{{D(\pbf)}^{1/3}{D(\hat{\pbf})}^{2/3}\left(\log(\payoffdimension)\right)^{1/3} \lvert (\mathcal{T}_i\setminus\mathcal{T}^+)\cup\mathcal{T}_i^+\rvert^{2/3}}}\,.
  \end{align*}
Furthermore, due to Blackwell reduciblity, which is a precondition for bandit Blackwell reduciblity (see Definitions \ref{def:blackwell-reducible} and \ref{def:bandit-blackwell-reducible}), $
\pay\left(\param_t^{(i)},\zbf_t^{(i-1)},f_t\right)=\pbf\left(\param_t^{(i)},\advfunB(\zbf_t^{(i-1)},f_t)\right)$. Thus,
  \begin{align}
  \label{eq:all-terms}
    \forall j\in[n]:~~ \left[ \sum_{t \in (\mathcal{T}_i\setminus\mathcal{T}^+)\cup\mathcal{T}_i^+} \pay(\param_t^{(i)}, \zbf_t^{(i-1)},f_t) \right]_j &\ge -\schange{\bigO{{D(\pbf)}^{1/3}{D(\hat{\pbf})}^{2/3}\left(\log(\payoffdimension)\right)^{1/3} \lvert (\mathcal{T}_i\setminus\mathcal{T}^+)\cup\mathcal{T}_i^+\rvert^{2/3}}}\,.
  \end{align}
For the rounds where no $\BanditBlackwellAlg{}^{(i)}$ explore, for any $j\in [n]$, we have 
  \begin{align*}
       \E \left[ \sum_{t \in \mathcal{T}^-} \pay(\param_t^{(i)}, \zbf_t^{(i-1)},f_t)\bigg|\mathcal{T}_i \right]_j& =\E \left[ \sum_{t \in (\mathcal{T}_i\setminus\mathcal{T}^+)\cup\mathcal{T}_i^+} \pay(\param_t^{(i)}, \zbf_t^{(i-1)},f_t)\bigg|\mathcal{T}_i \right]_j \\&-
  \E \left[ \sum_{t \in \mathcal T_i^+} \pay(\param_t^{(i)}, \zbf_t^{(i-1)},f_t)\bigg|\mathcal{T}_i \right]_j\\
  &\ge  - \schange{\bigO{{D(\pbf)}^{1/3}{D(\hat{\pbf})}^{2/3}\left(\log(\payoffdimension)\right)^{1/3} \lvert (\mathcal{T}_i\setminus\mathcal{T}^+)\cup\mathcal{T}_i^+\rvert^{2/3}}}- \diameter{{\pbf}}\E\left[M_i|\mathcal{T}_i\right]\,,
  \end{align*}
  where the expectation is with respect to $\zbf_t^{(i-1)}$, $t\in \mathcal T_i$. The equality holds because $\mathcal{T}^-\cup\mathcal{T}_i^+ = (\mathcal{T}_i\setminus\mathcal{T}^+)\cup\mathcal{T}_i^+$ and $\mathcal{T}^-\cap\mathcal{T}_i^+ = \emptyset$. The inequality follows from Equation \eqref{eq:all-terms} and the fact that $M_i = |\mathcal{T}_i^+|$ and for any $i\in [N]$ and $t\in [T]$,  $\pay(\param_t^{(i)}, \zbf_t^{(i-1)},f_t) \le \diameter{{ \pbf}}$.
  By considering the fact that \[\E\left[M_i|\mathcal{T}_i\right]~=~ {{D(\pbf)}^{-2/3}\diameter{\hat{\pbf}}^{2/3}(\log d)^{1/3}T^{-1/3}} |\mathcal T_i| ~\le~  \schange{{{D(\pbf)}^{-2/3}{D(\hat{\pbf})}^{2/3}\left(\log(\payoffdimension)\right)^{1/3}\lvert \mathcal{T}_i\rvert^{2/3}}}\]
  from our choice of  the probability of exploring, $q$, in \Cref{thm:bandit-blackwell-thm}, $\lvert \mathcal{T}_i\rvert\leq T$ and $\lvert(\mathcal{T}_i\setminus\mathcal{T}^+)\cup\mathcal{T}_i^+\rvert\leq T$, we have:
  \begin{align}
  \label{eq:bandit-final}
    \forall j\in[n]:~~ \E \left[ \sum_{t \in \mathcal{T}^-} \pay(\param_t^{(i)}, \zbf_t^{(i-1)},f_t) \right]_j &\ge - \schange{\bigO{{D(\pbf)}^{1/3}{D(\hat{\pbf})}^{2/3}\left(\log(\payoffdimension)\right)^{1/3} T^{2/3}}}\,.
  \end{align}
Because the offline algorithm $\offlinemeta$ (\Cref{alg:offline-meta}) is an extended $(\gamma, \delta)$-robust approximation, by focusing on rounds in $\mathcal{T}^{-}$ and applying Inequality~\eqref{eq:bandit-final}, together with linearity of expectation, we have:
  \begin{align*}
  \label{eq:type-one}
    \E\left[\sum_{t \in \mathcal{T}^-} {f_t(\zbf_t)}\right] &\geq \gamma \cdot \E\left[\sum_{t \in \mathcal{T}^-} f_t(\zbf^*)\right]-\schange{\bigO{{D(\pbf)}^{1/3}{D(\hat{\pbf})}^{2/3}N\delta\left(\log(\payoffdimension)\right)^{1/3} T^{2/3}}}\,,
  \end{align*}
  where $\zbf^*= \argmax_{\zbf \in \constraint} \sum_{t=1}^T f_t(\zbf)$ is the optimal in-hindsight solution. 
  
Finally, note that $\banditmeta$ does not explore too often in total among its subproblems. More precisely, 
  \begin{align*}
    \E \left[ |\mathcal{T}^+| \right] = \E \left[ \sum_{i=1}^N M_i \right]
                            \le \sum_{i=1}^N \bigO{\left(\log(\payoffdimension)\right)^{1/3}\E\left[\mathcal{T}_i^{2/3}\right]}
                            \le \bigO{N \left(\log(\payoffdimension)\right)^{1/3}T^{2/3}}\,.
  \end{align*} 
Noting the fact that the functions $f_t$ have output value at most $1$, for the remaining rounds $\mathcal{T}^+$ we have:
\begin{align}
 \E\left[\sum_{t \in \mathcal{T}^+} {f_t(\zbf_t)}\right] &\geq \gamma \cdot \E\left[\sum_{t \in \mathcal{T}^+} f_t(\zbf^*)\right]- \bigO{N\left(\log(\payoffdimension)\right)^{1/3}T^{2/3}}\,.
  \end{align}
  Combining the two types of bounds in rounds $\mathcal{T}^{-}$ and $\mathcal{T}^{+}$ yields the desired claim.
\[\myqed\]
\endproof{}

%% file: tex/apx-lowerbound.tex
\subsection{Bandit Blackwell Regret Lowerbound}
\label{app:bandit-blackwell-lowerbound}

In this section, we show that in a Bandit Blackwell sequential game, $\banditblackwellsequentialgame,$ the distance from the time-averaged payoff to the target set $S$ plus the time-averaged exploring penalty of any prediction strategy must be at least $\bigOmega{\Dbottom T^{-1/3}}$, where $\Dbottom = \min\{D(\pbf),\schange{D(\pbf)}\}$. Put differently, we show that the performance bound proved in \Cref{thm:bandit-blackwell-thm} is unimprovable with respect to $T$ (the number of rounds), i.e., no other strategies can have a better performance for all problems.

\begin{theorem}
\label{thm:bandit-blackwell-lower-bound}
In a bandit Blackwell sequential game, $\banditblackwellsequentialgame,$  there exists an adversary's strategy such that for every player 1's strategy, the resulting sequence of actions satisfy:
  \begin{align*}
    \standarddistance{\frac1T \sum_{t=1}^T \pbf(\bm{x}_t, \bm{y}_t)}{S} + \E \left[ \frac1T \schange{D(\pbf)} \cdot (\text{\# explore}) \right] &\ge \bigOmega{\Dbottom T^{-1/3}}\,.
  \end{align*}
where (\# explore) is the number of exploration rounds and $\Dbottom = \min\{D(\pbf),\schange{D(\pbf)} \}.$
\end{theorem}

\proof{\emph{Proof of \Cref{thm:bandit-blackwell-lower-bound}.}} Let $M$ be a random variable equal to the number of rounds the player explores. We first show in that if the number of rounds that the player explores at is at most $M,$ then there exists a Bandit Blackwell instance: an adversary's action $(\bm{y}_1,\ldots,\bm{y}_T),$ a convex closed set $S$, and an affine payoff $\pbf$ together with an unbiased estimator function $\hat{\pbf}$ such that:
  \begin{align*}
    \E\left[\standarddistance{\frac1T \sum_{t=1}^T \pbf(\bm{x}_t, \bm{y}_t)}{S}\bigg|M\right] &\ge \bigOmega{\dfrac{\Dbottom}{\sqrt{M}}}\,
  \end{align*}
for any player's strategies $(\bm{x}_1,\ldots,\bm{x}_T)$, where the expectation is taken over the randomness in the adversary's strategy. We later show this statement in Lemma \ref{lemma:lower-bound}. For now, we assume that the statement is true. Since the Bandit Blackwell total regret defined in Definition \ref{def:bandit-blackwell-approach} includes another term for the cost of exploring, the total regret conditioned on $M$ is
\begin{align*}
    \E\left[\standarddistance{\frac1T \sum_{t=1}^T \pbf(\bm{x}_t, \bm{y}_t)}{S}\bigg|M\right] + \E \left[ \frac1T\schange{D(\pbf)} \cdot (\# \text{explore}) \bigg|M\right] &{\ge} \bigOmega{\dfrac{\Dbottom}{\sqrt{M}}} + \bigOmega{\dfrac{\Dbottom M}{T}} \\
    &\overset{(1)}{\ge} \bigOmega{\Dbottom T^{-1/3}},
\end{align*}
where Inequality $(2)$ follows from setting $M = T^{2/3}$; notice that at $M = T^{2/3},$ $\frac{\underline{D}}{\sqrt{M}} = \frac{\underline{D} M}{T}$ and $\bigOmega{\frac{\underline{D}}{\sqrt{M}}} + \bigOmega{\frac{\underline{D} M}{T}}$ is minimized. Again, the expectation here is with respect to the adversary's strategy. Now, taking another expectation with respect to $M$, we have
\begin{align*}
    &\E\left[\standarddistance{\frac1T \sum_{t=1}^T \pbf(\bm{x}_t, \bm{y}_t)}{S}\right] + \E \left[ \frac1T \schange{D(\pbf)} \cdot (\# \text{explore})\right] \\&= \E\left[\E\left[\standarddistance{\frac1T \sum_{t=1}^T \pbf(\bm{x}_t, \bm{y}_t)}{S}\bigg|M\right]\right] + \E\left[\E \left[ \frac1T \schange{D(\pbf)} \cdot (\# \text{explore}) \bigg|M\right]\right]
    \ge \bigOmega{{\Dbottom}T^{-1/3}}.
\end{align*}
This completes the proof.
\[\myqed\]
\endproof{}

We now prove the lower bound on the distance from the average payoff to $S$ when the number of exploration rounds is $M$. As is common in proofs of lower bounds, we construct a random sequence of similar adversaries and show that with $M$ rounds of explorations, it is impossible to distinguish the different types of adversaries without suffering a regret less than $\tfrac{\underline{D}}{\sqrt{M}}.$

\begin{lemma}
\label{lemma:lower-bound}
In a Bandit Blackwell problem, if the number of exploration rounds is at most $M,$ there exists an adversary's strategy $(\bm{y}_1,\ldots,\bm{y}_T)$, a convex closed set $S$, and an affine payoff $\pbf$ together with an unbiased estimator function $\hat{\pbf}$ such that for any strategies $(\bm{x}_1,\ldots,\bm{x}_T),$
  \begin{align*}
    \E\left[\standarddistance{\frac1T \sum_{t=1}^T \pbf(\bm{x}_t, \bm{y}_t)}{S}\bigg | M\right] &\ge \bigOmega{\dfrac{{\Dbottom}}{\sqrt{M}}}\,,
  \end{align*}
where the expectation is taken with respect to the adversary's strategies.
\end{lemma}

\proof{\emph{Proof of Lemma \ref{lemma:lower-bound}}}
We only prove our lower bound for a deterministic player. Note that any randomized strategy can be expressed as a randomization of deterministic strategies, and based on Yao's minimax principle (\cite{yao1977probabilistic}), our lower bound still holds when we average them over several deterministic strategies according to some randomization. We refer to \cite{cesa2006prediction} for details on deriving an identical lower bound for a randomized adversary from a deterministic adversary.

Recall that $M$ is the random variable of the number of rounds that the player explores. Consider a fixed $M$. From this point onward (until the end of the proof), every probability and expectation are conditioned on $M$, we remove the dependence on $M$ on the notations for simplicity. Let $\algspaceB = \Delta([n])$, $\advspaceB = \{0,1\}^n$, the payoff function $\pbf(\bm{x},\bm{y}) = \bm{x}^T\bm{y}\mathbf{1}_n-\bm{y},$ and $S$ be the non-positive orthant, i.e. $S = \{\mathbf{s}\in\mathbb{R}^n|[\mathbf{s}]_j\leq 0~\forall j\in[n]\}.$  For deterministic strategies, $\bm{x}$ must be equal to $\mathbf{e}_z$ for some coordinate $z$, which happens when player 1 plays action $z\in[n].$ In that case, $\left[\pbf(\bm{x},\bm{y})\right]_j = [\bm{y}]_z-[\bm{y}]_j$ for all $j\in[n].$ 

We now define the adversary's strategy. For each round $t\in[T]$ and coordinate $j\in[n],$ let $[\bm{y}_t]_j$ be Bernoulli random variables whose joint distribution are defined as follows. We first pick a random variable $\zeta\sim \text{Uniform}\{1,2,\ldots,n\}.$ Then, given that $\zeta=i,$ $[\bm{y}_1]_j,[\bm{y}_2]_j,\ldots,[\bm{y}_T]_j$ are conditionally independent Bernoulli random variables with parameter $(1-\mu)/2$ if $j\neq i,$ and $(1+\mu)/2$ if $j=i,$ where $\mu<1/4$ (will be specified later). For analysis purposes, we define another move for the adversary, which we call the base move: all $[\bm{y}_1]_j,[\bm{y}_2]_j,\ldots,[\bm{y}_T]_j$ are conditionally independent Bernoulli variables with parameter $(1-\mu)/2.$ Suppose that this happens when $\zeta=0$ (just for ease of notations).

Let $I_t$ be the player's action (in $\{1,2,\ldots,n\}$) on round $t$, and $\pi_t$ be the exploring  indicator in round $t$:   $\pi_t = 1$ if we explore in round $t,$ and $0$ otherwise. Let $\bm{\eta}_{\tme} = (\pi_1,\ldots,\pi_t)$ be the history of exploration decisions  up to round $t.$ Since the player is deterministic, $I_t$ is determined by  $\left(\pbf(\bm{x}_1,\bm{y}_1),\pbf(\bm{x}_2,\bm{y}_2),\ldots,\pbf(\bm{x}_{t-1},\bm{y}_{t-1}),{\bm{\eta}_{t-1}}\right).$ Also, let $T_j = \sum_{t=1}^T\mathbbm{1}\left[{I_t=j}\right]$ be the number of times action $j$ is played in the first $T$ rounds. We further define  $\Pp_j$ and $\E_j$ as $\Pp(\cdot|\zeta=j)$  $\E(\cdot|\zeta=j)$, respectively. More rigorously, if $\mathcal{A}$ is a $\sigma$-algebra generated by all possible outcomes of the game, $\Pp_j$ is a measure on the $\sigma-$algebra $\mathcal{A}$ and $\E_j$ is an expectation taken with respect to the conditional probability $\Pp_j$, which solely depends on the adversary's move since we assume that player 1's strategy is deterministic.

Recall that for any $j\in [n]$, when $\zeta= j$, playing action $j$ has the highest average reward than any other actions. Then, we have
\begin{equation*}
\begin{aligned}
\E_j\left[\standarddistance{\frac1T \sum_{t=1}^T \pbf(\bm{x}_t, \bm{y}_t)}{S}  \right] &\overset{(1)}{\ge}
\underset{z\in[n]}{\textrm{max}}~ \E_j\left[\left[\frac1T \sum_{t=1}^T \pbf(\bm{x}_t, \bm{y}_t)\right]_z  \right]\\
&= \frac1T \underset{z\in[n]}{\max}~ \E_j\left[\sum_{t=1}^T\left([\bm{y}_t]_z - [\bm{y}_t]_{I_t}\right)  \right] \ge \frac1T \E_j\left[\sum_{t=1}^T\left([\bm{y}_t]_j - [\bm{y}_t]_{I_t}\right)  \right] 
\\&= \frac1T \sum_{t=1}^T\E_j[[\bm{y}_t]_j -[\bm{y}_t]_{I_t}]\\&\overset{(2)}{=} \frac1T \sum_{t=1}^T\mu\E_j[\mathbbm{1}\left({I_t\neq j}\right)]\\&= \frac1T \mu \underset{j'\neq j}\sum\E_j[T_{j'}] = \frac1T \mu (T - \E_j[T_j]) \\
&=  \left(\mu - \frac{\mu}{T}\E_j[T_j]\right).
\end{aligned}
\end{equation*}
Inequality $(1)$ follows because $S$ is the non-positive orthant. Equality $(2)$ follows because $\E_j[[\bm{y}_t]_{j'}]$ is $(1+\mu)/2$ when $j'=j$ and $(1-\mu)/2$ otherwise, so the difference between $\E_j[[\bm{y}_t]_{j}]$ and $\E_j[[\bm{y}_t]_{I_t}]$ is $\mu$ when $I_t\neq j$ and $0$ otherwise. 

As for each $j\in[n],$ since the event $\{\zeta=j\}$ happens with probability $\frac{1}{n},$ we have
\begin{equation}
\label{eq:supdist}
    \sup~\E\left[\standarddistance{\frac1T \sum_{t=1}^T \pbf(\bm{x}_t, \bm{y}_t)}{S}\right] \ge  \mu \left(1 - \frac{1}{nT}\sum_j\E_j[T_j]\right),
\end{equation}
where the expectation is taken with respect to the adversary's move and the supremum is taken over $\zeta\in\{1,2,\ldots,n\}$ since the $\zeta$ picked by the adversary in the beginning of the game determines his whole strategy. The proof now reduces to bounding $\E_j[T_j]$ from above. We do this by comparing $\E_j[T_j]$ with $\E_0[T_j].$ If player 1 chooses action $i$ at round $t$ and decides to explore, i.e., $I_t=i$ and $\pi_t=1,$ he then observes the payoff $[\bm{y}_t]_i.$ Recall that $\bm{y}_t$ is the random variable that represents adversary's move at round $t,$ where $\bm{y}_t\in \mathcal{Y}=\{0,1\}^n.$ For any sequence of history $(\mathbf{H}_t,\bm{\eta}_t)$ where $\mathbf{H}_t = ([\bm{y}_1]_{I_1},\ldots,[\bm{y}_t]_{I_t}) = (h_1,\ldots,h_t)\in\{0,1\}^t$ and $\bm{\eta}_t = (\pi_1,\ldots,\pi_t)\in\{0,1\}^t$, let
\begin{equation*}
    \chi_{t,j}(\mathbf{H}_t,\bm{\eta}_t) = \Pp_j\left([\bm{y}_1]_{I_1}=h_1,\ldots,[\bm{y}_t]_{I_t}=h_t,\bm{\eta}_t\right).
\end{equation*}
Note that $\Pp_j$ is a measure on the $\sigma-$algebra $\mathcal{A}$ as mentioned above, and the randomness comes from the adversary's moves (the adversary plays a randomized $\bm{y}_t$ at time $t,$ where the $j^{\text{th}}$ coordinate of $\bm{y}_t$ is a Bernoulli variable with mean either $(1+\mu)/2$ or $(1-\mu)/2$ depending on his choice of $\zeta$ at the beginning of the game). From our assumption that the player is deterministic, for any $\mathbf{H}_T\in\{0,1\}^T$ and history of exploration $\bm{\eta}_T$. Then,
\begin{equation}
\label{eq:lb-1}
    \E_i\Big[T_j\Big|[\bm{y}_1]_{I_1}=h_1,\ldots,[\bm{y}_T]_{I_T}=h_T,\bm{\eta}_T\Big] = \E_0\Big[T_j\Big|[\bm{y}_1]_{I_1}=h_1,\ldots,[\bm{y}_T]_{I_T}=h_T,\bm{\eta}_T\Big],~\forall 1\leq i\leq n,
\end{equation}
which means that no matter what the initial $\zeta$ decided by the adversary is, the player has the same sequence of moves given the same history. Therefore,
\begin{equation}
\begin{aligned}
\label{eq:lb-long}
\E_j[T_j]-\E_0[T_j] &= \underset{\mathbf{H}_T\in\{0,1\}^T,\bm{\eta}_T\in\{0,1\}^T}{\sum}\chi_{T,j}(\mathbf{H}_T,\bm{\eta}_T)\E_j\Big[T_j\Big|[\bm{y}_1]_{I_1}=h_1,\ldots,[\bm{y}_T]_{I_T}=h_T,\bm{\eta}_T\Big] \\&- \underset{\mathbf{H}_T\in\{0,1\}^T,\bm{\eta}_T\in\{0,1\}^T}{\sum}\chi_{T,0}(\mathbf{H}_T,\bm{\eta}_T)\E_0\Big[T_j\Big|[\bm{y}_1]_{I_1}=h_1,\ldots,[\bm{y}_T]_{I_T}=h_T,\bm{\eta}_T\Big]\\
&\overset{(1)}{=}\underset{\mathbf{H}_T\in\{0,1\}^T,\bm{\eta}_T\in\{0,1\}^T}{\sum} (\chi_{T,j}(\mathbf{H}_T,\bm{\eta}_T) - \chi_{T,0}(\mathbf{H}_T,\bm{\eta}_T))\E_j\Big[T_j\Big|[\bm{y}_1]_{I_1}=h_1,\ldots,[\bm{y}_T]_{I_T}=h_T,\bm{\eta}_T\Big]\\
&\leq \underset{(\mathbf{H}_T,\bm{\eta}_T):\chi_{T,j}(\mathbf{H}_T,\bm{\eta}_T)>\chi_{T,0}(\mathbf{H}_T,\bm{\eta}_T)}{\sum}(\chi_{T,j}(\mathbf{H}_T,\bm{\eta}_T) - \chi_{T,0}(\mathbf{H}_T,\bm{\eta}_T))\E_j\Big[T_j\Big|[\bm{y}_1]_{I_1}=h_1,\ldots,[\bm{y}_T]_{I_T}=h_T,\bm{\eta}_T\Big]\\ 
&\overset{(2)}{\leq} T\underset{(\mathbf{H}_T,\bm{\eta}_T):\chi_{T,j}(\mathbf{H}_T,\bm{\eta}_T)>\chi_{T,0}(\mathbf{H}_T,\bm{\eta}_T)}{\sum}(\chi_{T,j}(\mathbf{H}_T,\bm{\eta}_T) - \chi_{T,0}(\mathbf{H}_T,\bm{\eta}_T)),
\end{aligned}
\end{equation}
    where Equation $(1)$ follows from \Cref{eq:lb-1} and Inequality $(2)$ follows from $\E_j[T_j|[\bm{y}_1]_{I_1}=h_1,\ldots,[\bm{y}_T]_{I_T}=h_T,\bm{\eta}_T]\leq T.$ See that $\sum_{j=1}^n\E_0[T_j]=T$ since on each round, player 1's action is in $\{1,2,\ldots,n\}.$ We can bound the total variation using Pinsker's inequality:~\footnote{\emph{Pinsker's inequality bounds the total variation distance in terms of KL divergence. For two probability distributions $P$ and $Q$, Pinsker's inequality states that $||P-Q||_{\text{TV}}\leq\sqrt{\frac12\KL(P||Q)}$ where $||P-Q||_{\text{TV}}$ is the total variation distance $ \sup_{A}\{|P(A)-Q(A)|\}$ over measurable events $A$. Taking $A = \{x:P(x)>Q(x)\},$ we get $\sum_{x:P(x)>Q(x)}|P(x)-Q(x)|\leq\sqrt{\frac12\KL(P||Q)}.$ See Section A.2 \cite{cesa2006prediction} for details.}}
\begin{equation}
\label{eq:lb-2}
    \underset{(\mathbf{H}_T,\bm{\eta}_T):\chi_{T,j}(\mathbf{H}_T,\bm{\eta}_T)>\chi_{T,0}(\mathbf{H}_T,\bm{\eta}_T)}{\sum}(\chi_{T,j}(\mathbf{H}_T,\bm{\eta}_T) - \chi_{T,0}(\mathbf{H}_T,\bm{\eta}_T))\leq\sqrt{\dfrac{1}{2}\KL(\chi_{T,0}(\mathbf{H}_T,\bm{\eta}_T)||\chi_{T,j}(\mathbf{H}_T,\bm{\eta}_T))}~.
\end{equation} 
Putting Equations (\ref{eq:lb-long}) and (\ref{eq:lb-2}) together and applying Jensen's inequality to the concave the square root function, we get
\begin{equation}
\begin{aligned}
\label{eq:lb-7}
    \frac{1}{n}\sum_{j=1}^n\E_j[T_j]&\geq\frac1n\sum_{j=1}^n\left(\E_0[T_j]+T\underset{\mathbf{H}_T:\chi_{T,j}(\mathbf{H}_T,\bm{\eta}_T)>\chi_{T,0}(\mathbf{H}_T,\bm{\eta}_T)}{\sum}(\chi_{T,j}(\mathbf{H}_T,\bm{\eta}_T) - \chi_{T,0}(\mathbf{H}_T,\bm{\eta}_T))\right)\\
    &{\leq} T\left(\frac{1}{n}+\frac{1}{n}\sum_{j=1}^n\sqrt{\dfrac{1}{2}\KL(\chi_{T,0}(\mathbf{H}_T,\bm{\eta}_T)||\chi_{T,j}(\mathbf{H}_T,\bm{\eta}_T))}\right)\\
    &\leq T\left(\frac{1}{n}+\sqrt{\dfrac{1}{2n}\sum_{j=1}^n\KL(\chi_{T,0}(\mathbf{H}_T,\bm{\eta}_T)||\chi_{T,j}(\mathbf{H}_T,\bm{\eta}_T))}\right).
\end{aligned}
\end{equation}
Recall that from the definition of $\chi_{t,j},$ we have the following conditional distribution:
\begin{equation*}
    \chi_{t,j}(h_t|\mathbf{H}_{t-1},\bm{\eta}_{t-1}) = \Pp_j\left([\bm{y}_t]_{I_t}=h_t|[\bm{y}_1]_{I_1}=h_1,\ldots,[\bm{y}_{t-1}]_{I_{t-1}}=h_{t-1},\bm{\eta}_{t-1}\right).
\end{equation*}
Applying the chain rule, we have
\begin{equation*}
\begin{aligned}
\KL(\chi_{T,0}||\chi_{T,j}) 
&\overset{(1)}{=} \sum_{t=1}^T\underset{\mathbf{H}_{t-1},\bm{\eta}_{t-1}}{\sum}\chi_{t-1,0}(\mathbf{H}_{t-1},\bm{\eta}_{t-1})\KL\left(\chi_{t,0}(\cdot|\mathbf{H}_{t-1},\bm{\eta}_{t-1})\middle|\middle|\chi_{t,j}(\cdot|\mathbf{H}_{t-1},\bm{\eta}_{t-1})\right)\\
&\overset{(2)}{=} \sum_{t=1}^T\underset{\mathbf{H}_{t-1},\bm{\eta}_{t-1}}{\sum}\chi_{t-1,0}(\mathbf{H}_{t-1},\bm{\eta}_{t-1})\mathbbm{1}\left({\{I_t=j\text{ and } \pi_t=1|\mathbf{H}_{t-1},\bm{\eta}_{t-1}\}}\right)\KL\left(\frac{1-\mu}{2}\middle|\middle|\frac{1+\mu}{2}\right)\\
&\overset{(3)}{=} \KL\left(\frac{1-\mu}{2}\middle|\middle|\frac{1+\mu}{2}\right)\E_0\left[\sum_{t=1}^T\mathbbm{1}\left({\{I_t=j\text{ and } \pi_t=1\}}\right)\right].
\end{aligned}
\end{equation*}
Here, Equation $(1)$ follows from applying the chain rule to $\chi_{T,0}$ and $\chi_{T,j}.$ Equation $(2)$ holds because $\chi_{t,0}(\cdot|\mathbf{H}_{t-1},\bm{\eta}_{t-1}) = \Ber\left(\tfrac{1-\mu}{2}\right)$ and $\chi_{t,j}(\cdot|\mathbf{H}_{t-1},\bm{\eta}_{t-1}) = \Ber\left(\tfrac{1+\mu}{2}\right)$ when we play the arm $j$ on round $t,$ $I_t = j,$ and observe the payoff, $\pi_t = 1.$ Otherwise, they are identical and we have $\KL\left(\chi_{t,0}(\cdot|\mathbf{H}_{t-1},\bm{\eta}_{t-1})\middle|\middle|\chi_{t,j}(\cdot|\mathbf{H}_{t-1},\bm{\eta}_{t-1})\right) = 0.$ Lastly, we get Equation $(3)$ by factoring out $\KL\left(\tfrac{1-\mu}{2}||\tfrac{1+\mu}{2}\right)$, and collecting all the probability terms ($\chi_{t,0}(\mathbf{H}_t,\mu_t)$ for all $t$) to form the expectation of $\mathbbm{1}\left({\{I_t=j\text{ and } \pi_t=1\}}\right)$ with respect to $\Pp_0.$ 

Summing over $j$ and applying $\KL(p||q)\leq\frac{(p-q)^2}{q(1-q)}$:
\begin{equation}
\label{eq:lb-3}
\begin{aligned}
\sum_{j=1}^n \KL(\chi_{T,0}||\chi_{T,j}) &= \KL\left(\frac{1-\mu}{2}\middle|\middle|\frac{1+\mu}{2}\right)\sum_{j=1}^n \E_0\left[\sum_{t=1}^T\mathbbm{1}\left({\{I_t=j\text{ and } \pi_t=1\}}\right)\right]
\\&= \KL\left(\frac{1-\mu}{2}\middle|\middle|\frac{1+\mu}{2}\right)\E_0\left[\sum_{t=1}^T\sum_{j=1}^n \mathbbm{1}\left({\{I_t=j\text{ and } \pi_t=1\}}\right)\right] 
\\&= \KL\left(\frac{1-\mu}{2}\middle|\middle|\frac{1+\mu}{2}\right)\E_0\left[\sum_{t=1}^T\mathbbm{1}\left({\{\pi_t=1\}}\right)\right] \leq \dfrac{4\mu^2}{1-\mu^2}M,
\end{aligned}
\end{equation}
where the last line follows from the assumption that the number of rounds the player explores is $M$.

Putting Equation (\ref{eq:supdist}), (\ref{eq:lb-7}) and (\ref{eq:lb-3}) altogether we get:
\begin{equation}
    \E_j\left[\standarddistance{\frac1T \sum_{t=1}^T \pbf(\bm{x}_t, \bm{y}_t)}{S} \right] \ge  \mu \left(1-\frac{1}{n} - \sqrt{\frac{1}{2n} \dfrac{4\mu^2}{1-\mu^2}M}\right)\ge  \mu \left(1-\frac{1}{n} - 4\mu\sqrt{\dfrac{M}{6n}}\right),
\end{equation}
where the last inequality follows from $\mu\leq1/4.$ Taking $\mu = \lambda\sqrt{\frac{n}{M}},$ we have:
\begin{equation}
    \E_j\left[\standarddistance{\frac1T \sum_{t=1}^T \pbf(\bm{x}_t, \bm{y}_t)}{S} \right]\ge  \lambda\sqrt{\dfrac{n}{ M}}\left(\frac{1}{2} - \frac{4\lambda}{\sqrt{6}}\right) \geq \bigOmega{\frac{1}{\sqrt{ M}}} = \bigOmega{\frac{{\Dbottom}}{\sqrt{ M}}},
\end{equation}
where the last equality follows from ${\Dbottom} \leq D(\pbf) = 1$ in this case since the adversary's move is in $\{0,1\}^n.$
We finish the proof by choosing the constant $\lambda$ to be small enough to ensure that $\left( \tfrac12 - \tfrac{4\lambda}{\sqrt{6}} \right)$ is positive.
\[\myqed\]
\endproof

%% file: tex/conti_SM.tex
\section{Application to Non-monotone (Continuous) Submodular Maximization}
\label{subsec:usm}
\paragraph{Problem definition.} Consider the \textsc{Non-monotone submodular maximization (NSM)} problem, for both set  and continuous functions. For set functions, our goal is to maximize a non-monotone submodular set function without any constraints, and for continuous functions, our goal is to maximize a non-monotone continuous submodular function, either weak-DR  or strong-DR, over the unit hypercube $[0,1]^n$; see the definition of  weak-DR  and strong-DR continuous submodular functions below.

\textit{Continuous submordular functions.} We defined set submodular functions in  Example \ref{example:running-1}.  The concept of submodularity can be extended from subset lattice (above definition) to any discrete or continuous lattice. In particular, by considering the positive orthant cone lattice, we can define  the continuous variant of set submodularity.  A continuous multivariate function $f:[0,1]^n\rightarrow[0,1]$ is submodular if for all $\bm{x},\bm{y}\in[0,1]^n,$
\begin{equation*}
   f(\bm{x}\vee\bm{y})+f(\bm{x}\wedge\bm{y})\leq f(\bm{x}) + f(\bm{y})\,,
\end{equation*}
where $\vee$ and $\wedge$ are coordinate-wise max and min operations. As an equivalent definition~\citep{bian2016guaranteed}, $f$ is continuous submodular if for all $i\in[n]$, $z\in[0,1]$, $\xbf_{-i}\preceq \ybf_{-i}\in[0,1]^{n-1}$, and $\delta\geq 0,$
\begin{equation*}
  f(z+\delta,\bm{x}_{-i})-f(z,\bm{x}_{-i})\geq~ f(z+\delta,\bm{y}_{-i})-f(z,\bm{y}_{-i})\,.
\end{equation*}

 {The above class of continuous functions is also referred to as the weak-Diminishing Return (weak-DR) submodular in the literature \cite{wolsey1982analysis, bach2019submodular}.} We further consider a special subclass of these functions satisfying concavity along each coordinate, referred to as the strong-Diminishing Return (strong-DR).
  A continuous multivariate function $f:[0,1]^n\rightarrow[0,1]$ is strong-DR submodular if for all $i\in[n], \bm{x}\preceq \bm{y}\in[0,1]^n$, and $\delta\geq 0,$
\begin{equation*}
  f(x_i+\delta,\bm{x}_{-i})-f(\bm{x})\geq~ f(y_i+\delta,\bm{y}_{-i})-f(\bm{y})\,,
\end{equation*}
where $\bm{x}_{-i}$ (resp. $\bm{y}_{-i}$) is an $(n-1)$-dimensional vector with all coordinate values of $\bm{x}$ (resp. $\bm{y}$) except $i$, and $\bm{x}\preceq\bm{y}$ if and only if $\forall j\in[n]: x_j\leq y_j$.

\textit{Offline problems for submodular functions.} For set functions, the offline algorithm of \citealp{buchbinder2015tight} gives a $1/2$-approximation factor, which is known to be the best possible approximation factor with polynomial query calls to the function~\citep{feige2011maximizing}. For the continuous case, under both weak-DR and strong-DR submodularity, the offline algorithm of \citealp{niazadeh2018optimal} gives a $1/2$-approximation factor for Lipschitz continuous functions, which again achieves the best possible approximation factor with polynomial query calls to the function. 

To have a unified offline problem and algorithm capturing both of the above variations, we first consider a slight reformulation where a continuous (weak-DR) submodular function is restricted to a discrete domain $\coordinatevalues^n$ instead of $[0,1]^n$. Here, $\coordinatevalues = \{\coordinatevalue_1, \coordinatevalue_2, \ldots, \coordinatevalue_m\}$ is the finite set of possible coordinate values, where  $|\coordinatevalues| = m$ and $\coordinatevalue_1 < \coordinatevalue_2 < \cdots < \coordinatevalue_m$ are real numbers. Note that $\coordinatevalues=\{0,1\} $ when we focus on set functions. For Lipschitz continuous functions, one should think of $\coordinatevalues^n$ as an $\epsilon$-net that discretizes the function with $O(\epsilon)$ additive error due to Lipschitzness.

Given this unified setting, we essentially consider discrete functions $f:\coordinatevalues^n\rightarrow[0,1]$ that satisfy a  discrete version of (weak-DR) submodularity. This property is exactly the same as continuous submodularity, with a slight modification that we only consider points $\xbf\in\coordinatevalues^n$. Given such a function, our goal in the offline problem is to solve the optimization problem $\max_{\zbf\in \coordinatevalues^n} f(\zbf)$. We refer to this problem as \emph{discretized submodular maximization}. Note that this problem is an instance of problem~\eqref{eq:offline-optimization}, where both $\domain$ and the feasible region $\constraint$ are $\coordinatevalues^n$, and our function class is the class of submodular functions $f$ described above.

Inspired by the algorithms in \cite{buchbinder2015tight} and \cite{niazadeh2018optimal}, we then present a unified offline algorithm (which essentially is an adaptation of the algorithm in \cite{niazadeh2018optimal} restricted to the discrete domain $\coordinatevalues^n$) with the same $1/2$-approximation factor for the proposed unified offline problem. This is presented in Algorithm \ref{alg:usm-single}.  We then transform this offline algorithm to online full-information and bandit learning algorithms using our framework. 




 \Cref{alg:usm-single} is a modified version of the continuous randomized bi-greedy algorithm by \cite{niazadeh2018optimal}. The difference between \Cref{alg:usm-single} and  the continuous randomized bi-greedy algorithm is discussed in \Cref{apx:NSM-discussion} in the appendix. Throughout this section, we use the notation $(z',\zbf_{-i})$ to denote the point constructed by taking $\zbf$ and replacing its $i^{\text{th}}$ coordinate value with $z',$ and $f(z',\zbf_{-i})$ to denote the function evaluated at the corresponding point. The algorithm keeps track of two points: \emph{lower bound} $\usmlower^{(i)}$ and \emph{upper bound} $\usmupper^{(i)}$, where initially, $\usmlower^{(0)}=(\rho_1, \ldots, \rho_1)$, and $\usmupper^{(0)}=(\rho_m, \ldots, \rho_m)$. The lower and upper bounds get updated as the algorithm goes through $n$ subproblems. In subproblem $i$, the algorithm decides about 
the   $i^{\text{th}}$ coordinate: it sets the $i^{\text{th}}$ coordinate to $z_i'$, where  $z_i'$ is drawn from      distribution $\param^{(i)} \in \setofdistributions(\coordinatevalues)$. Here, this distribution is chosen in a way to  satisfy the following condition $\E_{z' \sim \param^{(i)}} \left[ \frac12 \alpha^{(i)}(z') + \frac12 \beta^{(i)}(z') - \zeta^{(i)}({\hat{z}}, z') \right] \ge 0$ for all ${\hat{z}} \in \coordinatevalues$. 
 Note that $\alpha^{(i)}(z') = f(z', \usmlower^{(i-1)}_{-i}) - f(\usmlower^{(i-1)})$ is the marginal value of increasing the value of $i^{\text{th}}$-coordinate from $\rho_1$ to $z'$ when the rest of coordinates are $\usmlower^{(i-1)}_{-i}$, and similarly $\beta^{(i)}(z') = f(z', \usmupper^{(i-1)}_{-i}) - f(\usmupper^{(i-1)})$ is the marginal value of decreasing the $i^{\text{th}}$ coordinate  from $\rho_m$ to $z'$ when the rest of coordinates are $\usmupper^{(i-1)}_{-i}$. Moreover, $\zeta^{(i)}({{\hat{z}}}, z')$ is equal to $\alpha^{(i)}({\hat{z}}) - \alpha^{(i)}(z')$ if $\hat z\ge z'$ and $\beta^{(i)}({\hat{z}}) - \beta^{(i)}(z')$ otherwise.
 Roughly speaking, $\zeta^{(i)}({{\hat{z}}}, z')$ measures the extent to which setting the $i^{\text{th}}$ coordinate  to $z'$, rather than $\hat z$, is locally suboptimal. With this interpretation, 
 the aforementioned condition ensures that the algorithm's choice for the $i^{\text{th}}$ coordinate approximately compensates for the cost caused by the suboptimality of this choice.  
  We refer the readers to \cite{niazadeh2018optimal} for a more detailed discussion on the intuition behind this condition.

  We now show how to cast the above algorithm as an instance of $\offlinemetatext$ (\Cref{alg:offline-meta}). In the language of Algorithm \ref{alg:offline-meta}, the aforementioned condition can be presented using the following $\pay$ function:
  \begin{equation}\label{eq:usmpay}
   j\in[m]:~~ \left[\pay\left(\param^{(i)},\usmlower^{(i-1)},f\right)\right]_j = \E_{z' \sim \param^{(i)}} \left[ \frac12 \alpha^{(i)}(z') + \frac12 \beta^{(i)}(z') - \zeta^{(i)}({\rho_j}, z') \right]\ge 0.
\end{equation}
  Moreover, we have $\Theta = \Delta(\coordinatevalues), \paramdimension = |\coordinatevalues| = m,$ and $\zbf$ is the vector $\usmlower$ that starts as $(\rho_1,\ldots,\rho_1)^T$ then gets updated at each iteration.\footnote{For any $i\in [n]$, one can construct $\usmupper^{(i)}$ from $\usmlower^{(i)}$ by replacing its last $n-1$ coordinates with $\rho_m$. Thus, it suffices to define $\pay$ as a function of $\usmlower^{(i)}$. } 

\begin{algorithm}
    \caption{Greedy Algorithm for Discretized NSM \citep{niazadeh2018optimal}}
    \label{alg:usm-single}
    \textbf{Input:} Discrete submodular function $f$.\\
    \textbf{Output:} Point $\zbf \in \coordinatevalues^n$.\\
    Set initial \emph{lower bound} $\usmlower^{(0)} \leftarrow (\coordinatevalue_1, \coordinatevalue_1, \ldots, \coordinatevalue_1)^T$ and \emph{upper bound} $\usmupper^{(0)} \leftarrow (\coordinatevalue_m, \coordinatevalue_m, \ldots, \coordinatevalue_m)^T$. \\
    \For{coordinate $i = 1, 2, \ldots, n$}{\vspace{0.2em}
      Define the lower marginal function $\alpha^{(i)}: \coordinatevalues \to [-1, +1]$ as $\alpha^{(i)}(z') = f(z', \usmlower^{(i-1)}_{-i}) - f(\usmlower^{(i-1)})$.\\
      Define the upper marginal function $\beta^{(i)}: \coordinatevalues \to [-1, +1]$ as $\beta^{(i)}(z') = f(z', \usmupper^{(i-1)}_{-1}) - f(\usmupper^{(i-1)})$.\\
      Define comparison function $\zeta^{(i)}: \coordinatevalues \times \coordinatevalues \to [-1, +1]$ as
      \begin{align*}
        \zeta^{(i)}({{\hat{z}}}, z') = \left\{\begin{array}{lr}
          \alpha^{(i)}({\hat{z}}) - \alpha^{(i)}(z') & \text{if } {\hat{z}} \ge z' \\
          \beta^{(i)}({\hat{z}}) - \beta^{(i)}(z') & \text{if } {\hat{z}} \le z'
        \end{array}\right.\,.
      \end{align*}
      
      {\underline{Local Optimization Step}\\}
      Choose $\param^{(i)} \in \setofdistributions(\coordinatevalues)$ so that for all ${\hat{z}} \in \coordinatevalues$, 
      \begin{equation}
      \label{eq:payoff-usm}
      \begin{aligned}
        \E_{z' \sim \param^{(i)}} \left[ \frac12 \alpha^{(i)}(z') + \frac12 \beta^{(i)}(z') - \zeta^{(i)}({\hat{z}}, z') \right] \ge 0
      \end{aligned}
      \end{equation}
      (done in \cite{niazadeh2018optimal} via preprocessing and computing a 2D convex hull).\\
      {\underline{Local Update Step}\\}
      Sample ${z_i'} \sim \param^{(i)}$. Set $\usmlower^{(i)} \leftarrow \usmlower^{(i-1)}$ and $\usmupper^{(i)} \leftarrow \usmupper^{(i-1)}$ and then update their $i^{th}$ coordinate:  $\indexintovector{\usmlower^{(i)}}{i} \leftarrow {z_i'}$ and $\indexintovector{\usmupper^{(i)}}{i} \leftarrow {z_i'}$.
    
    }
    \Return{$\zbf \leftarrow \usmlower^{(n)}$}.
\end{algorithm}

\begin{theorem}[Online learning for discretized non-monotone submodular maximization]
\label{thm:usm}
  Let  $n$ be the number of dimensions and $\coordinatevalues = \{\coordinatevalue_1, \coordinatevalue_2, \ldots, \coordinatevalue_m\}$ be the set of potential values that each coordinate $i\in [n]$ can take. 
   Assume that the maximum function value is normalized to one. Then, for the problem of maximizing a non-monotone submodular function in the online full-information setting, there exists a learning algorithm that obtains $\bigO{n T^{1/2} (\log m)^{1/2}}$ $\tfrac12$-regret, where $T$ is the number of rounds. Furthermore, in the online bandit setting, there exists an online learning algorithm that obtains $\bigO{n m\schange{^{2/3}} T^{2/3} (\log m)^{1/3}}$ $\tfrac12$-regret. Here, in both online algorithms, the benchmark in the regret bounds is $\frac{1}{2}\max_{\zbf\in \feasiblereserves^n}\sum_{t=1}^T f_t(\zbf)$.
\end{theorem}

The proof of Theorem \ref{thm:usm}, which is presented in \Cref{apx-proof-NSM}, has two main steps. In the first step, we show that the offline Algorithm \ref{alg:usm-single} is an extended  $(\tfrac{1}{2},\tfrac{1}{2})$- robust approximation algorithm and in the second step, we show that it is bandit Blackwell reducible. The challenging part of the proof is to construct an explore sampling device that leads to an unbiased estimator for the payoff function. We then  invoke Theorems \ref{thm:full-info-online-meta} and \ref{thm:banditILO} to get the final regret bounds.

The following is an immediate corollary of Theorem \ref{thm:usm}. 

\begin{corollary}[Online learning for non-monotone set submodular maximization]
\label{cor:usm-discrete}
  Let $n$ be the number of items, and assume the maximum function value is normalized to one. Then for the problem of maximizing a nonmonotone (set) submodular function in the online full-information setting, there exists an online learning algorithm that obtains \schange{$\bigO{n \sqrt{T}}$} $\tfrac12$-regret, where $T$ is the number of rounds. Furthermore, in the online bandit setting, there exists a learning algorithm that obtains \schange{$\bigO{nT^{2/3}}$} $\tfrac12$-regret, where $T$ is the number of rounds.
\end{corollary}

So far, we assumed that for the continuous submodular functions,  the set of potential value for each coordinate  is finite and belongs to set $\feasiblereserves = \{\feasiblereserve_1, \feasiblereserve_2, \ldots , \feasiblereserve_{m}\}$, rather than the interval $[0, 1]$, and we design learning algorithms with sublinear regret bounds where the regrets are computed with respect to $\frac{1}{2}\max_{\zbf\in \feasiblereserves}\sum_{t=1}^T f_t(\zbf)$. Now, one may wonder if one can design learning algorithms against the benchmark of  $\frac{1}{2}\max_{\zbf\in [0,1]^n}\sum_{t=1}^T f_t(\zbf)$ that allows the coordinates to be any number in $[0,1]^n$.
The following corollary answers this question for any $L$-Lipschitz non-monotone continuous submodular functions.

\begin{corollary}[Online learning for $L-$Lipschitz continuous submodular maximization]
\label{cor:usm-continuous}
  Let $n$ be the number of dimensions, and assume the maximum function value is normalized to one. Then for the problem of maximizing a coordinate-wise $L$-Lipschitz non-monotone (continuous) submodular function in the online full-information setting, there exists a learning algorithm that obtains $\bigO{n T^{1/2} \log^{1/2} (LT)}$ $\tfrac12$-regret, where $T$ is the number of rounds. Furthermore, in the online bandit setting, there exists a learning algorithm that obtains $\bigO{n \schange{L^{2/5}T^{4/5}}\log^{1/3} (LT)}$ $\tfrac12$-regret, where $T$ is the number of rounds. Here, in both online algorithms, the benchmark in the regret bounds is $\frac{1}{2}\max_{\zbf\in [0,1]^n}\sum_{t=1}^T f_t(\zbf)$.
\end{corollary}

Proofs of the above corollaries are in \Cref{apx:corollary-NSM}.

\subsection{Proof of Theorem \ref{thm:usm}}
\label{apx-proof-NSM}
\proof{\emph{Proof}.} 
We will show that our meta Algorithms \ref{alg:full-info-backbone} and \ref{alg:bandit-meta} work by verifying the following conditions.

    \paragraph{(i) Algorithm~\ref{alg:mmr-single} is an extended $(\tfrac12, \tfrac12)$-robust approximation algorithm.}
Following the analysis of the bi-greedy algorithm in \cite{buchbinder2015tight}, we consider three sequences of points: the lower bound sequence $\usmlower^{(i)}$, the upper bound sequence $\usmupper^{(i)}$, and the hybrid-optimal sequence ${\usmoptimal}^{(i)}$. The key proof idea is to bound the decrease in the hybrid-optimal sequence value ${\zbf^*}^{(i)}$ with the total increase in the lower bound and upper bound sequence values. We define $\usmlower^{(i)}$ and $\usmupper^{(i)}$ to agree on the first $i$ coordinates, while the rest of the coordinates are $\rho_1$ for $\usmlower^{(i)}$ and $\rho_m$ for $\usmupper^{(i)}.$ The hybrid-optimal sequence starts from ${\usmoptimal}^{(0)} \triangleq {\usmoptimal}$, then ${\usmoptimal}^{(i)}$ is equal to ${\usmoptimal}^{(i-1)}$ but with the $i^{th}$ coordinate replaced with the sampled $z'_i \sim \param^{(i)}$.

Importantly, if the $i^{th}$-coordinate of the optimal vector $\usmoptimal,$ which is $z^*_i,$ is less than our sampled point $z'_i$ from the $i^{th}$ subproblem/iteration, then the loss in value of the hybrid-optimal sequence is bounded by a difference of two $\beta^{(i)}$ evaluations. In particular, the submodularity of $f$ implies:
\begin{align*}
    f(z^*_i, {\usmoptimal}^{(i-1)}_{-i}) + f(z'_i, \usmupper^{(i-1)}_{-i}) &\le f(z'_i, {\usmoptimal}^{(i-1)}_{-i}) + f(z^*_i, \usmupper^{(i-1)}_{-i}) \\
    f({\usmoptimal}^{(i-1)}) + \beta^{(i)}(z'_i) &\le f({\usmoptimal}^{(i)}) + \beta^{(i)}(z^*_i) \\
    f({\usmoptimal}^{(i-1)}) - f({\usmoptimal}^{(i)}) &\le \beta^{(i)}(z^*_i) - \beta^{(i)}(z'_i).
\end{align*}
  
  There is also the symmetric case where the $i^{th}$-coordinate of the optimal vector $z^*_i$ is greater than our sampled point $z'_i$ from the the $i^{th}$ subproblem:
  \begin{align*}
    f(z'_i, \usmlower^{(i-1)}_{-i}) + f(z^*_i, {\usmoptimal}^{(i-1)}_{-i}) &\le f(z^*_i, \usmlower^{(i-1)}_{-i}) + f(z'_i, {\usmoptimal}^{(i-1)}_{-i}) \\
    \alpha^{(i)}(z'_i) + f({\usmoptimal}^{(i-1)}) &\le \alpha^{(i)}(z^*_i) + f({\usmoptimal}^{(i)}) \\
    f({\usmoptimal}^{(i-1)}) - f({\usmoptimal}^{(i)}) &\le \alpha^{(i)}(z^*_i) - \alpha^{(i)}(z'_i).
  \end{align*}
  
  Combining the two cases yields (this inequality explains our definition of $\zeta^{(i)}$):
  \begin{equation}
    f({\usmoptimal}^{(i-1)}) - f({\usmoptimal}^{(i)}) \le \zeta^{(i)}(z^*_i, z'_i)
    \label{eqn:usmoptimal}.
  \end{equation}
  
  Also, just by the definition of $\alpha^{(i)}$ and $\beta^{(i)}$ we know that:
  \begin{align}
    f(\usmlower^{(i)}) - f(\usmlower^{(i-1)}) &= \alpha^{(i)}(z'_i)
    \label{eqn:usmlower} \\
    f(\usmupper^{(i)}) - f(\usmupper^{(i-1)}) &= \beta^{(i)}(z'_i).
    \label{eqn:usmupper}
  \end{align}
  
  We are now ready to consider the ${\param}^{(i)}_t$ which guarantees that for all $i \in [n]$ and $\hat z \in \coordinatevalues$, including $z_i^*$:
  \begin{align*}
     \sum_{t=1}^T \E_{z'_i \sim {\param}^{(i)}_t} \left[ \frac12 \alpha^{(i)}_t(z'_i) + \frac12 \beta^{(i)}_t(z'_i) - \zeta_t^{(i)}(z_i^*, z'_i) \right] &\ge -h(T).
  \end{align*}
  Note that $\alpha^{(i)}_t$ $\beta^{(i)}_t$, and $\zeta_t^{(i)}$ are respectively obtained by replacing $f$ with $f_t$ in the definition of  $\alpha^{(i)}$ $\beta^{(i)}$, and $\zeta^{(i)}$. We sum those inequalities together and then apply Equations \eqref{eqn:usmoptimal}, \eqref{eqn:usmlower}, and \eqref{eqn:usmupper}:
  \begin{align*}
    &-nh(T) \\
    &\le \sum_{t=1}^T \sum_{i=1}^n
      \E_{z'_i \sim {\param}^{(i)}_t} \left[
        \frac12 \alpha^{(i)}_t(z'_i) +
        \frac12 \beta^{(i)}_t(z'_i) -
        \zeta^{(i)}_t(z^*_i, z'_i)
      \right] \\
    &\le \sum_{t=1}^T \sum_{i=1}^n
      \E \left[
        \frac12 \left[ f_t(\usmlower^{(i)}) - f_t(\usmlower^{(i-1)}) \right] +
        \frac12 \left[ f_t(\usmupper^{(i)}) - f_t(\usmupper^{(i-1)}) \right] -
        \left[ f_t({\usmoptimal}^{(i-1)}) - f_t({\usmoptimal}^{(i)}) \right]
      \right] \\
    &= \sum_{t=1}^T
      \E \left[
        \frac12 \left[ f_t(\usmlower^{(n)}) - f_t(\usmlower^{(0)}) \right] +
        \frac12 \left[ f_t(\usmupper^{(n)}) - f_t(\usmupper^{(0)}) \right] -
        \left[ f_t({\usmoptimal}^{(0)}) - f_t({\usmoptimal}^{(n)}) \right]
      \right] \\
    &= \sum_{t=1}^T
      \E \left[
        \frac12 \left[ f_t(\zbf_t) - \underbrace{f_t(\usmlower^{(0)})}_{\ge 0} \right] +
        \frac12 \left[ f_t(\zbf_t) - \underbrace{f_t(\usmupper^{(0)})}_{\ge 0} \right] -
        \left[ f_t({\usmoptimal}) - f_t(\zbf_t) \right]
      \right] \\
    &\le \sum_{t=1}^T
      \E \left[
        2 f_t(\zbf_t) - f_t({\usmoptimal})
      \right]\,.
  \end{align*}
  See that the fourth equality is because the algorithm returns $\zbf_t = \usmlower^{(n)} = \usmupper^{(n)}$ at round $t.$ We finish by moving terms between sides and dividing by two:
  \begin{align*}
    \sum_{t=1}^T \E \left[2 f_t(\zbf_t) - f_t({\usmoptimal}) \right] &\ge -nh(T) \\
    \sum_{t=1}^T \E \left[f_t(\zbf_t)\right] &\ge \frac12 \sum_{t=1}^T f_t({\usmoptimal}) - \frac12 nh(T).
  \end{align*}
  Thus, our algorithm is an extended $(\tfrac12, \tfrac12)$-robust approximation.
  \medskip
  
  \paragraph{(ii) Algorithm~\ref{alg:usm-single} is bandit Blackwell reducible.}
  We first show that \Cref{alg:usm-single} is Blackwell reducible. Consider an instance $(\algspaceB,\advspaceB,\pbf)$ of Blackwell where $\algspaceB \triangleq \paramspace = \setofdistributions(\coordinatevalues)$ and $\advspaceB \triangleq \setofdistributions(\constraint \times \funcspace) = \setofdistributions(\coordinatevalues^{[n]} \times \funcspace)$. Our synthetic Blackwell adversary function is the deterministic distribution that has weight $1$ on its input (point, function) pair and $0$ anywhere else, i.e. $\advfunB(\zbf,f) = \kappa$ where $\kappa(\zbf,f) =1$. The (asymmetric) biaffine Blackwell payoff $\pbf$ is the expectation of the $\pay$ function from \Cref{eq:usmpay} over its second input:
  \begin{align*}
    \pbf(\param, \kappa) \triangleq \E_{(\zbf, f) \sim \kappa} \left[ \pay(\param, \zbf, f) \right].
  \end{align*}
  
The positive orthant $S$ is response-satisfiable since given any player 2 distribution $\kappa$ over (point, function) pairs, we can convert each pair into the marginal functions $\alpha^{(i)}$ and $\beta^{(i)}$. Averaging these marginal functions together according to their likelihood in $\kappa$ does not impact the submodularity fact we require for our proofs. We can think of $\pbf(\param,\kappa)$ as
\begin{equation*}
    \pbf(\param, \kappa) \triangleq \E_{(\zbf, f) \sim \kappa} \left[ \pay(\param, \zbf, f) \right] = \pay(\param,\zbf,f'),
\end{equation*}
for another submodular function $f'\in\funcspace$ because a weighted average of submodular function is submodular. Since for any submodular functions $f\in\funcspace$ and $\zbf\in\constraint,$ we show that we can find $\param$ such that $\pay\left(\param,\zbf,f\right)\geq0$, for any $\kappa,$ the algorithm can find $\param$ such that $\pbf(\param,\kappa)$ is in $S$. Therefore, \Cref{alg:usm-single} is Blackwell reducible.

To show that \Cref{alg:usm-single} is bandit Blackwell reducible, we need to construct an unbiased estimator for $\pbf$ and an explore sampling device $U.$ In subproblem $i,$ $U$ receives pairs of the form $(\param, \usmlower^{(i-1)})$ and returns $(\subpeek{\wbf},\subpeek{\zbf})$ such that (i) for all $f\in\funcspace,\param\in\Theta,\usmlower^{(i-1)}\in\domain,$ $\hat{\pbf}\left(\param,\advfunB(\usmlower^{(i-1)},f)\right) = f(\subpeek{\zbf})\subpeek{\wbf}$ where $(\subpeek{\wbf},\subpeek{\zbf})\sim U(\param,\usmlower^{(i-1)})$, and (ii) $\hat{\pbf}$ is an unbiased estimator for the actual payoff, i.e. $\forall\param\in\Theta,\kappa\in\advspaceB,$ we have $\E\left[\hat{\pbf}(\param,\kappa)\right] = \pbf(\param,\kappa).$

Because we would like to construct an unbiased estimator of the actual payoff $\pbf,$ which is an expectation (over $\kappa$) of the payoff function $\pay,$ which is further an affine combination of the functions $\alpha^{(i)},\beta^{(i)},$ and $\zeta^{(i)}$ on $\coordinatevalues,$ we construct unbiased estimators from function evaluations for these functions. Observe that given $\usmlower^{(i-1)},$ $U$ can immediately reconstruct the corresponding upper bound point:
\begin{align*}
\usmupper^{(i-1)} \leftarrow \usmlower^{(i-1)} \vee (\underbrace{\coordinatevalue_1, \ldots, \coordinatevalue_1}_{\text{first $(i-1)$ coordinates}}, \underbrace{\coordinatevalue_m, \ldots, \coordinatevalue_m}_{\text{last $(n-i+1)$ coordinates}})^T = (z'_1,\ldots,z_{i-1}',\underbrace{\coordinatevalue_m, \ldots, \coordinatevalue_m}_{\text{last $(n-i+1)$ coordinates}})^T.
\end{align*}

We can use $\usmlower^{(i-1)}$ and $\usmupper^{(i-1)}$ to express the marginal functions $\bm{\alpha}^{(i)}$ and $\bm{\beta}^{(i)},$
\begin{align*}
\bm{\alpha}^{(i)} \triangleq
\begin{bmatrix}
  \alpha^{(i)}(\coordinatevalue_1) \\
  \alpha^{(i)}(\coordinatevalue_2) \\
  \vdots \\
  \alpha^{(i)}(\coordinatevalue_m)
\end{bmatrix}
&=
\begin{bmatrix}
  f(\coordinatevalue_1, \usmlower^{(i-1)}_{-i}) - f(\usmlower^{(i-1)}) \\
  f(\coordinatevalue_2, \usmlower^{(i-1)}_{-i}) - f(\usmlower^{(i-1)}) \\
  \vdots \\
  f(\coordinatevalue_m, \usmlower^{(i-1)}_{-i}) - f(\usmlower^{(i-1)}) \\
\end{bmatrix} \\
\bm{\beta}^{(i)} \triangleq
\begin{bmatrix}
  \beta^{(i)}(\coordinatevalue_1) \\
  \beta^{(i)}(\coordinatevalue_2) \\
  \vdots \\
  \beta^{(i)}(\coordinatevalue_m)
\end{bmatrix}
&=
\begin{bmatrix}
  f(\coordinatevalue_1, \usmupper^{(i-1)}_{-i}) - f(\usmupper^{(i-1)}) \\
  f(\coordinatevalue_2, \usmupper^{(i-1)}_{-i}) - f(\usmupper^{(i-1)}) \\
  \vdots \\
  f(\coordinatevalue_m, \usmupper^{(i-1)}_{-i}) - f(\usmupper^{(i-1)})
\end{bmatrix}.
\end{align*}
  
  These can be used in turn to express our comparison function $\zeta^{(i)}$:  
  \begin{align*}
    \bm{\zeta}^{(i)} &\triangleq
    \begin{bmatrix}
      \zeta^{(i)}(\coordinatevalue_1, \coordinatevalue_1) &
      \zeta^{(i)}(\coordinatevalue_1, \coordinatevalue_2) &
      \cdots &
      \zeta^{(i)}(\coordinatevalue_1, \coordinatevalue_m) \\
      \zeta^{(i)}(\coordinatevalue_2, \coordinatevalue_1) &
      \zeta^{(i)}(\coordinatevalue_2, \coordinatevalue_2) &
      \cdots &
      \zeta^{(i)}(\coordinatevalue_2, \coordinatevalue_m) \\
      \vdots & \vdots & \ddots & \vdots \\
      \zeta^{(i)}(\coordinatevalue_m, \coordinatevalue_1) &
      \zeta^{(i)}(\coordinatevalue_m, \coordinatevalue_2) &
      \cdots &
      \zeta^{(i)}(\coordinatevalue_m, \coordinatevalue_m) \\
    \end{bmatrix} \\
    &= \diagonal{\bm{\alpha}^{(i)}} \mathbf{L}_{m, m} - \mathbf{L}_{m, m} \diagonal{\bm{\alpha}^{(i)}}
     + \diagonal{\bm{\beta}^{(i)}} \mathbf{U}_{m, m} - \mathbf{U}_{m, m} \diagonal{\bm{\beta}^{(i)}},
  \end{align*}
where $\mathbf{L}_{m, m}$ is the lower-triangular matrix defined by $\indexintovector{\mathbf{L}_{m, m}}{i, j} = \indicator{$i>j$}$ and $\mathbf{U}_{m, m}$ is the upper-triangular matrix defined by $\indexintovector{\mathbf{U}_{m, m}}{i, j} = \indicator{$i<j$}$. Our desired payoff function can be expressed using all three of these functions:
\begin{align*}
\pay(\param, \usmlower^{(i-1)}, f) &=
\left[
\frac12 \mathbf{1}_m \left(\bm{\alpha}^{(i)}\right)^T
+ \frac12 \mathbf{1}_m \left(\bm{\beta}^{(i)}\right)^T
+ \left(\bm{\zeta}^{(i)}\right)
\right] \param,
\end{align*}
where $\mathbf{1}_m$ is the $m$-dimensional all-ones vector. By using matrix notation, we have managed to clearly express our desired payoff function as the linear combination of many function evaluations. 

We now define the explore sampling distribution $U:\Theta\times\domain\rightarrow\Delta(\mathbb{R}^m\times\constraint)$ as follows. With $\tfrac14$ probability, we return the point $\subpeek{\zbf} = \usmlower^{(i-1)}$ and weight vector $\subpeek{\wbf} = (-2) \text{diag}(\mathbf{1}_m) \param = (-2) \mathbf{1}_m$, where $\text{diag}(\mathbf{1}_m)$ is the identity matrix with size $m\times m$. With $\tfrac14$ probability, we return the point $\subpeek{\zbf} = \usmupper^{(i-1)}$ and weight vector $\subpeek{\wbf} = (-2) \text{diag}(\mathbf{1}_m) \param = (-2) \mathbf{1}_m$. For $i=1,...,m$, with $\tfrac{1}{4m}$ probability we return $\subpeek{\zbf} = (\coordinatevalue_i, \usmlower^{(i-1)}_{-i})$ and $\subpeek{\wbf} = (4m) \left[ \frac12 \mathbf{1}_m \bm{e}_i^T + \diagonal{\bm{e}_i} \mathbf{L}_{m, m} - \mathbf{L}_{m, m} \diagonal{\bm{e}_i} \right] \param$. For $i = 1, ..., m$, with $\tfrac{1}{4m}$ probability we return the point $\subpeek{\zbf} = (\coordinatevalue_i, \usmupper^{(i-1)}_{-i})$ and weight vector $\subpeek{\wbf} = (4m) \left[ \frac12 \mathbf{1}_m \bm{e}_i^T + \diagonal{\bm{e}_i} \mathbf{U}_{m, m} - \mathbf{U}_{m, m} \diagonal{\bm{e}_i} \right] \param$. Observe that, at subproblem $i$ (essentially by construction):
\begin{align*}
    \E\left[\hat{\pbf}(\param,\kappa)\right] &= \E_{(\usmlower^{(i-1)},f)\sim\kappa}~\E\left[\hat{\pbf}\left(\param,\advfunB(\usmlower^{(i-1)},f)\right)\right] \\
    &= \E_{(\usmlower^{(i-1)},f)\sim\kappa}~\E_{(\subpeek{\wbf}, \subpeek{\zbf}) \sim \unbiasedestimator(\param, \usmlower^{(i-1)})} \left[ f(\subpeek{\zbf}) \subpeek{\wbf} \right],
\end{align*}
where
  \begin{align*}
    &\E_{(\subpeek{\wbf}, \subpeek{\zbf}) \sim \unbiasedestimator(\param, \usmlower^{(i-1)})} \left[ f(\subpeek{\zbf}) \subpeek{\wbf} \right] \\
    &= \frac14 f(\usmlower^{(i-1)}) \left[(-2) \text{diag}(\mathbf{1}_m) \param\right]
    + \frac14 f(\usmupper^{(i-1)}) \left[(-2) \text{diag}(\mathbf{1}_m) \param\right] \\
    &\quad+ \sum_{i=1}^m \frac{1}{4m} f(\coordinatevalue_i, \usmlower^{(i-1)}_{-i})
       \left[ (4m) \left[ \frac12 \mathbf{1}_m \bm{e}_i^T + \diagonal{\bm{e}_i} L_{m, m} - L_{m, m} \diagonal{\bm{e}_i} \right] \param \right] \\
    &\quad+ \sum_{i=1}^m \frac{1}{4m} f(\coordinatevalue_i, \usmupper^{(i-1)}_{-i})
       \left[ (4m) \left[ \frac12 \mathbf{1}_m \bm{e}_i^T + \diagonal{\bm{e}_i} \mathbf{U}_{m, m} - \mathbf{U}_{m, m} \diagonal{\bm{e}_i} \right] \param \right] \\
    &= f(\usmlower^{(i-1)})
       \left[-\frac12 \mathbf{1}_m \mathbf{1}_m^T - \diagonal{\mathbf{1}_m} \mathbf{L}_{m, m} + \mathbf{L}_{m, m} \diagonal{\mathbf{1}_m} \right] \param \\
    &\quad+ f(\usmupper^{(i-1)})
       \left[-\frac12 \mathbf{1}_m \mathbf{1}_m^T - \diagonal{\mathbf{1}_m} \mathbf{U}_{m, m} + \mathbf{U}_{m, m} \diagonal{\mathbf{1}_m} \right] \param \\
    &\quad+ \sum_{i=1}^m f(\coordinatevalue_i, \usmlower^{(i-1)}_{-i})
       \left[ \left[ \frac12 \mathbf{1}_m \bm{e}_i^T + \diagonal{\bm{e}_i} \mathbf{L}_{m, m} - \mathbf{L}_{m, m} \diagonal{\bm{e}_i} \right] \param \right] \\
    &\quad+ \sum_{i=1}^m f(\coordinatevalue_i, \usmupper^{(i-1)}_{-i})
       \left[ \left[ \frac12 \mathbf{1}_m \bm{e}_i^T + \diagonal{\bm{e}_i} \mathbf{U}_{m, m} - \mathbf{U}_{m, m} \diagonal{\bm{e}_i} \right] \param \right] \\
    &= \left[ \frac12 \mathbf{1}_m ({\bm{\alpha}^{(i)}})^T
             + \frac12 \mathbf{1}_m ({\bm{\beta}^{(i)}})^T
             + {\bm{\zeta}^{(i)}} \right] \param \\
    &= \pay(\param, \usmlower^{(i-1)}, f).
  \end{align*}
  
  This explore sampling device also clearly runs in polynomial-time. Finally, we have
  \begin{align*}
    \E\left[\hat{\pbf}(\param,\kappa)\right] &= \E_{(\usmlower^{(i-1)},f)\sim\kappa}~\E_{(\subpeek{\wbf}, \subpeek{\zbf}) \sim \unbiasedestimator(\param, \usmlower^{(i-1)})} \left[ f(\subpeek{\zbf}) \subpeek{\wbf} \right] \\
    &= \E_{(\usmlower^{(i-1)},f)\sim\kappa}\left[\pay(\param,\usmlower^{(i-1)},f)\right] = \pbf(\param,\kappa).
  \end{align*}
  This completes the proof of bandit Blackwell reducibility.
  
  For our bounds, we care about both the $\ell_\infty$ diameter of the payoff $\diameter{\pbf}$ and the $\ell_\infty$ diameter of the payoff estimator $\diameter{\hat{\pbf}}$. The former is bounded by $O(1)$, since for any $\param,$ the payoff function is a linear combinatrion of $O(1)$ function evaluations with $O(1)$ coefficients. The latter is bounded by $O(m)$ since aside from the $O(4m)$-scaling, the function evaluation yields a result in the range $[0, 1]$ and the remaining terms have $O(1)$ norms:
  \begin{align*}
    &\standardnorm{\mathbf{1}_m} = 1 
    \quad 
    &\standardnorm{\frac12 \mathbf{1}_m \bm{e}_i^T \param} = \frac12 \indexintovector{\param}{i} \le 1 
    \\
    & \standardnorm{\diagonal{\bm{e}_i} \mathbf{L}_{m, m} \param} = \sum_{j < i} \indexintovector{\param}{j} \le 1 \quad 
   & \standardnorm{\mathbf{L}_{m, m} \diagonal{\bm{e}_i}} = \indexintovector{\param}{i} \le 1 
   \\
   & \standardnorm{\diagonal{\bm{e}_i} \mathbf{U}_{m, m} \param} = \sum_{j > i} \indexintovector{\param}{j} \le 1\quad 
    &\standardnorm{\mathbf{U}_{m, m} \diagonal{\bm{e}_i}} = \indexintovector{\param}{i} \le 1.
  \end{align*}
  
  We complete the proof by applying Theorem \ref{thm:full-info-online-meta} and \Cref{thm:banditILO}, noting that our payoff dimension $d$ equals the number of potential values that a coordinate can take, $m$:
  \begin{align*}
    \frac12\text{-regret(\Cref{alg:full-info-backbone} applied on \Cref{alg:usm-single})} &\le O(n T^{1/2} \log^{1/2} m) \\
    \frac12\text{-regret(\Cref{alg:bandit-meta} applied on \Cref{alg:usm-single})} &\le O(nm\schange{^{2/3}} T^{2/3} \log^{1/3} m).
  \end{align*}
\[\myqed\]
\endproof{}

\subsection{Proof of Corollaries \ref{cor:usm-discrete} and \ref{cor:usm-continuous}}
\proof{\emph{Proof of \Cref{cor:usm-discrete}}.}
\label{apx:corollary-NSM}
We invoke \Cref{thm:usm} with the discretization space $\coordinatevalues = \{0,1\}.$
\[\myqed\]
\endproof{}

\proof{\emph{Proof of \Cref{cor:usm-continuous}}.}
Let $m \in \mathbb{Z}_+$ be a discretization parameter that we choose later to balance terms. We invoke \Cref{thm:usm} with discretization $\coordinatevalues = \{0, \frac1m, \frac2m, \ldots, 1\}$. Because our functions are coordinate-wise $L$-Lipschitz, the (summed) discretization error is bounded by $T L \frac1m n$ (note that the error from each subproblem is upper-bounded by $L\frac1m$ and the error from each round is upper-bounded by $L\frac1m n$). We choose $m = L T^{1/2}$ for the full-information case and $m = \schange{L^{3/5}T^{1/5}}$ for the bandit case such that:
\begin{align*}
\bigO{n T^{1/2} \log^{1/2} m} + T L \frac1m n &= \bigO{n T^{1/2} \log^{1/2} (LT)} \\
\bigO{nm\schange{^{2/3}} T^{2/3} \log^{1/3} m} + T L \frac1m n &= \bigO{n \schange{L^{2/5}T^{4/5}} \log^{1/3} (LT)}. 
\end{align*}
This completes the proof.
\[\myqed\]
\endproof


%% file: tex/DR_SM.tex
{\section{Application to Strong-DR Monotone Submodular Maximization over Downward Closed Convex Sets}
\label{subsec:DR-SM}
\paragraph{Preliminaries.} Consider the \textsc{strong-DR monotone submodular maximization} problem. Recall that a continuous multivariate function $f:\mathbb{R}^n\rightarrow[0,1]$ is strong-DR submodular if for all $i\in[n],~\bm{x}\preceq \bm{y}\in[0,1]^n$, and $\delta\geq 0,$ we have
\begin{equation*}
 f(x_i+\delta,\bm{x}_{-i})-f(\bm{x})\geq~ f(y_i+\delta,\bm{y}_{-i})-f(\bm{y})\,.
\end{equation*}
Here, $\bm{x}_{-i}$ (resp. $\bm{y}_{-i}$) is an $(n-1)$-dimensional vector with all coordinate values of $\bm{x}$ (resp. $\bm{y}$) except $i$, and $\bm{x}\preceq\bm{y}$ if and only if $\forall j\in[n]: x_j\leq y_j$. 
Alternatively, a strong-DR continuous submodular function is a weak-DR continuous submodular function that is also concave along each coordinate. We also assume that function $f:\mathbb{R}^n\rightarrow[0,1]$ is $L$-Lipschitz smooth for some constant $L>0$; that is, for all $\bm{x},\bm{v}\in\mathbb{R}^n$, 
\begin{align*}
  f(\bm{x}+\bm{v})-f(\bm{x})\geq\langle\nabla f(\bm{x}),\bm{v}\rangle -\frac{L}{2}||\bm{v}||_{2}^2,
\end{align*}
where $\nabla f(\xbf) = \left[\frac{\partial f(\xbf)}{\partial x_1},\ldots,\frac{\partial f(\xbf)}{\partial x_n}\right]^T$ for $\xbf = [x_1,\ldots,x_n]^T$. Furthermore, we assume that the gradient of $f$ has bounded $\ell_\infty$ norm in some convex set $\Pcal\subseteq \mathbb{R}^n$, that is, there is a constant $U>0$ such that $||\nabla f(\bm{x})||_\infty\leq U$ for all $\bm{x}\in\mathcal{P}$. Note that as a simple corollary of this assumption, our functions are $U-$Lipschitz continuous in $\ell_2$ norm, that is, for all $\mathbf{x},\mathbf{y}\in \mathbb{R}^n$, 
$$
\lvert f(\mathbf{x})-f(\ybf)\rvert\leq U\lVert\xbf-\ybf\rVert_2
$$

\subsection{Offline Algorithm} \label{sec:new:offline:strong_DR} Let $\Pcal$ be a downward closed convex polytope with $\ell_2$ diameter $\diampolytope$, i.e., $\diampolytope = \max_{\bm{v}_1,\bm{v}_2\in\Pcal}||\bm{v}_1-\bm{v}_2||_{2}$. We further assume this convex set is polynomial-time separable; that is, we assume we have access to a polynomial-time algorithm for exactly solving any linear optimization over this convex set. More specifically, we assume that  for any $\xbf\in \mathcal P$, we can solve the following optimization problem in polynomial time;  $\underset{\direction\in\mathcal{P}}{\max}~\direction\cdot\nabla f(\xbf)$.  
The goal of the offline problem is to maximize an 
 $L$-Lipschitz smooth monotone strong-DR submodular functions $f$ on $\mathcal{P}$, where $||\nabla f(\zbf)||_\infty\leq U$ for all $\zbf\in\mathcal{P}$. 
For the offline problem, the Frank-Wolfe variant for monotone strong-DR submodular function gives a $(1-1/e)$-approximation, which is known to be a tight approximation factor. This algorithm is first introduced in \cite{feldman2011unified} and \cite{calinescu2011maximizing} for the special case of multi-linear extension of set submodular functions, and is later extended to general strong-DR continuous submodular functions in \cite{bian2016guaranteed} and \cite{mokhtari2018stochastic}. We present this offline algorithm in Algorithm \ref{alg:DR-SM}. We then transform it into an online adversarial learning algorithm under the full-information (Section \ref{sec:full_strong_DR}) and bandit settings (Section \ref{sec:bandit_strong_DR}).

\begin{algorithm}
  \caption{Frank-Wolfe variant for maximizing monotone strong-DR submodular function in a downward closed convex set. (c.f.,~\cite{bian2016guaranteed})}
  \label{alg:DR-SM}
  \textbf{Input:} Strong-DR monotone submodular function $f:\mathcal{P}\rightarrow[0,1]$ that is $L$-Lipschitz smooth with $||\nabla f(\xbf)||_\infty\leq U$ for all $\xbf\in\mathcal{P}$, where  $\mathcal{P}$ is a downward closed and convex polytope  with $\ell_2$ diameter $\diampolytope$. The number of iterations $N$.\\
  \textbf{Output:} Point $\zbf \in \mathcal{P}$.\\
  Initialize $\zbf^{(0)} \leftarrow \mathbf{0}$.\\
  \For{iteration $i = 1, 2, \ldots, N$}{\vspace{0.2em}
   
   {\underline{Local Optimization Step}\\}
   Choose $\direction^{(i)}\in\mathcal{P}$ such that
   \begin{equation*}
   \begin{aligned}
    \direction^{(i)}\in\underset{\direction\in\mathcal{P}}{\argmax}~\direction\cdot\nabla f(\zbf^{(i-1)})
   \end{aligned}
   \end{equation*}
   (done in \cite{bian2016guaranteed} via maximizing a linear objective).\\
   {\underline{Local Update Step}\\}
   Set $\zbf^{(i)}\leftarrow\zbf^{(i-1)}+\frac{1}{N}\direction^{(i)}$.
  
  }
  \Return{$\zbf\leftarrow\zbf^{(N)}$}.\\
\end{algorithm}

The Frank-Wolfe variant algorithm takes the number of iterations $N$ as input, which affects the approximation factor and also an additive error in its performance. The approximation factor of the algorithm for a finite $N$ is $\gamma_N = 1- \left((1-1/N)\right)^N$, which goes to $1-1/e$ from above as $N$ goes to infinity, and the additive error is in the order of $O(\frac{LR_{\Pcal}}{N})$, which goes to zero as $N$ goes to infinity (c.f., \cite{bian2016guaranteed}. That is, 
$f(\zbf^{(N)})\geq\gamma_N \max_{\zbf\in \Pcal} f(\zbf)-\frac{LR_{\Pcal}}{N}
$, where $\zbf^{(N)}$ is the output of the Frank-Wolfe  algorithm after $N$ iterations. 
(We will also see the same bound later in the proof of \Cref{thm:DR-SM}).)  In the language of $\offlinemetatext$, the parameter space is $\paramspace = \mathcal{P}\subset \mathbb{R}^n$ and $\paramdimension = n$. In each subproblem $i\in [N]$, the algorithm picks the direction $\direction^{(i)}$ in the polytope $\mathcal{P}$ that maximizes  function $\langle\direction,\nabla f(\zbf^{(i-1)})\rangle$,  (i.e., $\direction^{(i)}\in\underset{\direction\in\mathcal{P}}{\argmax}~\direction\cdot\nabla f(\zbf^{(i-1)})$.)  Intuitively, $\direction^{(i)}$ is the direction along which we can maximize the improvement in the function value while still remaining feasible. Picking the direction inside the polytope eliminates the need for projecting the obtained point back to $\mathcal{P}$ at each iteration, which is usually an essential step in the Frank-Wolfe algorithm. We define the vector payoff function to be
\begin{equation*}
  \pay(\direction^{(i)},\zbf^{(i-1)},f)=
  \begin{bmatrix}
  -\direction^{(i)}\cdot\ybf^{(i)}\\
  \ybf^{(i)}
  \end{bmatrix}
\end{equation*}
where 
\begin{align*}
  \ybf^{(i)}&\triangleq\nabla f(\zbf^{(i-1)}),
\end{align*}
is the gradient of the function on the point from the previous iteration. The target set $S$ is the polar cone of the $\textrm{Cone}(1\oplus\mathcal{P})$, denoted by $ \textrm{Cone}(1\oplus\mathcal{P})^\circ$. Note that for $\bm{x}\in\mathbb{R}^{d_1}$ and $\bm{y}\in\mathbb{R}^{d_2}$,  $\bm{x}\oplus\bm{y} = \begin{bmatrix}
\bm{x}\\
\bm{y}
\end{bmatrix}\in\mathbb{R}^{d_1+d_2}$.
Moreover, for a cone $C\subseteq\mathbb{R}^d$, the polar cone of $C$, denoted by $C^\circ$, is defined as $C^\circ = \{\bm{\theta}\in\mathbb{R}^d:\bm{\theta}\cdot\bm{x}\leq 0,\forall \bm{x}\in C\}$. See that when $\pay(\direction^{(i)},\zbf^{(i-1)},f)\in S$, we have
\begin{align*}
  -\direction^{(i)}\cdot\ybf^{(i)} + \ybf^{(i)}\cdot\direction\leq 0,~~\forall\direction\in\mathcal{P},
\end{align*}
which implies
\begin{align*}
  \direction^{(i)}\in\argmax_{\direction\in\mathcal{P}}\direction\cdot\ybf^{(i)} = \argmax_{\direction\in\mathcal{P}}\direction\cdot\nabla f(\zbf^{(i-1)}). 
\end{align*}
So, picking $\direction^{(i)}$ that maximizes $\direction\cdot\nabla f(\zbf^{(i-1)})$ over $\direction\in\mathcal{P}$ is equivalent to picking $\direction^{(i)}$ such that the payoff $\pay(\direction^{(i)},\zbf^{(i-1)},f)$ is in $S$.

Algorithm \ref{alg:DR-SM} is a variant of $\offlinemetatext$ (\Cref{alg:offline-meta}) where the local optimization step is replaced with picking $\direction^{(i)}\in\mathcal{P}$ such that the payoff function evaluated on $\direction^{(i)}$ falls in the target set $S$. Recall that in $\offlinemetatext$, $\direction^{(i)}$ is picked in the local optimization step such that the payoff function on $\direction^{(i)}$ is non-negative, i.e., $\pay(\direction^{(i)},\zbf^{(i-1)},f)\geq\mathbf{0}$. In this variant, we want to pick $\direction^{(i)}$ such that $\pay(\direction^{(i)},\zbf^{(i-1)},f)\in S$ for some convex set $S$, which is not a positive orthant. 

Now, consider a Blackwell sequential game, as defined in Section \ref{sec:blackwell}, where $\algspaceB = \paramspace = \mathcal{P}$, $\advspaceB = [-U,U]^n$, $\pbf(\bm{x},\bm{y}) = \begin{bmatrix}
-\bm{x}\cdot\bm{y}\\
\bm{y}
\end{bmatrix}\in\mathbb{R}^{n+1}$, and $S= \textrm{Cone}(1\oplus\mathcal{P})^\circ\subset{\mathbb{R}^{n+1}}$. Note that the target set and payoff are $(n+1)$-dimensional. The set $S$ is response-satisfiable, since for every player 2's action $\bm{y}\in\advspaceB$, there exists a player 1's action, $\bm{x}^*\in\argmax_{\bm{x}\in\mathcal{P}}\bm{x}\cdot\bm{y}$, such that $\pbf(\bm{x}^*,\bm{y})\in S$. Specifically, when $\bm{x}^*\in\argmax_{\bm{x}\in\mathcal{P}}\bm{x}\cdot\bm{y}$, then $\bm{x}^*\cdot\bm{y}\geq\bm{x}\cdot\bm{y}$ for all $\bm{x}\in\mathcal{P}$, which implies $\pbf(\bm{x}^*,\bm{y})\cdot[1,\bm{x}]^T\leq 0$ for all $\bm{x}\in\mathcal{P}$, further implying $\pbf(\bm{x}^*,\bm{y})\in S$. Note that since $S$ is response-satisfiable and $S$ is polynomial-time separable, there exists a polynomial-time algorithm $\BlackwellAlg$, based on Theorem \ref{def:blackwell-thm}, such that for any sequence of actions $\ybf_1,\ldots,\ybf_T\in\advspaceB$ generated by an adaptive adversary, $\BlackwellAlg$ can generate a sequence of actions $\xbf_1,\ldots,\xbf_T\in\mathcal{P}$ such that
\smallskip
\begin{align}
\label{eq:guarantee_matroid}
d_\infty\left(\frac{1}{T}\sum_{t=1}^T\pbf(\xbf_t,\ybf_t),S\right)\leq\bigO{D_\infty(\pbf)\sqrt{\frac{\log{d_{\pbf}}}{T}}} = \bigO{U\sqrt{\frac{\log{n}}{T}}}.
\end{align}
\vspace{2mm}

\subsection{Online Learning Algorithm for Full-Information Setting} \label{sec:full_strong_DR}
In the online problem, we consider an adversarial sequence of strong-DR submodular functions $f_1,\ldots,f_T$, where for each $1\leq t\leq T$, $f_t:\mathcal{P}\rightarrow[0,1]$ for some downward closed polynomial-time separable convex polytope $\Pcal\subseteq \mathbb{R}^n$ with $\ell_2$ diameter $\diampolytope$. Moreover, for each $t\in[T]$, $f_t$ is $L$-Lipschitz smooth and $||\nabla f_t(\zbf)||_\infty\leq U$ for all $\zbf\in\mathcal{P}$. Then, for each subproblem $i\in[N]$ in Algorithm~\ref{alg:DR-SM}, we can replace the local optimization step with a polynomial-time online algorithm $\BlackwellAlg^{(i)}$ where $\xbf_t=\direction_t^{(i)}$ and $\ybf_t = \textrm{AdvB}(\zbf^{(i-1)}_t,f_t) = \nabla f_t(\zbf_t^{(i-1)})$. For subproblem $i$, the guarantee in \Cref{eq:guarantee_matroid} becomes
\begin{align}
\label{eq:guarantee_matroid_payoff}
d_\infty\left(\frac{1}{T}\sum_{t=1}^T\pay(\direction_t^{(i)},\zbf_t^{(i-1)},f_t),S\right)\leq\bigO{U\sqrt{\frac{\log{n}}{T}}}.
\end{align}
We utilize $N$ parallel runs of $\BlackwellAlg$, one for each subproblem as shown in Algorithm~\ref{alg:online-DR-SM}. Note that we do not directly apply our meta \Cref{alg:full-info-backbone} in this problem because the target set $S$ is not the positive orthant. But due to this subtle difference, we derive our results with similar  approach and our algorithm follows the same framework.

\begin{algorithm}
  \caption{Full-information online learning algorithm for Algorithm \ref{alg:DR-SM}}
  \label{alg:online-DR-SM}
  \textbf{Input:} A sequence of $L$-Lipschitz smooth strong-DR monotone submodular functions $f_1,\ldots,f_T:\mathcal{P}\rightarrow[0,1]$ such that $||\nabla f_t(\zbf)||_\infty\leq U$ for all $\zbf\in\mathcal{P}, t\in[T]$, where $\Pcal$ is a downward closed and convex polytope   with $\ell_2$ diameter $\diampolytope$. $N$ algorithms $\{\BlackwellAlg^{(i)}\}_{i=1,\ldots,N}$, where each $\BlackwellAlg^{(i)}$ is an online learning algorithm for a Blackwell sequential game with $\algspaceB=\mathcal{P},\advspaceB=[-U,U]^n,\pbf(\xbf,\ybf)=\begin{bmatrix}
  -\bm{x}\cdot\bm{y}\\
  \bm{y}
  \end{bmatrix}$, and $S = \textrm{Cone}(1\oplus\mathcal{P})^\circ$.\\
  \textbf{Output:} $\zbf_1,\ldots,\zbf_T\in\mathcal{P}$.\\
  \For{$t=1,\ldots,T$}{\vspace{0.2em}
  Initialize $\zbf^{(0)}_t \leftarrow \mathbf{0}$.\\
  \For{iteration $i = 1, 2, \ldots, N$}{\vspace{0.2em}
   
   Choose $\direction^{(i)}_t\in\mathcal{P}$ by querying online algorithm $\BlackwellAlg^{(i)}$ given the update parameters and vector payoffs prior to round $t$ in the Blackwell sequential game of subproblem $i$, i.e.,  $\direction_t^{(i)}\leftarrow\BlackwellAlg^{(i)}\left(\direction_1^{(i)},\ldots,\direction_{t-1}^{(i)},\left\{\pbf\left(\direction_\tau^{(i)},\ybf_{\tau}^{(i)})\right)\right\}_{\tau\in[1:t-1]}\right)$.\\
   Set $\zbf_t^{(i)}\leftarrow \zbf_t^{(i-1)} + \frac{1}{N}\direction_t^{(i)}$.
  
  }
  \Return{$\zbf_t\leftarrow\zbf_t^{(N)}$}.\\
  \emph{Adversary reveals function $f_t$.}\\
  \For{iteration $i = 1, 2, \ldots, N$}{\vspace{0.2em}
    Compute $\nabla f_t(\zbf_t^{(i-1)})$.\\
    Give feedback $\pbf(\direction_t^{(i)},\ybf_t^{(i)})\leftarrow\begin{bmatrix}
-\bm{\direction_t^{(i)}}\cdot\nabla f_t(\zbf_t^{(i-1)})\\
\nabla f_t(\zbf_t^{(i-1)})
\end{bmatrix}$ to the Blackwell algorithm $\BlackwellAlg^{(i)}$. Note that $\ybf_t^{(i)} = \textrm{AdvB}(\zbf_t^{(i-1)},f_t) = \nabla f_t(\zbf_t^{(i-1)})$.
  }
  }
\end{algorithm}

\begin{theorem}\normalfont{\textbf{(Online learning for strong-DR monotone SM maximization over downward closed convex sets)}}
\label{thm:DR-SM}
Let $\Pcal$ be a downward closed convex polytope with $\ell_2$ diameter $\diampolytope$. Consider the online learning problem of maximizing strong-DR monotone submodular functions on $\mathcal{P}$, which are $L$-Lipschitz smooth and have bounded gradient of $U$ for any point $\zbf\in\mathcal{P}$. Then, under full-information setting, \Cref{alg:online-DR-SM} with $N$ iterations
 obtains $\bigO{{\diampolytope}U\sqrt{Tn\log{n}} + \frac{TL{\diampolytope}^2}{N}}$ $\gamma_N$-regret, where $T$ is the number of rounds and $\gamma_N = 1-(1-1/N)^{N}\geq \left(1-1/e\right)$. The benchmark in the regret bounds is $\gamma_N\max_{\zbf\in \mathcal{P}}\sum_{t=1}^T f_t(\zbf)$. By setting $N= \frac{\sqrt{T}LR_{\mathcal{P}}}{U\sqrt{n\log{n}}}$, we get $\bigO{{\diampolytope}U\sqrt{Tn\log{n}}}$ $\gamma_N$-regret. 
\end{theorem}

Before proving the above theorem, we first prove the following technical lemma using strong-DR submodularity property of our functions. The proof of Lemma \ref{lemma:iterative_z_matroid} is presented at the end of this section.
\begin{lemma}
\label{lemma:iterative_z_matroid}
Consider any iteration $i\in [N]$ of the Frank-Wolfe algorithm (Algorithm \ref{alg:DR-SM}). For all $\zbf\in\mathcal{P}$, we have
\begin{equation*}
  \direction^{(i)}\cdot\nabla f(\zbf^{(i-1)})\geq f(\zbf) - f(\zbf^{(i-1)}).
\end{equation*}
\end{lemma}

\proof{\emph{Proof of Theorem~\ref{thm:DR-SM}.}}
We show that \Cref{alg:online-DR-SM} works in two steps: (i) proving a variant of the extended robustness property and (ii) proving a variant of the implication of Blackwell reducibility.

We start with a variant of the extended robustness property. Suppose that $\left\{\direction_t^{(i)},\zbf_t^{(i)}\right\}_{t\in[T],0\leq i\leq N}$ are the values from \Cref{alg:online-DR-SM} with $N$ iterations. We show that if the following equation holds for some function $h$:
\begin{equation}
\label{eq:condition_h_matroid}
  \forall i \in [N], \quad \sum_{t=1}^T{\direction}^{(i)}_t\cdot\nabla f_t(\zbf_t^{(i-1)})\geq\max_{\direction\in\mathcal{P}}\sum_{t=1}^T\direction\cdot\nabla f_t(\zbf_t^{(i-1)}) - h(T),
\end{equation}
then we have
\begin{equation}
\label{eq:extend-Rank-1}
  \forall\zbf^*\in\mathcal{P},\quad \sum_{t=1}^Tf_t(\zbf_t)\geq\gamma_N\sum_{t=1}^Tf_t(\zbf^*)-h(T)-\frac{TL{\diampolytope}^2}{2N}\,,
\end{equation}
where $\gamma_N = 1- \left(1-\frac{1}{N}\right)^N$. 
To show this, we use Lemma~\ref{lemma:iterative_z_matroid}. Note that for each round $t$ and subproblem $i\in[N]$, we have
\begin{equation*}
\begin{aligned}
  f_t(\zbf_t^{(i)}) - f_t(\zbf_t^{(i-1)}) &= f_t(\zbf_t^{(i-1)} +\frac{1}{N}\direction_t^{(i)}) - f_t(\zbf_t^{(i-1)})\\
  &\overset{(1)}{\geq} \frac{1}{N}{\nabla f_t(\zbf_t^{(i-1)})\cdot\direction_t^{(i)}} -\frac{1}{2}\frac{L||\direction_t^{(i)}||_{2}^2}{N^2}\\
  &\overset{(2)}{\geq}\frac{1}{N}{\nabla f_t(\zbf_t^{(i-1)})\cdot\direction_t^{(i)}} - \frac{1}{2}\frac{L{\diampolytope}^2}{N^2}.
\end{aligned}
\end{equation*}
Inequality~(1) holds because $f_t$ is $L$-Lipschitz smooth and Inequality~(2) holds since $\direction^{(i)}_t\in\mathcal{P}$ and $\mathcal{P}$ is a downward closed convex polytope with $\ell_2$ diameter $\diampolytope$. This implies that 
$
  ||\direction_t^{(i)}||_{2}^2\leq{\diampolytope}^2
$.


Summing this for all $t\in[T]$, for any $\zbf^*\in\Pcal$, we have
\begin{equation*}
\begin{aligned}
  \sum_{t=1}^T\left(f_t(\zbf_t^{(i)}) - f_t(\zbf_t^{(i-1)})\right) 
  &{\geq}\frac{1}{N}\sum_{t=1}^T\nabla f_t(\zbf_t^{(i-1)})\cdot\direction_t^{(i)} - \frac{1}{2}\frac{TL{\diampolytope}^2}{N^2} \\
  &\overset{(1)}{\geq}\frac{1}{N}\max_{\direction\in\mathcal{P}}\sum_{t=1}^T\direction\cdot\nabla f_t(\zbf_t^{(i-1)}) - \frac{1}{N}h(T) - \frac{1}{2}\frac{TL{\diampolytope}^2}{N^2}\\
  &\overset{(2)}{\geq}\frac{1}{N}\left(\sum_{t=1}^Tf_t(\zbf^*)-\sum_{t=1}^Tf_t(\zbf_t^{(i-1)})\right) - \frac{1}{N}h(T) - \frac{1}{2}\frac{TL{\diampolytope}^2}{N^2},
\end{aligned}
\end{equation*}
where Inequality~(1) holds because of the assumption in \Cref{eq:condition_h_matroid}, and Inequality~(2) is reached by applying \Cref{lemma:iterative_z_matroid}. Rearranging the terms, for any $\zbf^*\in\Pcal$, we have
\begin{equation*}
\sum_{t=1}^T\left(f_t(\zbf_t^{(i)}) - f_t(\zbf^{*})\right) \geq\left(1-\frac{1}{N}\right)\left(\sum_{t=1}^T\left(f_t(\zbf_t^{(i-1)}) - f_t(\zbf^{*})\right) \right) - \frac{1}{N}h(T) - \frac{1}{2}\frac{TL{\diampolytope}^2}{N^2}, 
\end{equation*}
and iterating from $i=1,2,\ldots,N$, we get 
\begin{equation*}
\sum_{t=1}^T\left(f_t(\zbf_t^{(N)}) - f_t(\zbf^{*})\right) \geq\left(1-\frac{1}{N}\right)^N\left(\sum_{t=1}^T\left(f_t(\zbf_t^{(0)}) - f_t(\zbf^{*})\right) \right) - h(T) - \frac{1}{2}\frac{TL{\diampolytope}^2}{N}.
\end{equation*}
Since $f_t(\zbf_t^{(0)}) = f_t(\mathbf{0})\geq 0$, for any $\zbf^*\in\Pcal$, we have
\begin{equation*}
\begin{aligned}
  \sum_{t=1}^Tf_t(\zbf_t^{(N)})\geq \left(1-\left(1-\frac{1}{N}\right)^N\right)
  \sum_{t=1}^Tf_t(\zbf^*) - h(T) - \frac{TL{\diampolytope}^2}{2N} = \gamma_N\sum_{t=1}^Tf_t(\zbf^*)-h(T)-\frac{TL{\diampolytope}^2}{2N},
\end{aligned}
\end{equation*} which is the desired inequality in Equation \eqref{eq:extend-Rank-1}.

We now prove that the regret guarantee in \Cref{eq:guarantee_matroid_payoff}, which is obtained from \Cref{def:blackwell-thm},  implies \Cref{eq:condition_h_matroid} for $h(T) = \bigO{R_\Pcal U\sqrt{T{n\log{n}}}}$. From the Blackwell guarantee (i.e., Equation \eqref{eq:guarantee_matroid_payoff}), we have
\begin{align}
\label{eq:first_line_matroid}
\bigO{U\sqrt{\frac{\log{n}}{T}}} 
&\geq d_\infty\left(\frac{1}{T}\sum_{t=1}^T\pay(\direction_t^{(i)},\zbf_t^{(i-1)},f_t),S\right)\\
&\overset{(1)}{\geq}\frac{1}{\sqrt{n+1}}d_2\left(\frac{1}{T}\sum_{t=1}^T\pay(\direction_t^{(i)},\zbf_t^{(i-1)},f_t),S\right)\nonumber\\
&\overset{(2)}{=}\frac{1}{\sqrt{n+1}}\quad\underset{\bm{w}\in \textrm{Cone}(1\oplus\mathcal{P})\cap B_2^{n+1}(1)}{\max}~\bm{w}\cdot{\frac{1}{T}\sum_{t=1}^T\pay(\direction_t^{(i)},\zbf_t^{(i-1)},f_t)}\nonumber\\
&\overset{(3)}{=}\frac{1}{\sqrt{n+1}}\underset{\direction\in\mathcal{P}}{\max}~\frac{(1\oplus\direction)}{||1\oplus\direction||_{2}}\cdot{\frac{1}{T}\sum_{t=1}^T\pay(\direction_t^{(i)},\zbf_t^{(i-1)},f_t)}\nonumber\\
&\overset{(4)}{=}\frac{1}{\sqrt{n+1}}\underset{\direction\in\mathcal{P}}{\max}~\dfrac{\left(\frac{1}{T}\sum_{t=1}^T\ybf_t^{(i)}\cdot\direction - \frac{1}{T}\sum_{t=1}^T\ybf_t^{(i)}\cdot\direction_t^{(i)}\right)}{||1\oplus\direction||_{2}}\nonumber\\
\label{eq:last_line_matroid}
&\overset{(5)}{\geq}\frac{1}{\sqrt{n+1}\sqrt{1+{\diampolytope}^2}}\frac{1}{T}~{\left(\underset{\direction\in\mathcal{P}}{\max}\sum_{t=1}^T\ybf_t^{(i)}\cdot\direction - \sum_{t=1}^T\ybf_t^{(i)}\cdot\direction_t^{(i)}\right)}
\end{align}
Here, $B_2^{n+1}(1)$ is a $(n+1)$-dimensional $\ell_2$ unit ball. Inequality~(1) holds since the dimension of the payoff vector is $(n+1)$, Inequality~(2) follows from Lemma 13 in \cite{abernethy2011blackwell}. (The lemma says that for every convex cone $C$ in $\mathbb{R}^d$, $d_2(\xbf, C) = \max_{\bm{\theta}\in C^{\circ}\cap B_2^d(1)}\langle\bm{\theta},\xbf\rangle$.) Inequality~(3) follows from rewriting $\bm{w}$ as a function of $\direction\in\mathcal{P}$, Inequality~(4) holds because the definition of the payoff vector where $\ybf_t^{(i)} = \nabla f_t(\zbf_t^{(i-1)})$, and Inequality~(5) holds because $|| 1\oplus\direction||_{2}\leq\sqrt{1+{\diampolytope}^2}$ for any $\direction\in \Pcal$.
Putting things together and rearranging, we get
\begin{equation*}
\begin{aligned}
\sum_{t=1}^T\direction_t^{(i)}\cdot\ybf_t^{(i)}\geq\max_{\direction\in\mathcal{P}}\direction\cdot\sum_{t=1}^T\ybf_t^{(i)} - \bigO{U\sqrt{\frac{\log{n}}{T}}\cdot {\diampolytope}\sqrt{n}T} = \max_{\direction\in\mathcal{P}}\direction\cdot\sum_{t=1}^T\ybf_t^{(i)} - \bigO{{\diampolytope}U\sqrt{T{n\log{n}}}},
\end{aligned}
\end{equation*}
which is \Cref{eq:condition_h_matroid} with $h(T) = \bigO{{\diampolytope}U\sqrt{T{n\log{n}}}}$.

In conclusion, by invoking Equation \eqref{eq:extend-Rank-1} with $h(T) = \bigO{{\diampolytope}U\sqrt{T{n\log{n}}}}$, the total regret of \Cref{alg:online-DR-SM} is $\bigO{{\diampolytope}U\sqrt{Tn\log{n}} + \frac{TL{\diampolytope}^2}{N}}$. By setting $N= \frac{\sqrt{T}L{\diampolytope}}{U\sqrt{n\log{n}}}$, we get $\bigO{{\diampolytope}U\sqrt{Tn\log{n}}}$ $\gamma_N$-regret, as desired. 

\[\myqed\]
\endproof{}
For completeness, we provide the proof of the technical lemma used in the above proof.
\begin{proof}{\emph{Proof of Lemma~\ref{lemma:iterative_z_matroid}.}}
Let $\direction^* = \zbf\vee\zbf^{(i-1)} - \zbf^{(i-1)}$. Observe that
\begin{equation*}
\begin{aligned}
  f(\zbf)-f(\zbf^{(i-1)}) &\overset{(1)}{\leq} f(\zbf\vee\zbf^{(i-1)})-f(\zbf^{(i-1)}) = f(\zbf^{(i-1)}+\direction^*) - f(\zbf^{(i-1)}) \\
  &\overset{(2)}{\leq}\direction^*\cdot\nabla f(\zbf^{(i-1)}) \overset{(3)}{\leq}\direction^{(i)}\cdot\nabla f(\zbf^{(i-1)}),
\end{aligned}
\end{equation*}
where Inequality~(1) holds because $f$ is a monotone function, Inequality~(2) holds because of the concavity (along each direction) property of strong-DR submodular functions, and Inequality~(3) holds because $\direction^{(i)}$ maximizes $\direction\cdot\nabla f(\zbf^{(i-1)})$ over $\mathcal{P}$ and $\direction^*\in\mathcal{P}$ since $\mathcal{P}$ is downward closed.
\[\myqed\]
\end{proof}

\subsubsection{Application to Set Submodular Maximization Subject to Matroids}
As an important application of \Cref{thm:DR-SM}, we consider the online learning version of maximizing monotone set submodular functions subject to a matroid constraint. We start by stating the following simple corollary of \Cref{thm:DR-SM} for maximizing strong-DR monotone submodular maximization over matroid polytopes, which is the key to obtain similar results for maximizing monotone set submodular maximization subject to matroids constraints.
\begin{corollary}
\label{corr:strong-DR-SM-matroid}
Let $M = (V, \mathcal{I})$ be a matroid with $|V| = n$ and  $\textrm{Rank}(M) = R$. Define  $\Pcal_{M} = \textrm{Conv}(\{\mathbf{1}_I:I\in\mathcal{I}\})$ as a matroid polytope associated with matroid $M = (V, \mathcal{I})$.
When we restrict our domain to $\Pcal_{M}$ (i.e., if we set $\Pcal$ to $\Pcal_{M}$), we obtain a total regret of $\bigO{RU\sqrt{Tn\log{n}} + \frac{TL{R}^2}{Nn}}$ for \Cref{alg:online-DR-SM}. Setting $N= \frac{\sqrt{T}LR}{Un\sqrt{n\log{n}}}$, we get $\bigO{RU\sqrt{Tn\log{n}}}$ $\gamma_N$-regret.

\end{corollary}
\begin{proof}{Proof of Corollary \ref{corr:strong-DR-SM-matroid}.}
First observe that the $\ell_2$ diameter of the polytope $\Pcal_{M}$  becomes $\frac{R}{\sqrt{n}}$ since for any $\direction$ in matroid polytope $\Pcal_M$,  we have $$||\direction||_{2}^2\leq n\left(\frac{\textrm{Rank}(M)}{n}\right)^2 = \frac{R^2}{n}.$$ Moreover, we have $|| 1\oplus\direction||_{2}\leq\max\{1,\sqrt{2}\max_{\direction\in\mathcal{P}}||\direction||_{2}\}\leq\max\{1,\sqrt{2}\frac{R}{\sqrt{n}}\}$, which results in $h(T) = \bigO{RU\sqrt{T{n\log{n}}}}$ in \Cref{eq:condition_h_matroid}. This implies a total regret of $\bigO{RU\sqrt{Tn\log{n}} + \frac{TL{R}^2}{Nn}}$ for \Cref{alg:online-DR-SM}. Setting $N= \frac{\sqrt{T}LR}{Un\sqrt{n\log{n}}}$, we get $\bigO{RU\sqrt{Tn\log{n}}}$ $\gamma_N$-regret. 
\end{proof}


Next, we seek to understand if we can obtain similar results as in  Corollary \ref{corr:strong-DR-SM-matroid} for maximizing monotone set submodular functions with a matroid constraint.  To do so, we rely on the notion of \emph{multi-linear extension} of set   submodular functions and \emph{pipage rounding algorithm} proposed in  \cite{calinescu2011maximizing}. Using these two notions, \cite{calinescu2011maximizing} further propose an optimal approximation algorithm for maximizing monotone set submodular functions with a matroid constraint. We briefly describe this approach below and then show how to combine it with \Cref{alg:DR-SM} to extend our result to the full information online learning version of this problem.\footnote{Because of a subtle technical reason, this result does not immediately extend to the bandit setting; see \Cref{rem:bandit-DR-SM}.}
 
Given a set submodular function $f:2^{[n]}\rightarrow[0,1]$, the following continuous extension of this function, known as the {multi-linear extension}, is a strong-DR submodular continuous function \citep{calinescu2011maximizing}:
\begin{equation}
\label{eq:MLE-SM}
\forall \mathbf{x}\in[0,1]^n:~~~f^{\textsc{MLE}}(\mathbf{x})=\sum_{S\subseteq [n]}f(S)\prod_{i\in S}x_i\prod_{i\notin S}(1-x_i)\,.
\end{equation}
Notably, $f^{\textsc{MLE}}(\mathbf{x})$ is essentially the expected value of $f(S)$ at a randomized set $S$, where each element $i\in[n]$ is placed in $S$ independently with probability $x_i$. As a result, (i) the two functions have the same value when $\mathbf{x}\in\{0,1\}^n$, and (ii) if $\mathbf{x}^*$ maximizes $f(\mathbf{x})$ over a matroid polytope, $f^{\textsc{MLE}}(\mathbf{x}^*)$ should be equal to the maximum value of $f(S)$ over independent sets $S$ of its associated matroid. This last property is a simple consequence of the existence of polynomial-time loss-less randomized rounding algorithms in the matroid polytope for monotone submodular functions, such as {the pipage rounding} introduced in \cite{calinescu2011maximizing}. Given a point $\mathbf{x}$ in the matroid polytope, the pipage randomized rounding algorithm returns a randomized set $\tilde{S}$ such that $\tilde{S}$ is always an independent set of the matroid and  $\expect{f(\tilde{S})}\geq f^{\textsc{MLE}}(\mathbf{x})$. Applying such a loss-less rounding algorithm to the maximizer point $\mathbf{x}^*$ will result in a distribution over independent sets of the matroid, where every point in the support of this distribution should be a maximizer of $f(S)$ over independent sets of the polytope (otherwise, $\mathbf{x}^*$ could not be the maximizer). We should also highlight that one can aim to approximately maximize the multi-linear extension of the monotone submodular function subject to a matroid constraint as it is strong-DR, Lipschitz continuous and smooth, e.g., see \cite{bian2016guaranteed}. This can be done using a first-order method such as \Cref{alg:DR-SM}. Then, we can use the pipage rounding algorithm to obtain a randomized approximation algorithm with \emph{exactly} the same approximation guarantee in-expectation for monotone set submodular maximization subject to a matroid constraint. 

We now sketch how to use the multi-linear extension and the pipage rounding algorithm in a full-information online learning setting, and how to obtain a similar result as in \Cref{corr:strong-DR-SM-matroid} but for monotone set submodular functions $f_1,\ldots,f_T$. The idea is running the online learning algorithm of the previous section (\Cref{alg:online-DR-SM}) on the sequence of functions $f^{\textsc{MLE}}_1,\ldots,f^{\textsc{MLE}}_T$ (which are strong-DR continuous submodular functions) given the full information feedback $f_t(\cdot)$ at the end of each round $t$ (will be detailed later). Running \Cref{alg:online-DR-SM} generates points $\zbf_1,\ldots,\zbf_T$ in the matroid polytope. We then sample randomized sets $\tilde{S}_1,\ldots,\tilde{S}_T$ at the end of each round $t$ using the pipage randomized rounding algorithms, so that these sets are independent sets in the matroid and also $\expect{f_t(\tilde{S}_t)}\geq f^{\textsc{MLE}}(\zbf_t)$.
In this way, we have:
$$
\gamma_N\cdot\underset{S\in\mathcal{I}}{\max}{\sum_{t\in[T]}f_t(S)}-\sum_{t\in[T]}\expect{f_t(\tilde{S}_t)}\leq \gamma_N\cdot\underset{\mathbf{z}\in\textrm{Conv}(\mathcal{I})}{\max}{\sum_{t\in[T]}f^{\textsc{MLE}}_t(\zbf)}-\sum_{t\in[T]}f^{\textsc{MLE}}_t(\zbf_t)~,
$$
and therefore the regret bound of \Cref{alg:online-DR-SM} carries over to this setting. 

To run the full-information online learning algorithm \Cref{alg:online-DR-SM} as described above, we need to be able to construct the full information feedback $f_t^{\textsc{MLS}}(\mathbf{z})$ (for all $\mathbf{z}\in [0,1]^n$) for the underlying strong-DR continuous submodular function $f_t^{\textrm{MLE}}$ given the set function full-information feedback $f_t$ that we receive in each round. Given the full information feedback $f_t$, $f_t^{\textsc{MLS}}(\mathbf{\zbf})$ can be estimated at any point $\mathbf{z}$ through sampling and taking average. The sampling procedure evaluates $f_t$ at a randomized set $S$, where every element $i$ is placed in $S$ independently with probability $z_i$. By repeating the same process and averaging one can obtain an estimation of $f_t^{\textsc{MLS}}(\mathbf{\zbf})$ at any accuracy level.  (For example, with samples complexity of $\Omega(T^4)$, the accuracy level in estimating multi-linear extension (and hence its gradient) at any point will be of $O(\frac{1}{T^2})$ by applying a concentration bound such as the Chernoff-Hoeffding bound. Note that such an additive error is negligible in our regret bounds as it causes an additional regret that is in the order of $o(1)$ by following similar lines as in the proof of \Cref{thm:DR-SM}. We omit the details for brevity.)

\subsection{Online Learning Algorithm  for Bandit Setting} \label{sec:bandit_strong_DR} In the bandit setting, we observe the function value evaluated on the chosen point at each round, instead of getting the full function $f_t$. That prevents us from providing the right feedback for each of the Blackwell algorithms. Recall that in \Cref{alg:DR-SM}, the algorithm gives the feedback \[\pbf(\direction_t^{(i)},\ybf_t^{(i)})\leftarrow\begin{bmatrix}
-\bm{\direction_t^{(i)}}\cdot\nabla f_t(\zbf_t^{(i-1)})\\
\nabla f_t(\zbf_t^{(i-1)})
\end{bmatrix}\] to $\BlackwellAlg^{(i)}$ in each round.
Given this, to design an online learning algorithm for the bandit setting, we would like 
 to provide  an unbiased estimator of the gradient $\nabla f_t(\zbf_t^{(i-1)})$ in terms of the function value $f_t(\zbf_t^{(i-1)})$ every time the algorithm explores. To do so, following \cite{zhang2019online},
 we use an intermediate function $\hatf_{\delta,t}(\zbf) = \mathbb{E}_{\bm{v}\sim\textrm{Unif}\left( B_2^n\right)}[f(\zbf + \delta\bm{v})]$ for some small value of $\delta>0$ that we will determine later, where $B_2^n$ denotes the $\ell_2$ unit ball in $\mathbb{R}^n$. Note that  when $f_t$ is monotone, continuous strong-DR submodular, $L$-Lipschitz smooth with a bounded gradient of $U$ for any point $\zbf\in\mathcal{P}$, so is $\hatf_{\delta,t}$. We further have $|f_t(\zbf)-\hatf_{\delta,t}(\zbf)|\leq U\delta$ for all $\zbf\in\Pcal$. Moreover, let  $S_2^{(n-1)}$ denote the $\ell_2$ unit sphere in $\mathbb{R}^n$ (i.e., $S_2^{(n-1)} = \{\bm{v}\in\mathbb{R}^n: \lVert\bm{v}\rVert_2 = 1\}$). Then,  when $\bm{u}$ is uniformly drawn at random from $S_2^{(n-1)}$, $\dfrac{n}{\delta}f_t(\zbf+\delta\bm{u})\bm{u}$ is an unbiased estimator of $\nabla\hatf_{\delta,t}(\zbf)$, i.e., $\nabla\hatf_{\delta,t}(\zbf) = \mathbb{E}_{\bm{u}\sim \textrm{Unif}\left(S_2^{(n-1)}\right)}\left[\dfrac{n}{\delta}f_t(\zbf+\delta\bm{u})\bm{u}\right]$ (\cite{flaxman2005online}, \cite{zhang2019online}). We further use an intermediate convex downward closed feasible region $\Pcal' = (1-\alpha)\Pcal+\delta\mathbf{1}$, where  $\alpha = \frac{\sqrt{n}+1}{r}\delta$, and $r$ is the largest positive real number such that $r\cdot \left(B_2^n\cap \mathbb{R}^n_{\geq 0}\right)\subseteq\Pcal$. The new feasible region  $\Pcal'$ is such that for any $\zbf\in\Pcal'$ and $\bm{u}\in S^{(n-1)}$, the point $\zbf+\delta\bm{u}$ is in $\Pcal$, keeping the chosen point of $\zbf+\delta\bm{u}$ feasible in the original problem.

In the bandit algorithm, after showing bandit Blackwell reducability as in \Cref{thm:bandit-blackwell-thm}, we replace the local optimization step with a polynomial-time online algorithm $\BanditBlackwellAlg^{(i)}$ for each subproblem $i\in[N]$. We utilize $N$ parallel runs of $\BanditBlackwellAlg$, one for each subproblem as shown in Algorithm~\ref{alg:bandit-DR-SM}. Each $\BanditBlackwellAlg^{(i)}$ is an algorithm for a particular bandit Blackwell sequential game (see its general definition in \Cref{sec:bandit-blackwell}). In order to be able to obtain the feedback required in this bandit Blackwell sequential game, we develop our reduction as if we wanted to maximize  $\hatf_{\delta,t}$ instead of $f_t$ over feasible region $\mathcal P'$ (not $\mathcal P$). Since $\delta$ is very small, this does not impact the regret of the designed algorithm. Specifically, we consider a bandit Blackwell sequential game with $\algspaceB = \paramspace = \mathcal{P}'$, $\advspaceB = [-U_{\delta,n},U_{\delta,n}]^n$, where $U_{\delta,n}=\frac{n}{\delta}$ (i.e., an upper-bound on the absolute value of each coordinate of the vector $\nabla\hatf_{\delta}(\zbf^{(i-1)})$). We further define the vector payoff function $\pbf(\bm{x},\bm{y}) = \begin{bmatrix}
-\bm{x}\cdot\bm{y}\\
\bm{y}
\end{bmatrix}$ and  the target set $S= \textrm{Cone}(1\oplus\mathcal{P})^\circ$. In this reduction, we use a different $\advfunB$ (compared to the one used for the full-information setting). We define $\advfunB(\zbf^{(i-1)},f)$ to be $\nabla\hatf_{\delta}(\zbf^{(i-1)})$, and hence we have $\bm{y}_t^{(i)} = \advfunB(\zbf^{(i-1)}_t,f_t) = \nabla\hatf_{\delta, t}(\zbf^{(i-1)})$. Moreover, in each round, we have $\xbf_t^{(i)} = \direction_t^{(i)}$.

As stated earlier, at the core of our bandit algorithm, we need to provide the right feedback  when subproblem $i$ explores in round $t$. Such a feedback should be an unbiased estimator of 
$$
{\pbf}\left(\direction_t^{(i)},\advfunB(\zbf_t^{(i-1)}, f_t)\right)=\begin{bmatrix}
-\direction_t^{(i)}\cdot \nabla\hatf_{\delta,t}(\zbf_t^{(i-1)})\\
\nabla\hatf_{\delta,t}(\zbf_t^{(i-1)})
\end{bmatrix}\,.
$$
Here, $\direction^{(i)}_t$ is the output of $\BanditBlackwellAlg^{(i)}$.
Denoting the desired unbiased estimator for subproblem $i$ in round $t$ by $\hat\pbf_t^{(i)}$, we construct $\hat\pbf_t^{(i)}$ with the help of our exploration sampling device, denoted by $\unbiasedestimator$. Given the current point $\zbf_t^{(i-1)}\in \Pcal'$ received from the previous sub-problem and $\direction_t^{(i)}$ as the current decision of $\BanditBlackwellAlg$ for exploration, $\unbiasedestimator$ returns $(\wbf_{\textrm{exp}},\zbf_{\textrm{exp}})$ such that (i)~$\zbf_{\textrm{exp}}\in\Pcal$, i.e., it is feasible, and (ii)~$\hat\pbf_t^{(i)}=f_t(\zbf_{\textrm{exp}})\bm{w}_{\textrm{exp}}$ is a desired unbiased estimator, i.e.,  we have 
$$
\expect{\hat\pbf_t^{(i)}}=\expect{f_t(\zbf_{\textrm{exp}})\bm{w}_{\textrm{exp}}}={\pbf}\left(\direction_t^{(i)},\advfunB(\zbf_t^{(i-1)}, f_t)\right)\,.
$$
Note that the above unbiased estimator can be constructed by only evaluating the function $f_t$ at a feasible point $\zbf_{\textrm{exp}}$ in $\Pcal$. It only remains to show how to construct such an exploration sampling device.

In order to satisfy the above properties (i) and (ii), the exploration sampling device $\unbiasedestimator$ works as follows: It first uniformly at random picks a point $\bm{u}$ from the sphere $S_2^{(n-1)}$. Then it sets $$\wbf_{\textrm{exp}}= \dfrac{n}{\delta}\begin{bmatrix}
-{\direction_t^{(i)}}\cdot \bm{u}\\
\bm{u}
\end{bmatrix}~~~~,~~~~\zbf_{\textrm{exp}}=\zbf_t^{(i-1)}+\delta\bm{u}\,.$$
The point $\zbf_{\textrm{exp}}$ is a feasible point in $\Pcal$, because $\zbf_t^{(i-1)}\in\Pcal'$. Moreover, we have:
$$
\expect{f_t(\zbf_{\textrm{exp}})\bm{w}_{\textrm{exp}}}=\expect{\begin{bmatrix}
-\bm{\direction_t^{(i)}}\cdot\dfrac{n}{\delta}f_t(\zbf_t^{(i-1)}+\delta\bm{u})\bm{u}\\
\dfrac{n}{\delta}f_t(\zbf_t^{(i-1)}+\delta\bm{u})\bm{u}
\end{bmatrix}}= \begin{bmatrix}
-\direction_t^{(i)}\cdot \nabla\hatf_{\delta,t}(\zbf_t^{(i-1)})\\
\nabla\hatf_{\delta,t}(\zbf_t^{(i-1)})
\end{bmatrix}={\pbf}\left(\direction_t^{(i)},\advfunB(\zbf_t^{(i-1)}, f_t)\right)~,
$$
which shows $\hat\pbf_t^{(i)}=f_t(\zbf_{\textrm{exp}})\bm{w}_{\textrm{exp}}$ is our desired unbiased estimator. 
Given the above bandit Blackwell sequential game, since we have shown that $S$ is response-satisfiable, there exists a polynomial-time algorithm $\BanditBlackwellAlg$, based on \Cref{thm:bandit-blackwell-thm} and its proof in \Cref{sec:appendix-bandit}, such that for any $\ybf_1,\ldots,\ybf_T\in\advspaceB$ and exploration probability $q\in[0,1]$, $\BanditBlackwellAlg$ can generate a sequence of decisions $\xbf_1,\ldots,\xbf_T\in\algspaceB$ such that
\begin{equation}
\begin{aligned}
\label{eq:guarantee_online_BW_matroid}
d_\infty\left(\frac{1}{T}\sum_{t=1}^T\pbf(\xbf_t,\ybf_t),S\right) + \mathbb{E}\left[\frac{1}{T}D(\pbf)\cdot(\text{\# explore})\right]
&\leq\bigO{\frac{1}{qT}D(\hat{\pbf})(\log {d_{\pbf}})^{1/2}(Tq)^{1/2}} + \bigO{D(\pbf)q}~,
\end{aligned}
\end{equation}
as shown in the proof of \Cref{thm:bandit-blackwell-thm} in Section \ref{sec:appendix-bandit}. Here, the diameter of $\hat{\pbf}$, i.e., $D(\hat{\pbf})$, is $\dfrac{n}{\delta}\max\{1,\diampolytope\}$. This is because for any $\zbf\in\Pcal, \bm{u}\in S_2^{(n-1)}$, $\lVert\left(\dfrac{n}{\delta}f(\zbf)\bm{u}\right)\rVert_{+\infty}\leq\dfrac{n}{\delta}\cdot1\cdot1 = \dfrac{n}{\delta}$ and $\lvert-\direction\cdot\dfrac{n}{\delta}f(\zbf+\delta\bm{u})\bm{u}\rvert\leq\lVert-\direction\rVert_{2}\lVert\dfrac{n}{\delta}f(\zbf+\delta\bm{u})\bm{u}\rVert_{2}\leq \diampolytope\dfrac{n}{\delta}$, and hence $D(\hat{\pbf}) = \dfrac{n}{\delta}\max\{1,\diampolytope\}$. We would like to highlight that based on our definition of the actual vector payoff, here, we similarly have $D({\pbf}) = O\left(\diampolytope\dfrac{n}{\delta}\right)$.

We summarize the result for the bandit setting in the following theorem.



\begin{algorithm}
  \caption{Bandit online learning algorithm for Algorithm \ref{alg:DR-SM}}
  \label{alg:bandit-DR-SM}
  \textbf{Input:} A sequence of $L$-Lipschitz smooth strong-DR monotone submodular functions $f_1,\ldots,f_T:\mathcal{P}\rightarrow[0,1]$ such that $||\nabla f_t(\zbf)||_\infty\leq U$ for all $\zbf\in\mathcal{P}, t\in[T]$; parameter $\delta$; $N$ bandit Blackwell algorithms $\{\BanditBlackwellAlg^{(i)}\}_{i=1,\ldots,N}$, where each $\BanditBlackwellAlg^{(i)}$ is an online algorithm for a bandit Blackwell sequential game with $\algspaceB=\mathcal{P}'$ (where $\mathcal{P}' = (1-\alpha)\Pcal + \delta\mathbf{1}$, $\alpha = \frac{\sqrt{n}+1}{r}\delta$, and $r$ is the largest positive real number such that $r\left(B_2^n\cap \mathbb{R}^n_{\geq 0}\right)\subseteq\Pcal$), $\advspaceB=[-U_{\delta,n},U_{\delta,n}]^n$ with $U_{\delta,n}=\frac{n}{\delta}$), $\pbf(\xbf,\ybf)=\begin{bmatrix}
  -\bm{x}\cdot\bm{y}\\
  \bm{y}
  \end{bmatrix}$, and $S = \textrm{Cone}(1\oplus\mathcal{P})^\circ$. Whenever called, $\BanditBlackwellAlg^{(i)}$ explores with probability $q$.\\
\textbf{Output:}   $\zbf_1,\ldots,\zbf_T\in\mathcal{P}$\\
  Initialize $N$ parallel instances $\{\BanditBlackwellAlg^{(i)}\}_{i=1}^N$.\\
  \For{$t=1,\ldots,T$}{\vspace{0.2em}
  Initialize $\zbf^{(0)}_t \leftarrow \mathbf{0}$.\\
  \For{iteration $i = 1, 2, \ldots, N$}{\vspace{0.2em}
   Choose the direction $\direction^{(i)}_t\in\mathcal{P}'$ and the exploration signal $\pi_t^{(i)} \in \{ \textsc{Yes},\textsc{No}\}$ by querying $\BanditBlackwellAlg^{(i)}$ given the update parameters and vector payoffs $\hat{\pbf}_t^{(i)}$ of exploration rounds prior to round $t$  in the bandit Blackwell sequential game of subproblem $i$, i.e. $\left(\direction^{(i)}_t,\pi_t^{(i)}\right)\leftarrow\BanditBlackwellAlg^{(i)}\left(\direction^{(i)}_1,\ldots,\direction^{(i)}_{t-1}, \{\hat\pbf_\tau^{(i)}\}_{\tau\leq t-1:\pi_\tau^{(i)}=\textsc{Yes}}\right)$.\\
   \If{$\pi^{(i)}_t = \textsc{Yes}$,} {
    Play the exploration point $\zbf_t \leftarrow \zbf^{(i)}_{t}+\delta\bm{u}$ where $\bm{u}$ is drawn uniformly at random from the unit $(n-1)$ dimensional sphere $S^{(n-1)}$.\\
    $\quad\quad$\emph{//Bandit information feedback: observe  $f_t(\zbf_t+\delta\bm{u})$.}\\
    Give payoff vector feedback $\hat{\pbf}_t^{(i)} = \begin{bmatrix}
    -\direction_t^{(i)}\cdot \frac{n}{\delta}f_t(\zbf_t+\delta\bm{u})\bm{u}\\
    \frac{n}{\delta}f_t(\zbf_t+\delta\bm{u})\bm{u}
    \end{bmatrix}$ to $\BanditBlackwellAlg^{(i)}$. Skip immediately to the beginning of the next round $t+1$.\\
    $\quad\quad$\emph{//Note that $\expect{\hat\pbf^{(i)}_t}=\pbf(\xbf_t^{(i)},\ybf_t^{(i)})$, where $\xbf_t^{(i)}=\direction_t^{(i)}$ and $\ybf_t^{(i)} = \advfunB(\zbf_t^{(i-1)},f_t) =  {\nabla}\hatf_{\delta,t}(\zbf_t^{(i)})$. Here,  $\hatf_{\delta,t}(\zbf_t^{(i)}) = \mathbb{E}_{\bm{v}\sim \textrm{Unif}( B_2^n)}[f_t(\zbf_t^{(i)}+\delta\bm{v})]$.}
    
    }
    Set $\zi{i}_t\leftarrow\zi{i-1}_t + \frac{1}{N}\direction_t^{(i)}.$\\
  }
  Play the final point $\zbf_t\leftarrow\zbf_t^{(N)}$, receive the bandit feedback $f_t(\zbf_t)$ and ignore them.
  }
\end{algorithm}


\begin{theorem}\normalfont{\textbf{(Bandit learning for strong-DR monotone SM maximization in a downward closed convex set)}}
\label{thm:bandit-DR-SM}
Let $\Pcal$ be a downward closed convex polytope with $\ell_2$ diameter $\diampolytope$. Consider the online learning problem of maximizing strong-DR monotone submodular functions on $\mathcal{P}$, which are $L$-Lipschitz smooth and have bounded gradient of $U$ for any point $\zbf\in\mathcal{P}$. Then, under bandit setting, with the total number of iterations $N =\bigO{T^{1/6}\diampolytope L^{1/2}\left(\log(n)\right)^{-1/6}}$, the explore probability $q = \bigO{T^{-1/3}\left(\log(n)\right)^{1/3}}$, and  $\delta=\bigO{T^{-1/6}\diampolytope^{1/2}n^{1/2}\left(\log(n)^{1/6}\right)U^{-1/2}}$, \Cref{alg:bandit-DR-SM} with $N$ iterations obtains $\bigO{T^{5/6}\left(\log(n)\right)^{1/6}\diampolytope\left(n\diampolytope^{1/2}U^{1/2}+L^{1/2}\right)}$ $\gamma_N$-regret, where $\gamma_N = 1-(1-1/N)^N$.
\end{theorem}


Notice that our bandit regret is $\bigO{T^{5/6}}$ in terms of $T$, which improves the previous regret bound in \cite{zhang2019one}, $\bigO{T^{8/9}}$, in terms of $T$. We now present the proof of the theorem.
\proof{\emph{Proof of Theorem \ref{thm:bandit-DR-SM}}.} 
Define  $\xbf^* = \argmax_{\xbf\in\Pcal}\sum_{t=1}^Tf_t(\xbf)$
and $\xbf_{\delta}^* = \argmax_{\xbf\in\Pcal'}\sum_{t=1}^T\hatf_{\delta,t}(\xbf)$. Since we are running $\BanditBlackwellAlg$ as if it were for the function $\hatf_{\delta,t}$ over the feasible region $\mathcal P'$,
our total $\gamma_N-$regret can be written as
\begin{equation}
\label{eq:final_total_regret_DR_SM}
\begin{aligned}
\mathbb{E}\left[ \sum_{t=1}^T\gamma_Nf_t(\xbf^*)-f_t(\zbf_t)\right] &=\mathbb{E}\left[ \sum_{t=1}^T\gamma_N(f_t(\xbf^*)-f_t(\xbf_\delta^*)+f_t(\xbf_\delta^*)) - \sum_{t=1}^Tf_t(\zbf_t)\right] \\
&= \mathbb{E}\Bigg[\gamma_N\sum_{t=1}^T\left(f_t(\xbf^*) - f_t(\xbf_{\delta}^*)\right) + \gamma_N\sum_{t=1}^T\left(f_t(\xbf_{\delta}^*)-\hatf_{\delta,t}(\xbf_\delta^*)\right) + \sum_{t=1}^T\left(\gamma_N\hatf_{\delta,t}(\xbf_{\delta}^*) - \hatf_{\delta,t}(\zbf_t)\right)\\ &+\sum_{t=1}^T\left(\hatf_{\delta,t}(\zbf_t)-f_t(\zbf_t)\right)\Bigg]\\
&\overset{(1)}{\leq}\mathbb{E}\left[\gamma_N\sum_{t=1}^T\left(f_t(\xbf^*)-f_t(\xbf_{\delta}^*)\right) + \gamma_NTU\delta + \sum_{t=1}^T\left(\gamma_N\hatf_{\delta,t}(\xbf_{\delta}^*)-\hatf_{\delta,t}(\zbf_t)\right) + TU\delta\right],
\end{aligned}
\end{equation}
where Inequality $(1)$ holds because we have $|f_t(\zbf)-\hatf_{\delta,t}(\zbf)|\leq U\delta$ for any $\zbf\in\Pcal$. We are now left with bounding the first and third terms in \Cref{eq:final_total_regret_DR_SM}.

To bound the first term in Equation \eqref{eq:final_total_regret_DR_SM}, let $\xbf'\in\Pcal'$ be the point in $\Pcal'$ that is closest to $\xbf^*$, i.e., $||\xbf'-\xbf^*||_{2} = d(\xbf',\xbf^*) = d(\xbf^*,\Pcal')\leq d(\Pcal,\Pcal')$, where $d(\xbf,\Pcal') = \min_{\bm{v}\in\Pcal'}||\xbf-\bm{v}||_2$ and
$d(\Pcal,\Pcal') = \max_{\bm{v}\in\Pcal}d(\xbf,\Pcal').$
We then have 
\begin{equation}
\label{eq:third_DW}
\begin{aligned}
\sum_{t=1}^T\left(f_t(\xbf^*)-f_t(\xbf_{\delta}^*)\right) &= \sum_{t=1}^T\left(f_t(\xbf^*)-f_t(\xbf')\right) + \sum_{t=1}^T\left(f_t(\xbf')-f_t(\xbf^*_{\delta})\right)\\
&\overset{(1)}{\leq} U\sum_{t=1}^T\lVert\xbf^*-\xbf'\rVert_{2} + 0\\
&\leq UTd(\Pcal,\Pcal')\\
&\overset{(2)}{\leq} UT\left(\sqrt{n}\left(\frac{\diampolytope}{r}+1\right)+\frac{\diampolytope}{r}\right)\delta,
\end{aligned}
\end{equation}
where Inequality $(1)$ holds because $f_t$ is $U$-Lipschitz continuous and $\xbf_\delta^*$ maximizes $f_t$ in $\Pcal'$, while Inequality $(2)$ holds because
\[
d(\Pcal,\Pcal')\leq\left(\sqrt{n}\left(\frac{\diampolytope}{r}+1\right)+\frac{\diampolytope}{r}\right)\delta
\]
from Lemma 1 in \cite{zhang2019one}.

To bound the third term of Equation \eqref{eq:final_total_regret_DR_SM}, i.e., $\sum_{t=1}^T\left(\gamma_N\hatf_{\delta,t}(\xbf_{\delta}^*)-\hatf_{\delta,t}(\zbf_t)\right)$, we follow the same lines as in the proof for \Cref{thm:banditILO}, by adapting it as if the underlying functions are $\hatf_{\delta,t}$. We then employ the result in \Cref{eq:extend-Rank-1} for a new value of $h(T)$ to show that \Cref{alg:bandit-DR-SM} works. Let $\mathcal{T}_i\subseteq[T]$ be the set of rounds where $\BanditBlackwellAlg^{(i)}$ is invoked. Note that $\mathcal{T}_i$ is a random set since it depends on the realization of binary signals $\{\pi_t^{(i')}\}_{i'\in[i-1],t\in[T]}$. Consider a subproblem $i \in [N]$. Fix a particular realization of set $\mathcal{T}_i$. By using the upper-bound in \Cref{eq:guarantee_online_BW_matroid} and the fact that $D(\hat{\pbf}) = \bigO{\dfrac{n}{\delta}\diampolytope}$ and $D({\pbf}) = \bigO{\dfrac{n}{\delta}\diampolytope}$, we have 
  \begin{align*}
    \standarddistance{\frac{1}{|\mathcal{T}_i|} \sum_{t \in \mathcal{T}_i} \pbf\left(
      \direction_t^{(i)}, \advfunB(\zbf_t^{(i-1)},f_t)
    \right)}{S} &\le
    \bigO{\frac{1}{q|\mathcal{T}_i|}D(\hat{\pbf})\log{(d_{\mathbf p})^{1/2}(|\mathcal{T}_i|q)^{1/2}}} + \bigO{D(\pbf)q}\\
    &= \bigO{\frac{1}{q|\mathcal{T}_i|}\dfrac{n}{\delta}\diampolytope(\log n)^{1/2}(|\mathcal{T}_i|q)^{1/2}} + \bigO{\frac{n}{\delta}\diampolytope q}\,.
  \end{align*}
Since $\pbf\left(\direction_t^{(i)},\advfunB(\zbf_t^{(i-1)}, f_t)\right) = \pay\left(\direction_t^{(i)},\zbf_t^{(i-1)}, f_t\right)$ by definition, we have
\begin{align}
\label{eq:bandit_matroid}
\bigO{\frac{1}{q|\mathcal{T}_i|}\dfrac{n}{\delta}\diampolytope(\log{n})^{1/2}(|\mathcal{T}_i|q)^{1/2}} + \bigO{\frac{n}{\delta}q} &\geq \standarddistance{\frac{1}{|\mathcal{T}_i|} \sum_{t \in \mathcal{T}_i} \pbf\left(\direction_t^{(i)}, \advfunB(\zbf_t^{(i-1)},f_t)\right)}{S}\nonumber\\
&= \standarddistance{\frac{1}{|\mathcal{T}_i|} \sum_{t \in \mathcal{T}_i} \pay\left(\direction_t^{(i)}, \zbf_t^{(i-1)},f_t\right)}{S}\nonumber\\
&\overset{}{\geq}\frac{1}{\sqrt{n+1}\sqrt{1+{\diampolytope}^2}}\frac{1}{|\mathcal{T}_i|}~{\left(\underset{\direction\in\Pcal'}{\max}\sum_{t\in\mathcal{T}_i}\ybf_t^{(i)}\cdot\direction - \sum_{t\in\mathcal{T}_i}\ybf_t^{(i)}\cdot\direction_t^{(i)}\right)},
\end{align} 
where the last line follows from the equivalent of \Cref{eq:first_line_matroid} and \Cref{eq:last_line_matroid}, which is still true because by shrinking the domain for $\{\zbf_t^{(i)}\}_{1\leq i\leq N, 1\leq t\leq T}$ to be $\Pcal'$, all of the chosen points are guaranteed to be in $\Pcal$, so $\Pcal'\subseteq\Pcal$ and the diameter of $\Pcal'$ is upper-bounded by the diameter of $\Pcal$, $\diampolytope$ as $\Pcal'\subseteq\Pcal$.

Let $\mathcal{T}^-$ be the set of rounds where no $\BanditBlackwellAlg{}^{(i)}$  explored and $\mathcal{T}^+$ be the set of rounds where some $\BanditBlackwellAlg{}^{(i)}$ explored. See that $\mathcal{T}^- \subseteq \mathcal{T}_i$ because if no algorithm explores, $\BanditBlackwellAlg^{(i)}$ will be invoked. For a set of rounds $\mathcal{T}$, let $\direction^*(\mathcal{T}) = \underset{\direction\in\Pcal'}{\argmax}\sum_{t\in \mathcal{T}}\direction\cdot\nabla \hatf_{\delta,t}(\zbf_t^{(i-1)})$. Also, let $M_i$ be the number of rounds that $\BanditBlackwellAlg^{(i)}$ explores out of the $|\mathcal{T}_i|$ rounds it is invoked. For subproblem $i$, we have
\begin{align*}
&\mathbb{E}\left[\left(\underset{\direction\in\Pcal'}{\max}\sum_{t\in\mathcal{T}^{-}}\ybf_t^{(i)}\cdot\direction - \sum_{t\in\mathcal{T}^{-}}\ybf_t^{(i)}\cdot\direction_t^{(i)}\right)\bigg|\mathcal{T}_i\right]\\
&= \mathbb{E}\left[\left(\sum_{t\in\mathcal{T}^{-}}\ybf_t^{(i)}\cdot\direction^*(\mathcal{T}^-) - \sum_{t\in\mathcal{T}^{-}}\ybf_t^{(i)}\cdot\direction_t^{(i)}\right)\bigg|\mathcal{T}_i\right]\\
&= \mathbb{E}\left[\left(\sum_{t\in\mathcal{T}_i}\ybf_t^{(i)}\cdot\direction^*(\mathcal{T}^-) - \sum_{t\in\mathcal{T}_i}\ybf_t^{(i)}\cdot\direction_t^{(i)}\right)\bigg|\mathcal{T}_i\right] - \mathbb{E}\left[\left(\sum_{t\in\mathcal{T}_i\setminus\mathcal{T}^{-}}\ybf_t^{(i)}\cdot\direction^*(\mathcal{T}^-) - \sum_{t\in\mathcal{T}_i\setminus\mathcal{T}^{-}}\ybf_t^{(i)}\cdot\direction_t^{(i)}\right)\bigg|\mathcal{T}_i\right]\\
&\leq \mathbb{E}\left[\left(\sum_{t\in\mathcal{T}_i}\ybf_t^{(i)}\cdot\direction^*(\mathcal{T}_i) - \sum_{t\in\mathcal{T}_i}\ybf_t^{(i)}\cdot\direction_t^{(i)}\right)\bigg|\mathcal{T}_i\right] - \mathbb{E}\left[\left(\sum_{t\in\mathcal{T}_i\setminus\mathcal{T}^{-}}\ybf_t^{(i)}\cdot\direction^*(\mathcal{T}^-) - \sum_{t\in\mathcal{T}_i\setminus\mathcal{T}^{-}}\ybf_t^{(i)}\cdot\direction_t^{(i)}\right)\bigg|\mathcal{T}_i\right]\\
&\overset{(1)}{\leq} \bigO{\frac{{\diampolytope}\sqrt{n}}{q}\dfrac{n}{\delta}\diampolytope(\log n)^{1/2}(|\mathcal{T}_i|q)^{1/2}} + \bigO{{\diampolytope}\sqrt{n}|\mathcal{T}_i|\frac{n}{\delta}{\diampolytope}q} + D(\pbf)\mathbb{E}[M_i|\mathcal{T}_i]\\
&\overset{(2)}{\leq} \bigO{\frac{{\diampolytope}^2n^{3/2}}{q^{1/2}\delta}(\log n)^{1/2}|\mathcal{T}_i|^{1/2}} + \bigO{\frac{{\diampolytope}^2n^{3/2}}{\delta}|\mathcal{T}_i|q} + \bigO{\frac{n}{\delta}\diampolytope q\lvert\mathcal{T}_i\rvert}\\
&\overset{(3)}{\leq} \bigO{\frac{{\diampolytope}^2n^{3/2}}{q^{1/2}\delta}(\log n)^{1/2}T^{1/2}} +  \bigO{\frac{{\diampolytope}^2n^{3/2}}{\delta}{T}q} + \bigO{\frac{n}{\delta}\diampolytope qT}\,,
\end{align*}
where all of the expectations are with respect to $\zbf_t^{(i-1)}, t\in\mathcal{T}_i$. Here, Inequality~(1) follows from \Cref{eq:bandit_matroid} and the fact that $|\mathcal{T}_i\setminus\mathcal{T}^-|\leq M_i$. Moreover, for any $t\in[T], \direction_1,\direction_2\in\Pcal\subseteq[0,1]^n$, $\ybf_t^{(i)}\cdot(\direction_1-\direction_2)\leq D_\infty(\ybf_t^{(i)})\leq U$ since $\Pcal$ is closed-down. Inequality (2) holds because $\E [M_i \big| \mathcal{T}_i] = q|\mathcal{T}_i|$. Inequality (3) holds because $|\mathcal{T}_i|\leq T$. Fixing a function $h(\cdot)$, see that given a fixed $\mathcal{T}_i$,
\begin{equation*}
\mathbb{E}\left[\left(\underset{\direction\in\Pcal'}{\max}\sum_{t\in\mathcal{T}^{-}}\ybf_t^{(i)}\cdot\direction - \sum_{t\in\mathcal{T}^{-}}\ybf_t^{(i)}\cdot\direction_t^{(i)}\right)\bigg|\mathcal{T}_i\right]\leq h(|\mathcal{T}^-|)
\end{equation*}
is equivalent to
\begin{equation*}
\begin{aligned}
\mathbb{E}\left[\left(\underset{\direction\in\Pcal'}{\max}\sum_{t\in\mathcal{T}^{-}}\nabla \hatf_{\delta,t}(\zbf_t^{(i-1)})\cdot\direction - \sum_{t\in\mathcal{T}^{-}}\nabla \hatf_{\delta,t}(\zbf_t^{(i-1)})\cdot\direction_t^{(i)}\right)\right] \leq h(|\mathcal{T}^-|)\\
\Leftrightarrow\mathbb{E}\left[\sum_{t\in\mathcal{T}^{-}}\nabla \hatf_{\delta,t}(\zbf_t^{(i-1)})\cdot\direction_t^{(i)}\right]&\geq\mathbb{E}\left[\underset{\direction\in\Pcal'}{\max}\sum_{t\in\mathcal{T}^{-}}\nabla \hatf_{\delta,t}(\zbf_t^{(i-1)})\cdot\direction\right] - h(|\mathcal{T}^-|).
\end{aligned}
\end{equation*}
We can then directly invoke the result in \Cref{eq:extend-Rank-1} for a fixed $\mathcal{T}^-$, with $h(|\mathcal{T}^-|)$ being replaced with its upper-bound above, i.e., $\bigO{\frac{{\diampolytope}^2n^{3/2}}{q^{1/2}\delta}(\log n)^{1/2}T^{1/2}+\frac{{\diampolytope}^2n^{3/2}}{\delta}{T}q + \frac{n}{\delta}\diampolytope qT}$. Then for any $\zbf_{\delta}^*\in\Pcal'$, we have
\begin{equation}
\label{eq:first_eq_DW} 
\begin{aligned}
  \mathbb{E}\left[\sum_{t\in\mathcal{T}^-}\hatf_{\delta,t}(\zbf_t^{(N)})\right]&\geq \gamma_N
  \mathbb{E}\left[\sum_{t\in\mathcal{T}^-}\hatf_{\delta,t}(\zbf_{\delta}^*)\right]- \bigO{\frac{{\diampolytope}^2n^{3/2}}{q^{1/2}\delta}(\log n)^{1/2}T^{1/2}} - \bigO{\frac{{\diampolytope}^2n^{3/2}}{\delta}{T}q} \\&- \bigO{\frac{n}{\delta}\diampolytope qT} - \frac{|\mathcal{T}^-|L{\diampolytope}^2}{2N}\\
  &\geq \gamma_N
  \mathbb{E}\left[\sum_{t\in\mathcal{T}^-}\hatf_{\delta,t}(\zbf_{\delta}^*)\right]- \bigO{\frac{{\diampolytope}^2n^{3/2}}{q^{1/2}\delta}(\log n)^{1/2}T^{1/2}} - \bigO{\frac{{\diampolytope}^2n^{3/2}}{\delta}{T}q} \\&- \bigO{\frac{n}{\delta}\diampolytope qT} - \bigO{\frac{TL{\diampolytope}^2}{N}},
\end{aligned}
\end{equation}

We further show that \Cref{alg:bandit-DR-SM} does not explore too often among its subproblems, specifically
\begin{align*}
\E \left[ |\mathcal{T}^+| \right] = \E \left[ \sum_{i=1}^N M_i \right]
    \le \sum_{i=1}^N \bigO{q\E\left[\mathcal{T}_i\right]}\leq\sum_{i=1}^N \bigO{qT},
\end{align*}
and since  functions $\hatf_{\delta,t}$ output values that are in $[0,1]$, for the remaining rounds $\mathcal{T}^+$ we have:
\begin{align}
\label{eq:second_eq_DW} 
\E\left[\sum_{t \in \mathcal{T}^+} {\hatf_{\delta,t}(\zbf_t)}\right] &\geq \gamma_N \cdot \E\left[\sum_{t \in \mathcal{T}^+} \hatf_{\delta,t}(\zbf_{\delta}^*)\right]- \bigO{NqT}.
\end{align}

Summing \Cref{eq:first_eq_DW} and \Cref{eq:second_eq_DW}, we have
\begin{equation}
\label{eq:final_fhat_regret}
\begin{aligned}
\mathbb{E}\left[\gamma_N\sum_{t\in[T]}\hatf_{\delta,t}(\zbf_{\delta}^*)-\sum_{t\in[T]}\hatf_{\delta,t}(\zbf_t)\right] 
&\leq  \bigO{\frac{{\diampolytope}^2n^{3/2}}{q^{1/2}\delta}(\log n)^{1/2}T^{1/2}} + \bigO{\frac{{\diampolytope}^2n^{3/2}}{\delta}{T}q} \\&+ \bigO{\frac{n}{\delta}\diampolytope qT} + \bigO{\frac{TL{\diampolytope}^2}{N}}+ \bigO{NqT},
\end{aligned}
\end{equation}
which we can use to bound the third term in \Cref{eq:final_total_regret_DR_SM}. Revisiting the total regret, and substituting \Cref{eq:third_DW} and \Cref{eq:final_fhat_regret} to our total $\gamma_N$-regret in \Cref{eq:final_total_regret_DR_SM}, we then have
\begin{equation*}
\begin{aligned}
\mathbb{E}\left[ \sum_{t=1}^T\gamma_Nf_t(\xbf^*)-f_t(\zbf_t)\right]~\leq &O\Bigg(TU\delta\sqrt{n}{\diampolytope}  + \frac{{\diampolytope}^2n^{3/2}}{q^{1/2}\delta}(\log n)^{1/2}T^{1/2} + \frac{{\diampolytope}^2n^{3/2}}{\delta}{T}q  \\&+ \frac{TL{\diampolytope}^2}{N}+ NqT \Bigg).
\end{aligned}
\end{equation*}

We finally tune parameters $q$, $\delta$, and $N$ for the best regret bound. Specifically, by substituting $q=\bigO{T^{-1/3}\left(\log(n)\right)^{1/3}}$, $\delta=\bigO{T^{-1/6}\diampolytope^{1/2}n^{1/2}\left(\log(n)^{1/6}\right)U^{-1/2}}$, and $N=\bigO{T^{1/6}\diampolytope L^{1/2}\left(\log(n)\right)^{-1/6}}$, we get the total regret of
$$
\bigO{T^{5/6}\left(\log(n)\right)^{1/6}\diampolytope\left(n\diampolytope^{1/2}U^{1/2}+L^{1/2}\right)}~.
$$
This completes the proof.
\[\myqed\]
\endproof{}

\begin{remark}
\label{rem:bandit-DR-SM}
As we showed in \Cref{thm:bandit-DR-SM}, we can obtain sub-linear approximate regret with an optimal approximation factor for the general strong-DR Lipschitz continuous smooth submodular functions under the bandit feedback. As multi-linear extension of a set submodular function satisfies these conditions, it seems possible to extend our result in \Cref{sec:full_strong_DR} to this setting. However, there is a technical caveat and that is generating an unbiased estimator for $f_t^{\textrm{MLE}}(\mathbf{x})$ given query access at \emph{feasible} sets to the set function $f_t$, which is needed to obtain the required bandit feedback for \Cref{alg:bandit-DR-SM} when it is run on the sequence of functions $f_1^{\textrm{MLE}},\ldots,f_T^{\textrm{MLE}}$. Given a point $\mathbf{x}\in \Pcal$, where $\Pcal$ is the matroid polytope, the pipage rounding algorithm will sample a randomized set $S$ such that $\expect{f_t(S)} \geq  f_t^{\textrm{MLE}}(\mathbf{x})$. To have the equality, and hence an unbiased estimator for $f_t^{\textrm{MLE}}(\mathbf{x})$, one possible way is placing each element independently with probability $x_i$ in random set $S$; but then this set is not necessarily  an independent set and hence we are not allowed to return it as the algorithm in our problem.  We pose whether our approach can be extended to the bandit setting as an open problem.
\end{remark}}

%% file: tex/apx-ranking.tex
\section{Proofs and Remarks of Section~\ref{subsec:ranking} -- Product Ranking and Sequential Submodular Maximization}

\label{apx:ranking}

In this appendix, we give the missing proofs of the results from \Cref{subsec:ranking}. 

\subsection{Proof of \texorpdfstring{\Cref{thm:seqsub}}{}}
\proof{\emph{Proof}.}
We show that our meta \Cref{alg:bandit-meta} works by verifying the following conditions.

\paragraph{(i) Algorithm~\ref{alg:rank-off} is an extended $(\tfrac12, \tfrac12)$-robust approximation algorithm.}
We need to show that if the following equation holds for some function $h$:
\begin{equation*}
    \forall j \in [n], \quad \left[\sum_{t=1}^T\pay(\tilde\param^{(i)}_t, \pibf_t^{(i-1)}, f_t)\right]_j\geq -h(T),
\end{equation*}
then we must have
\begin{equation}
\label{eqn:extend-rank}
    \forall\pibf^*\in\Pi,\quad \sum_{t=1}^T\E\left[f_t(\pibf_t)\right]\geq\frac{1}{2}\sum_{t=1}^Tf_t(\pibf^*)-\frac{1}{2}nh(T).
\end{equation}
Recall that for any $j\in[n],$ we have $\left[\pay(\param^{(i)},\pibf^{(i-1)},f)\right]_j=\param^T\ybf^{(i)}-[\ybf^{(i)}]_j,$ where $\ybf^{(i)}\triangleq\big[f(\pibf^{(i-1)}+j\mathbf{e}_i)$\\$-f(\pibf^{(i-1)})\big]_{j\in[n]}$. First, we prove several inequalities that will later be used to prove inequality~\eqref{eqn:extend-rank}. Since each function $\seqsubf{{t,i}}$ is monotone submodular, we have:
\begin{align*}
\lambda_i\sum_{j=1}^i\bigg(\seqsubf{{t,i}}&(\{[\pibf_t]_1,\ldots,[\pibf_t]_j\}) - \seqsubf{{t,i}}\left(\{[\pibf_t]_1,\ldots,[\pibf_t]_{j-1}\}\right)\bigg)\\ 
&+\lambda_i\sum_{j=1}^i\bigg(\seqsubf{{t,i}}\left(\{[\pibf_t]_1,\ldots,[\pibf_t]_{j-1},[\pibf^*]_{j}\}\right) - \seqsubf{{t,i}}\left(\{[\pibf_t]_1,\ldots,[\pibf_t]_{j-1}\}\right)\bigg)\\
&\overset{(1)}{\geq} \lambda_i\sum_{j=1}^i\bigg(\seqsubf{{t,i}}\left(\{[\pibf_t]_1,\ldots,[\pibf_t]_j,[\pibf^*]_1,\ldots,[\pibf^*]_j\}\right) - \seqsubf{{t,i}}\left(\{[\pibf_t]_1,\ldots,[\pibf_t]_{j-1},[\pibf^*]_1,\ldots,[\pibf^*]_{j-1}\}\right)\bigg)
\\
&=\lambda_i\seqsubf{{t,i}}\left(\{[\pibf_t]_1,\ldots,[\pibf_t]_{i},[\pibf^*]_{1},\ldots,[\pibf^*]_{i}\}\right)-\lambda_i\seqsubf{{t,i}}\left(\emptyset\right)\\
&\overset{(2)}{\geq} \lambda_i\seqsubf{{t,i}}\left(\{[\pibf^*]_{1},\ldots,[\pibf^*]_{i}\}\right),
\end{align*}
where inequality~$(1)$ follows from submodularity and inequality~$(2)$ follows from monotonicity and non-negativity of $\seqsubf{{t,i}}$. To be more clear, inequality~$(1)$ holds because for each $j = 1,2,\ldots,i,$ the sum of the  marginal values of adding $[\pibf_t]_j$ to $\pibf^{(j-1)}$ and adding $[\pibf^*]_j$ to $\pibf^{(j-1)}$ is greater than equal to the marginal value of adding $\{[\pibf_t]_j, [\pibf^*]_j\}$ to $\{[\pibf_t]_1,\ldots,[\pibf_t]_{i},[\pibf^*]_{1},\ldots,[\pibf^*]_{i}\}$ as $f_{t,i}$ is submodular.
Recall that 
\begin{equation*}
    f_t(\pibf) \triangleq \lambda_1\seqsubf{{t,1}}\left(\{[\pibf]_1\}\right)+\lambda_2\seqsubf{{t,2}}\left(\{[\pibf]_1,[\pibf]_2\}\right)+\ldots+\lambda_n\seqsubf{{t,n}}\left(\{[\pibf]_1,\ldots,[\pibf]_n\}\right)\qquad\forall t\in[T],
\end{equation*}
so summing the inequalities above for $i=1,2,\ldots,n,$ we get:
\begin{align}
&\sum_{i=1}^n\lambda_i\sum_{j=1}^i\bigg(\seqsubf{{t,i}}(\{[\pibf_t]_1,\ldots,[\pibf_t]_j\}) - \seqsubf{{t,i}}\left(\{[\pibf_t]_1,\ldots,[\pibf_t]_{j-1}\}\right)\bigg)\nonumber\\ 
&+\sum_{i=1}^n\lambda_i\sum_{j=1}^i\bigg(\seqsubf{{t,i}}\left(\{[\pibf_t]_1,\ldots,[\pibf_t]_{j-1},[\pibf^*]_{j}\}\right) - \seqsubf{{t,i}}\left(\{[\pibf_t]_1,\ldots,[\pibf_t]_{j-1}\}\right)\bigg)\qquad\geq \sum_{i=1}^n \lambda_i\seqsubf{{t,i}}\left(\{[\pibf^*]_{1},\ldots,[\pibf^*]_{i}\}\right)\nonumber\\
&{\Leftrightarrow}\sum_{j=1}^n\sum_{i=j}^n\left(\lambda_i\seqsubf{{t,i}}\left(\{[\pibf_t]_1,\ldots,[\pibf_t]_j\}\right) - \lambda_i\seqsubf{{t,i}}\left(\{[\pibf_t]_1,\ldots,[\pibf_t]_{j-1}\}\right)\right)\nonumber\\
&+ \sum_{j=1}^n\sum_{i=j}^n\left(\lambda_i\seqsubf{{t,i}}\left(\{[\pibf_t]_1,\ldots,[\pibf_t]_{j-1},[\pibf^*]_j\}\right) - \lambda_i\seqsubf{{t,i}}\left(\{[\pibf_t]_1,\ldots,[\pibf_t]_{j-1}\}\right)\right) \qquad\geq f_t(\pibf^*)\nonumber\\
&\label{eq:final-ineq}\Leftrightarrow\sum_{j=1}^n\left(f_t(\pibf_t^{(j)}) - f_t(\pibf_t^{(j-1)})\right) + \sum_{j=1}^n \left(f_t(\pibf_t^{(j-1)}+[\pibf^*]_{j}\mathbf{e}_j) - f_t(\pibf_t^{(j-1)})\right) \qquad\geq f_t(\pibf^*).
\end{align}
We get the first equivalence by switching the summations. We now use the inequality~\eqref{eq:final-ineq} to prove the final claim, i.e., the desired inequality~\eqref{eqn:extend-rank}. We have:

\begin{align*}
    \sum_{t=1}^T\E\left[f_t(\pibf_t)\right]
    &= \dfrac12\sum_{t=1}^T\sum_{i=1}^n2~\E\left[f_t(\pibf_t^{(i)})-f_t(\pibf_t^{(i-1)})\right]\\
    &= \dfrac12\sum_{t=1}^T\sum_{i=1}^n\E\left[f_t(\pibf_t^{(i)})-f_t(\pibf_t^{(i-1)})\right]+\dfrac12\sum_{i=1}^n\sum_{t=1}^T\E\left[f_t(\pibf_t^{(i)})-f_t(\pibf_t^{(i-1)})\right]\\
    &\qquad-\dfrac12\sum_{i=1}^n\sum_{t=1}^T\E\left[f_t(\pibf_t^{(i)}+[\pibf^*]_{i}\mathbf{e}_i)-f_t(\pibf_t^{(i-1)})\right]+\dfrac12\sum_{i=1}^n\sum_{t=1}^T\E\left[f_t(\pibf_t^{(i)}+[\pibf^*]_{i}\mathbf{e}_i)-f_t(\pibf_t^{(i-1)})\right]\\
    &\overset{(1)}{=} \dfrac12\sum_{t=1}^T\sum_{i=1}^n\E\left[f_t(\pibf_t^{(i)})-f_t(\pibf_t^{(i-1)})\right] + \dfrac12\sum_{i=1}^n\sum_{t=1}^T\left[\pay(\tilde{\param}_t^{(i)}, \pibf_t^{(i-1)},f_t)\right]_{[\pibf^*]_{i}}\\ &\qquad+\dfrac12\sum_{i=1}^n\sum_{t=1}^T\E\left[f_t(\pibf_t^{(i-1)}+[\pibf^*]_{i}\mathbf{e}_i)-f_t(\pibf_t^{(i-1)})\right]\\
    &\overset{(2)}{\geq} \dfrac12\sum_{t=1}^T\sum_{i=1}^n\E\left[f_t(\pibf_t^{(i)})-f_t(\pibf_t^{(i-1)})\right] - \dfrac12 nh(T)+\dfrac12\sum_{t=1}^T\sum_{i=1}^n\E\left[f_t(\pibf_t^{(i)}+[\pibf^*]_{i}\mathbf{e}_i)-f_t(\pibf_t^{(i-1)})\right]\\
    &\overset{(3)}{\geq} \dfrac12\sum_{t=1}^Tf_t(\pibf^*) - \dfrac12 nh(T)~~.
\end{align*}
In the above chain of inequalities,  equality~$(1)$ follows from the definition of $\pay$, inequality~$(2)$ follows from our assumption, and inequality~$(3)$ follows from inequality~\eqref{eq:final-ineq}. Rearranging the terms will finish the proof.

\paragraph{(ii) Algorithm~\ref{alg:rank-off} is bandit Blackwell reducible.}
We verify the following conditions based on \Cref{def:bandit-blackwell-reducible} to show bandit Blackwell reducibility:
\begin{itemize}
    \item \textit{Algorithm \ref{alg:rank-off} is Blackwell reducible.} For each subproblem, consider an instance $(\algspaceB,\advspaceB,\pbf)$ of Blackwell where $\algspaceB = \Theta = \Delta([n])$ and $\advspaceB  = [-1,1]^n.$ Our Blackwell adversary function is the marginal increase in the objective function of placing item on position $i,$ $\advfunB(\pibf^{(i-1)},f) = \left[f(\pibf_t^{(i-1)}+j\mathbf{e}_i) - f(\pibf_t^{(i-1)})\right]_{j\in[n]}$. The biaffine payoff is $\pbf(\param,\ybf) = \param^T\ybf\mathbf{1}_n - \ybf$, where $\mathbf{1}_n$ is an $n$-dimensional all ones vector. The target set $S$ is the non-negative orthant, and it is response-satisfiable since for every adversary's action $\ybf\in\advspaceB,$ the strategy $\param = \mathbf{e}_{j^*}$ where $j^* = \text{argmax}_{j\in[n]}[\ybf]_j$ results in $\pbf(\param,\ybf)\geq0.$
  
    \item \textit{An unbiased estimator for the Blackwell payoff function $\mathbf{p}$ can be constructed.} Specifically, we need to construct an exploration sampling device $U$ that receives $(\param,\pibf^{(i-1)})$ in subproblem $i$ and returns $(\subpeek{\mathbf{w}}, \subpeek{\pibf})$ such that (i) for all $f \in \funcspace, \param \in \paramspace, \pibf^{(i-1)} \in \domain, i \in [n]: ~~ \hat{\pbf}\left(\param, \advfunB(\pibf^{(i-1)}, f) \right) = f(\subpeek{\pibf})\wbf_{\textrm{exp}}$, where $\left(\wbf_{\textrm{exp}},\subpeek{\pibf}\right)\sim \unbiasedestimator(\param,\pibf^{(i-1)})$, and (ii)
    $\hat\pbf$ is an unbiased estimator for the actual payoff, i.e, $\forall \param\in\paramspace,\ybf\in\advspaceB: \E[\hat{\pbf}(\param,\ybf)] = \pbf(\param,\ybf)$. The explore sampling device $U$ works as follows. Given a point $\pibf^{(i-1)}\in\Pi$ and a parameter $\param\in\paramspace,$ it draws $j\sim\text{Uniform}\{1,2,\ldots,n\}$ and returns
    \begin{equation} \label{eq:sampling_device_product_ranking}
        (\subpeek{\wbf},\subpeek{\pibf}) = \left(n(\param_j\mathbf{1}_n-\mathbf{e}_j), \pibf^{(i-1)}+j\mathbf{e}_i\right).
    \end{equation}
    Now, $\hat{\pbf}$ is an unbiased estimator of $\pbf$ because 
    
    \begin{equation*}
    \begin{aligned}
    \expect{\hat{\pbf}(\param,\ybf)}
      &= \expect{\hat{\pbf}(\param,\advfunB(\pibf^{(i-1)}, f))}
      = \E[f({\pibf}_{\text{exp}})\wbf_\text{exp}]\\
    &= \E\bigg[n\param_jf(\pibf^{(i-1)}+j\mathbf{e}_i)\mathbf{1}_n - f(\pibf^{(i-1)}+j\mathbf{e}_i)\mathbf{e}_j\bigg]\\
    &\overset{(1)}{=}\mathbf{1}_n\sum_{j=1}^n\param_jf(\pibf^{(i-1)}+j\mathbf{e}_i)-\bigg[f(\pibf^{(i-1)}+1\mathbf{e}_i),f(\pibf^{(i-1)}+2\mathbf{e}_i),\ldots,f(\pibf^{(i-1)}+n\mathbf{e}_i)\bigg]^T\\
    &\overset{(2)}{=}\mathbf{1}_n\sum_{j=1}^n\param_j\left(f(\pibf^{(i-1)}+j\mathbf{e}_i) - f(\pibf^{(i-1)})\right)\\
    &\qquad\qquad-\bigg[f(\pibf^{(i-1)}+1\mathbf{e}_i),f(\pibf^{(i-1)}+2\mathbf{e}_i),\ldots,f(\pibf^{(i-1)}+n\mathbf{e}_i)\bigg]^T+f(\pibf^{(i-1)})\mathbf{1}_n\\
    &=\mathbf{1}_n\sum_{j=1}^n\param_j\left(f(\pibf^{(i-1)}+j\mathbf{e}_i) - f(\pibf^{(i-1)})\right)\\
    &\qquad\qquad-\bigg[f(\pibf^{(i-1)}+1\mathbf{e}_i)-f(\pibf^{(i-1)}),\ldots,f(\pibf^{(i-1)}+n\mathbf{e}_i)-f(\pibf^{(i-1)})\bigg]^T\\
    &=\param^T\ybf\mathbf{1}_n - \ybf = \pbf(\param,\advfunB(\pibf^{(i-1)},f)),
    \end{aligned}
    \end{equation*}
    where $\ybf\triangleq[f(\pi^{(i-1)}(1),\ldots,\pi^{(i-1)}(i-1),j)-f(\pibf^{(i-1)})]_{j\in[n]}$. 
\end{itemize}
Here, equation~$(1)$ holds because we take $j\sim\text{Uniform}\{1,2,\ldots,n\},$ and equation~$(2)$ holds because $\sum_{j=1}^n\param_j = 1.$ Intuitively, at every round, $U$ randomly picks one of the items $j\in[n],$ and evaluate the marginal benefit of putting element $j$ on the $i^{\text{th}}$ position of $\pibf^{(i-1)}$.

Putting (i) and (ii) altogether, \Cref{alg:rank-off} is a $(\tfrac12,\tfrac12)-$extended robust approximation algorithm with $n$ subproblems. Its payoff diameter $D(\pbf)$ is $\bigO{1}$ and its payoff estimator diameter $D(\hat{\pbf})$ is $O(n)$. The dimension of vector payoffs is also $\payoffdimension=n$. It is also bandit Blackwell reducible, hence from Theorems \ref{thm:full-info-online-meta} and \ref{thm:banditILO}:
\begin{equation*}
\begin{aligned}
    \frac12\text{-regret(\Cref{alg:full-info-backbone} applied on \Cref{alg:rank-off})}&\le O(n\sqrt{T\log{n}}).\\
    \frac12\text{-regret(\Cref{alg:bandit-meta} applied on \Cref{alg:rank-off})}&\le O(n^{\schange{5/3}}\left(\log{n}\right)^{1/3}T^{2/3}).
\end{aligned}
\end{equation*}
This completes the proof.
\[\myqed\]
\endproof{}

\subsection{Proof of \texorpdfstring{\Cref{cor:rank-general}}{}}
\proof{\emph{Proof.}}
The proof for the model from \cite{asadpour2020ranking} is a direct application of \Cref{thm:seqsub} by taking $\lambda_i\triangleq\Pp_{u\sim\mathcal{G}}\left(\theta_u=i\right),$ the probability that a consumer has patience level $i,$ and $f_i(S) \triangleq \E_{u\sim\mathcal{G}}\left[\kappa_u(S)|\theta_u =i\right],$ the expected probability that a consumer with patience level $i$ clicks on any of the top $i$ products in $S$, as mentioned in \Cref{subsec:ranking}. Thus, the sequential submodular function of interest is the expected probability that a consumer clicks on at least one product when offered an ordering $\pibf:$
\begin{equation*}
    f(\pibf) = \sum_{i=1}^n\lambda_if_i\left(\{\pi(1),\ldots,\pi(i)\}\right) = \sum_{i=1}^n\Pp_{u\sim\mathcal{G}}(\theta_u=i)\E_{e\sim\mathcal{G}}\left[\kappa_u(\pibf)|\theta_u=i\right].
\end{equation*}
By invoking \Cref{thm:seqsub}, we get the desired $\bigO{n\sqrt{T\log{n}}}$ $\tfrac12$-regret in the full-information setting and $\bigO{n^{\schange{5/3}}\left(\log{n}\right)^{1/3}T^{2/3}}$ $\tfrac12-$regret in the bandit setting.

For the special consumer choice model in \cite{ferreira2019learning}, a consumer is characterized by two parameters: distribution of clicks for each item $\qbf_u = (q_{u,1},\ldots,q_{u,n})$ and attention window size $k_u.$ A consumer $u$ examines the items in the top $k_u$ positions, and an examined item $i$ is clicked with probability $q_{u,i}$ while unexamined items are never clicked. The events of clicking on two different items $i$ or $j$ in the event window are assumed to be independent. Notice that this is a special case of the choice model by \cite{asadpour2020ranking}, where $\theta_u = k_u$ and
\begin{align*}
    \kappa_u(\{\pi(1),\ldots,\pi(\theta_u)\} = \kappa_u(\{\pi(1),\ldots,\pi(k_u)\} = 1- \prod_{i=1}^{k_u}\left(1-q_{u,i}\right).
\end{align*}
The probability of click function $\kappa_u$ is monotone since when $X\subseteq Y\subseteq [n],$ we have $\prod_{i\in X}\left(1-q_{u,i}\right)\geq\prod_{i\in Y}\left(1-q_{u,i}\right)$ (as $0\leq q_{u,i}\leq 1$ for all $u$ and $i$), which implies $\kappa_u(X)\leq\kappa_u(Y).$ It is also submodular, as for all $X\subset Y\subseteq [n]$ and any item $j\notin Y, j\in[n],$ we have
\begin{align*}
    1 - \prod_{i\in Y\setminus X}\left(1-q_{u,i}\right) &\geq 0\\
    \Leftrightarrow (1-(1-q_j))\left(\prod_{i\in X}\left(1-q_{u,i}\right)\right)\left(1 - \prod_{i\in Y\setminus X}\left(1-q_{u,i}\right)\right) &\geq 0\\
    \Leftrightarrow \prod_{i\in X}(1-q_{u,i}) - (1-q_{u,j})\prod_{i\in X}(1-q_{u,i}) &\geq \prod_{i\in Y}(1-q_{u,i}) - (1-q_{u,j})\prod_{i\in Y}(1-q_{u,i})\\
    \Leftrightarrow \prod_{i\in X}(1-q_{u,i}) - \prod_{i\in X\cup\{j\}}(1-q_{u,i}) &\geq \prod_{i\in Y}(1-q_{u,i}) - \prod_{i\in Y\cup\{j\}}(1-q_{u,i})\\
    \Leftrightarrow \kappa_u(X\cup\{j\}) - \kappa_u(X) &\geq \kappa_u(Y\cup\{j\}) - \kappa_u(Y).
\end{align*}
Since this choice model is a special case of the choice model in \Cref{cor:rank-general}, we can invoke \Cref{cor:rank-general} to get the desired $\bigO{n\sqrt{T\log{n}}}$ $\tfrac12$-regret in the full-information setting and $\bigO{n^{\schange{5/3}}\left(\log{n}\right)^{1/3}T^{2/3}}$ $\tfrac12-$regret in the bandit setting.
\[\myqed\]
\endproof{}

%% file: tex/apx-mmr.tex
\section{Proofs and Remarks of Section~\ref{subsec:mmr} -- Maximizing Multiple Reserves}

\label{apx:mmr}

In this appendix, first provide a discussion on major difference between \Cref{alg:mmr-single} and the algorithm in \cite{roughgarden2019minimizing}. We then give the missing proofs of the results from \Cref{subsec:mmr}. These results are restated for convenience.
\subsection{Discussion on \texorpdfstring{\Cref{alg:mmr-single}}{}}
\label{apx:mmr-discussion}
The main difference between our algorithm and the algorithm in \cite{roughgarden2019minimizing} is the choice of revenue-from-reserves function $q$. Their revenue-from-reserves function $q$ is different (coordinate-wise less) than ours. As it becomes more clear later in the proof, the need to design a new revenue-from-reserves function stems from our requirement to construct an explore sampling device for the online bandit learning algorithm.
\subsection{Proof of \texorpdfstring{\Cref{thm:mmr}}{}}
\label{apx:mmr-main-proof}
    \proof{\emph{Proof}.} We will show that our meta Algorithms \ref{alg:full-info-backbone} and \ref{alg:bandit-meta} work by verifying the following conditions.

    \paragraph{(i) Algorithm~\ref{alg:mmr-single} is an extended $(\tfrac12, \tfrac12)$-robust approximation algorithm.}
    By Definition \ref{def:extended-robust-approx}, we need to show that if each coordinate of our vector payoffs is bounded by some function $h$:
    \begin{equation*}
        \forall j \in [m], \quad \left[\sum_{t=1}^T\pay^{(i)}(\tilde\param^{(i)}_t, \rbf_t^{(i-1)}, \vibf_t)\right]_j\geq -h(T),
    \end{equation*}
    then we must have that our overall solution's error is bounded by:
    \begin{equation*}
      \forall \rbf^* \in \constraint, \quad \sum_{t=1}^T\expect{f(\rbf_t, \vibf_t)}\geq \frac{1}{2}\cdot\sum_{t=1}^T f(\rbf^*, \vibf_t)-\frac{1}{2} nh(T).
    \end{equation*}
    Recall from Section~\ref{subsec:mmr}, that we defined the $j^{th}$ coordinate of this vector payoff to be ($j \in [m]$), $
    \indexintovector{\pay^{(i)}(\param^{(i)}, \rbf^{(i-1)}, \vibf)}{j} \triangleq
    \E_{z' \sim \param^{(i)}} \left[ q^{(i)}(z') - q^{(i)}(\feasiblereserve_j) \right]
    $, and that $\rbf_t^{(i)}$ is the reserve vector after subproblem $i$. Let's now define $S_i$ to be the set of rounds where bidder $i$ has the highest bid. We now carry out the standard offline analysis, but summed over all rounds $t \in [T]$.
    \begin{align*}
      \sum_{t=1}^{T} \expect{f(\rbf_t, \vibf_t)}
        &\overset{(1)}{=} \frac{1}{2}\sum_{t=1}^T \indexintovector{\vibf_t}{\hat{j}_t} +
          \frac12 \expect{\sum_{i=1}^n \sum_{t\in S_i} \indexintovector{\rbf_t}{i} \indicator{$\indexintovector{\rbf_t}{i} \in [\indexintovector{\vibf_t}{\hat{j}_t}, \indexintovector{\vibf_t}{j^*_t}$}} \\
        &\overset{(2)}= \frac12 \sum_{t=1}^T \indexintovector{\vibf_t}{\hat{j}_t}
         + \frac12 \expect{\sum_{i=1}^n \sum_{t=1}^T q^{(i)}_t([\rbf_t]_i)}\\
        &= \frac12 \sum_{t=1}^T \indexintovector{\vibf_t}{\hat{j}_t} + \frac12 \sum_{i=1}^n \sum_{t \in S_i} q^{(i)}_t(\indexintovector{\rbf^*}{i})
         + \frac12 \expect{\sum_{i=1}^n \sum_{t=1}^T q^{(i)}_t([\rbf_t]_i)} - \frac12 \sum_{i=1}^n \sum_{t \in S_i} q^{(i)}_t(\indexintovector{\rbf^*}{i}) \\
        &\overset{(3)}{\geq} \frac12 \sum_{t=1}^T f(\rbf^*, \vibf_t)
         + \frac12 \sum_{i=1}^n \sum_{t=1}^T \expect{q^{(i)}_t(\indexintovector{\rbf_t}{i})} - \frac12 \sum_{i=1}^n \sum_{t \in S_i} q^{(i)}_t(\indexintovector{\rbf^*}{i}) \\
        &\overset{(4)}= \frac12 \sum_{t=1}^T f(\rbf^*, \vibf_t)
         + \frac12 \sum_{i=1}^n \sum_{t=1}^T \indexintovector{\pay^{(i)}(\tilde\param^{(i)}_t, \zbf_t^{(i-1)}, f_t)}{b_i}\\
        &\overset{(5)}{\geq}\frac{1}{2}\sum_{t=1}^Tf_t(\zbf^*)-\frac{1}{2}nh(T)\,,
    \end{align*}
    where $j^*_t$ and $\hat{j}_t$ are respectively the highest and second highest bidders in the valuation profile $\vibf_t$. We also defined a function $q_t^{(i)}$ for each round, which is the same as the original $q^{(i)}$ except with $\vibf$ replaced by $\vibf_t$. Note that Inequality~$(1)$ holds because in each round $t$, with probability $1/2$, the algorithm returns $\zbf_t = \mathbf{0}_n$, which implies $f(\rbf_t, \vibf_t) = \indexintovector{\vibf_t}{\hat{j}_t}$ and with the same probability, it returns $\rbf_t= \rbf_t^{(n)}$ which implies $f(\rbf_t, \vibf_t)$ is at least equal to the reserve of the buyer with the highest bid when his reserve is less than his bid. Inequality~$(2)$ follows from the definition of $q_t^{(i)}$. Inequality~$(3)$ holds because under the optimal reserve price $\rbf^*$, $f(\rbf^*, \vibf_t)$ is less than or equal to the second highest bid $\indexintovector{\vibf_t}{\hat{j}_t}$ when the bidder with the highest bid does not win or they win and their reserve is less than or equal to $\indexintovector{\vibf_t}{\hat{j}_t}$; otherwise, $f(\rbf^*, \vibf_t)$ is equal to $q^{(i)}_t\left(\indexintovector{\rbf_t}{j^*_t}\right) \ge \indexintovector{\vibf_t}{\hat{j}_t}$. Equality~$(4)$ follows from the definition of $\pay^{(i)}.$ In this inequality $b_i$ is the index of element $\indexintovector{\rbf^*}{i}$; that is, $\indexintovector{\rbf^*}{i} =\feasiblereserve_{b_i}$. Recall that $\tilde\param^{(i)}_t$ is the (approximately-locally-optimal) distribution from which we are drawing $\indexintovector{\rbf_t}{i}$. Finally, inequality~$(5)$ follows from the assumption. This inequality is the desired result.

    \paragraph{(ii) Algorithm~\ref{alg:mmr-single} is bandit Blackwell reducible.}
    Per \Cref{def:bandit-blackwell-reducible}, to show this statement, we will verify the following conditions:
    \begin{itemize}
    \item \textit{Algorithm \ref{alg:mmr-single} is Blackwell reducible.} For every subproblem $i \in [n]$, consider an instance $\left(\algspaceB,\advspaceB,\mathbf{p}^{(i)}\right)$ of Blackwell where $\algspaceB = \paramspace = \Delta(\rcal)$ and $\advspaceB=[0,1]^{d_{\text{param}}}$, where $d_{\text{param}} = |\rcal| =m$.  We can use the Blackwell adversary function (note that we identify adversary functions with valuation vector) $\advfunB^{(i)}(\rbf, \vibf)=\left[q^{(i)}(\rho_j)\right]_{j=1,2,\ldots,m}$.
     
    The biaffine Blackwell payoff is $\mathbf{p}^{(i)}(\param,\ybf)=\param^T\ybf\mathbf{1}_n-\ybf$ where $\mathbf{1}_n$ is an $n$-dimensional all ones vector. Notice that the target set $S,$ the non-negative orthant, is response-satisfiable because if player 1 plays $\param=\mathbf{e}_{j^*}$ where $j^* = \argmax_{j \in [m]} \indexintovector{\ybf}{j}$ then for every adversary's action $\ybf \in \advspaceB$, $\pbf^{(i)}(\param, \ybf)\geq 0$.
  
    \item \textit{An unbiased estimator for the Blackwell payoff function $\mathbf{p}^{(i)}$ can be constructed.} We will show that for every subproblem $i \in [n]$, there exists an explore sampling device $U^{(i)}$ that returns $(\subpeek{\mathbf{w}^{(i)}}, \subpeek{\rbf^{(i)}})$ such that
    (i) for all $f \in \funcspace, \param \in \paramspace, \rbf \in \domain, i \in [n]: ~~ \hat{\pbf}^{(i)}\left(\param, \advfunB^{(i)}(\rbf, \vibf) \right) = f(\subpeek{\rbf^{(i)}},\vibf)\wbf^{(i)}_{\textrm{exp}}$, where $\left(\wbf^{(i)}_{\textrm{exp}},\subpeek{\rbf^{(i)}}\right)\sim \unbiasedestimator^{(i)}(\param,\rbf)$, and (ii)
    $\hat\pbf$ is an unbiased estimator for the actual payoff, i.e, $\forall \param\in\paramspace,\ybf\in\advspaceB: \E[\hat{\pbf}^{(i)}(\param,\ybf)] = \pbf^{(i)}(\param,\ybf)$.
    More specifically, we will construct  a exploring distribution $U^{(i)}$ such that 
    if $\ybf=\advfunB(\rbf,f)$ for some  $f\in\funcspace,\rbf\in\domain$, then $\E[\hat\pbf(\param,\ybf)]=\E[f(\subpeek{\rbf})\wbf_{\textrm{exp}}] = \pbf(\param,\ybf)$,
    where the expectation is taken with respect to $U^{(i)}.$ Notice that in Definition \ref{def:bandit-blackwell-reducible}, $U$ is not indexed by subproblems, but since the  $\advfunB$ for this particular problem is subproblem specific, the distribution $U$ should also depend on the subproblem. Because  we would like to construct   an unbiased estimator of the actual payoff $\pbf^{(i)},$ which is an affine function of $\ybf=q^{(i)},$ we focus on constructing an unbiased estimator for the function $q^{(i)}$. To do so, we make use of the following representation of $q^{(i)}$:
    \begin{equation} \label{eq:sample_device_reserve}
      q^{(i)}(r) = f(r\mathbf{1}_n, \vibf) - f(r(\mathbf{1}_n-\mathbf{e}_i), \vibf).
    \end{equation}
    To see what the above equation holds note that when bidder $i$ does not have the highest bid in an auction, both $q^{(i)}(r)$ and $f(r\mathbf{1}_n, \vibf) - f(r(\mathbf{1}_n-\mathbf{e}_i), \vibf)$, which is the revenue gain of increasing bidder $i$'s reserve price from zero to $r$, are both zero. When bidder $i$ has the highest bid in the auction, $q^{(i)}(r) = r$ if $r \in [\indexintovector{\vibf}{\hat{j}}, \indexintovector{\vibf}{j^*}]$ and zero otherwise. Furthermore, the revenue from the reserve price $r\mathbf{1}_n$, i.e., $f(r\mathbf{1}_n, \vibf)$, is $\indexintovector{\vibf}{\hat{j}}$ if $r < \indexintovector{\vibf}{\hat{j}}$; $r$ if $r \in [\indexintovector{\vibf}{\hat{j}}, \indexintovector{\vibf}{j^*}]$ and zero otherwise (the case $r > \indexintovector{\vibf}{j^*}$). The revenue from the reserve price $r(\mathbf{1}_n-\mathbf{e}_i)$, i.e., $f(r(\mathbf{1}_n - \mathbf{e}_i), \vibf)$, is $\indexintovector{\vibf}{\hat{j}}$ if $r < \indexintovector{\vibf}{\hat{j}}$ and zero otherwise. Thus, $f(r\mathbf{1}_n, \vibf) - f(r(\mathbf{1}_n - \mathbf{e}_i), \vibf)$ is $r$ if $r \in [\indexintovector{\vibf}{\hat{j}}, \indexintovector{\vibf}{j^*}]$ and zero otherwise, which is exactly $q^{(i)}(r)$. This interesting relationship is depicted in \Cref{fig:MMR-peek}.
  
    \input{figures/MMR-peek.tex}

    We now define the sampling distribution $U^{(i)}: \domain \times \paramspace \to \setofdistributions([0, 1]^m \times \constraint).$ For each $j \in [m],$ we pick:
    \begin{align*}
    (\subpeek{\wbf^{(i)}}, \subpeek{\rbf^{(i)}}) &= (2m(\param_j\mathbf{1}_{n}-~\mathbf{e}_j),\rho_j\mathbf{1}_n),~~\text{or}\\
    (\subpeek{\wbf^{(i)}}, \subpeek{\rbf^{(i)}}) &= (-2m(\param_j\mathbf{1}_{n}-~\mathbf{e}_j),\rho_j(\mathbf{1}_n-\mathbf{e}_i)),
    \end{align*}
    with probability $\frac{1}{2m}$ each. Recall that $\domain = \constraint =\rcal^n$, where $\rcal =\{\rho_1, \rho_2, \ldots, \rho_m\}$ is the set of possible reserve prices and $\rho_j$ is the $j$-th largest reserve price in set $\rcal$. We then have
    \begin{align*}
    \expect{\hat{\pbf}^{(i)}(\param,\ybf)}
      &= \expect{\hat{\pbf}^{(i)}(\param,\advfunB^{(i)}(\rbf, \vibf))}
      = \E[f({\rbf}^{(i)}_{\text{exp}}, \vibf)\wbf^{(i)}_\text{exp}]\\
    &= \param^T 
    \begin{bmatrix}
        f(\feasiblereserve_1 \mathbf{1}_n, \vibf) - f(\feasiblereserve_1(\mathbf{1}_n-\mathbf{e}_i), \vibf)\\
        \vdots\\
        f(\feasiblereserve_m \mathbf{1}_n, \vibf) - f(\feasiblereserve_m(\mathbf{1}_n-\mathbf{e}_i), \vibf)
      \end{bmatrix} - \begin{bmatrix}
        f(\feasiblereserve_1 \mathbf{1}_n, \vibf) - f(\feasiblereserve_1(\mathbf{1}_n-\mathbf{e}_i), \vibf)\\
        \vdots\\
        f(\feasiblereserve_m \mathbf{1}_n, \vibf) - f(\feasiblereserve_m(\mathbf{1}_n-\mathbf{e}_i), \vibf)
      \end{bmatrix}\\
    &= \param^T
    \begin{bmatrix}
        q^{(i)}(\rho_1)\\
        \vdots\\
        q^{(i)}(\rho_m)
      \end{bmatrix} - \begin{bmatrix}
        q^{(i)}(\rho_1)\\
        \vdots\\
        q^{(i)}(\rho_m)
      \end{bmatrix} = \param^T\ybf\mathbf{1}_n-\ybf = \pbf^{(i)}(\param,\ybf).
    \end{align*}
\end{itemize}
    Wrapping up, \Cref{alg:mmr-single} is an extended $\left(\frac{1}{2},\frac{1}{2}\right)$-robust approximation algorithm with $n$ subproblems and with a payoff diameter $\diameter{\pbf}$ of $O(1)$ and a payoff estimator diameter $\diameter{\hat{\pbf}}$ of $O(m)$. It is also bandit Blackwell reducible. Therefore, from Theorems \ref{thm:full-info-online-meta} and \ref{thm:banditILO}:
    \begin{align*}
      \frac12\text{-regret(\Cref{alg:full-info-backbone} applied on \Cref{alg:mmr-single})} &\le O(n T^{1/2} \log^{1/2} m) \\
      \frac12\text{-regret(\Cref{alg:bandit-meta} applied on \Cref{alg:mmr-single})} &\le O(\schange{nm^{2/3}} T^{2/3} \log^{1/3} m).
    \end{align*}
    This completes the proof.
\[\myqed\]
\endproof{}

\subsection{Proof of \texorpdfstring{\Cref{cor:mmr}}{}}
\label{apx:mmr-cor}
\proof{\emph{Proof}.}
  Let $m \in \mathbb{Z}_+$ be a parameter we choose later to balance terms. We invoke Theorem~\ref{thm:mmr} with the discretization $\rcal = \{0, \frac1m, \frac2m, \ldots, 1\}$. Given any reserves $\bm{r}^* \in [0, 1]^n$, we can produce rounded reserves $\tilde{\bm{r}}^*$ defined by rounding every reserve down to the nearest multiple of $\frac1m$: $\indexintovector{\tilde{\bm{r}}}{i} = \frac1m \left\lfloor m \indexintovector{\bm{r}}{i} \right\rfloor$. Importantly, this never causes any bidder to fail to clear their reserve price (this is why we must round down and cannot round up). Hence this can only grow the set of bidders that clear their reserve and hence the maximum bid from this set can only increase. If a bidder that was already in this set proceeds to win the auction, then they are only competing with more bidders and their reserve price drops by at most $1/m$, so their payment can only drop by at most $1/m$. If a bidder not previously in this set proceeds to win the auction, then their reserve price used to be higher than their valuation, but their valuation must be higher than the previous winner's valuation. They pay at least their reserve less $\frac1m$, so the revenue of the auction drops by at most $1/m$ in this case as well. Hence the (summed) discretization error is $T \frac1m$, and we choose either $m = \frac1n T^{1/2}$ (full-information) or $m = \schange{n^{-3/5}T^{1/5}}$ (bandit) to obtain:
  \begin{align*}
    \bigO{n T^{1/2} \log^{1/2} m} + T \frac1m &= \bigO{n T^{1/2} \log^{1/2} T} \\
    \bigO{nm\schange{^{2/3}} T^{2/3} \log^{1/3} m} + T \frac1m &= \bigO{\schange{n^{3/5}T^{4/5}} \log^{1/3} (nT)}
  \end{align*}
  This completes the proof.
\[\myqed\]
\endproof{}

%% file: figures/MMR-peek.tex
\begin{figure}[htb]
\centering
\begin{tikzpicture}[%
  auto,
  scale=0.85,
  ]
  
  \def\V{2}
  \def\v{1}
  \def\eps{0.1}
  \def\sep{6}
  
  \draw[<->] (-0.5, 0) -- (3.5, 0); \node at (4, 0) {$r$};
  \draw[<->] (0, -0.5) -- (0, 3.5); \node at (0, 4) {$f(r\mathbf{1}_m, \vbf)$};
  
  \draw (\v, -\eps) -- (\v, \eps); \node at (\v, -0.5) {$\indexintovector{\vbf}{\hat{j}}$};
  \draw (\V, -\eps) -- (\V, \eps); \node at (\V, -0.5) {$\indexintovector{\vbf}{j^*}$};
  
  \draw[red, thick] (0, \v) -- (\v, \v) -- (\V, \V);
  \draw[red, thick] (\V, 0) -- (3, 0);
  
  \draw[blue, dashed, thick] (0, \v) -- (\v, \v) -- (\V, \V);
  \draw[blue, dashed, thick] (\V, 0) -- (3, 0);
  
  \node at (0.5, 1.25) {$\indexintovector{\vbf}{\hat{j}}$};
  \node at (1.5, 1.8) {$r$};
  \node at (2.5, 0.25) {$0$};
  
  \draw[<->] (\sep-0.5, 0) -- (\sep+3.5, 0); \node at (\sep+4, 0) {$r$};
  \draw[<->] (\sep, -0.5) -- (\sep, 3.5); \node at (\sep, 4) {$f(r(\mathbf{1}_m - \mathbf{e}_i), \vbf)$};
  
  \draw (\sep+\v, -\eps) -- (\sep+\v, \eps); \node at (\sep+\v, -0.5) {$\indexintovector{\vbf}{\hat{j}}$};
  \draw (\sep+\V, -\eps) -- (\sep+\V, \eps); \node at (\sep+\V, -0.5) {$\indexintovector{\vbf}{j^*}$};
  
  \draw[red, thick] (\sep, \v) -- (\sep+\v, \v);
  \draw[red, thick] (\sep+\v, 0) -- (\sep+3, 0);
  
  \draw[blue, dashed, thick] (\sep, \v) -- (\sep+\v, \v) -- (\sep+\V, \V);
  \draw[blue, dashed, thick] (\sep+\V, 0) -- (\sep+3, 0);
  
  \node at (\sep+0.5, 1.25) {$\indexintovector{\vbf}{\hat{j}}$};
  \node at (\sep+1.5, 1.8) {$r$};
  \node at (\sep+1.5, 0.25) {$0$};
  \node at (\sep+2.5, 0.25) {$0$};
  
  \draw[<->] (2*\sep-0.5, 0) -- (2*\sep+3.5, 0); \node at (2*\sep+4, 0) {$r$};
  \draw[<->] (2*\sep, -0.5) -- (2*\sep, 3.5); \node at (2*\sep, 4) {$q^{(i)}(r)$};
  
  \draw (2*\sep+\v, -\eps) -- (2*\sep+\v, \eps); \node at (2*\sep+\v, -0.5) {$\indexintovector{\vbf}{\hat{j}}$};
  \draw (2*\sep+\V, -\eps) -- (2*\sep+\V, \eps); \node at (2*\sep+\V, -0.5) {$\indexintovector{\vbf}{j^*}$};
  
  \draw[red, thick] (2*\sep, 0) -- (2*\sep+\v, 0);
  \draw[red, thick] (2*\sep+\v, \v) -- (2*\sep+\V, \V);
  \draw[red, thick] (2*\sep+\V, 0) -- (2*\sep+3, 0);
  
  \draw[blue, dashed, thick] (2*\sep, 0) -- (2*\sep+3, 0);
  
  \node at (2*\sep+0.5, 0.25) {$0$};
  \node at (2*\sep+1.5, 1.8) {$r$};
  \node at (2*\sep+1.5, 0.25) {$0$};
  \node at (2*\sep+2.5, 0.25) {$0$};
\end{tikzpicture}
\caption{The function $q^{(i)}$ (right) and the two functions we combine to get it (left, center). The solid red line denotes the function value when $i$ is the highest bidder, and the dashed blue line denotes the function value when $i$ is not the highest bidder.}
\label{fig:MMR-peek}
\end{figure}
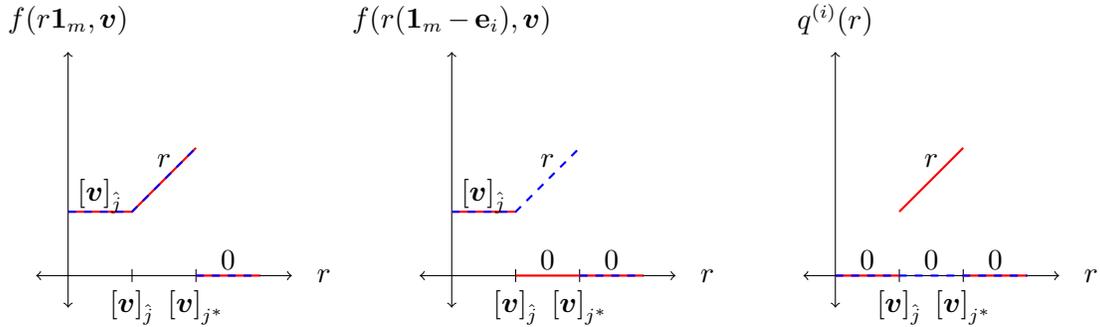

%% file: tex/apx-usm.tex
\section{Proofs and Remarks of Section~\ref{subsec:usm} -- Non-monotone Submodular Maximization}

\label{apx:usm}

In this appendix, we first discuss the differences between \Cref{alg:usm-single} and the bi-greedy algorithm by \cite{niazadeh2018optimal}, and show that 
despite these differences \Cref{alg:usm-single} obtains the same approximation factor as that of the bi-greedy algorithm.
We then present the proof of Theorem \ref{thm:usm}. 

\subsection{Discussion on \texorpdfstring{\Cref{alg:usm-single}}{}}
\label{apx:NSM-discussion}
\Cref{alg:usm-single} is a modification of the bi-greedy algorithm by \cite{niazadeh2018optimal}. But, as we show in this section,  these modifications do not change the $1/2$ approximation factor of the bi-greedy algorithm. We modify the bi-greedy algorithm to better satisfy the form of \Cref{alg:offline-meta}, ease our construction of the sampling device and unbiased estimators in the bandit case, and provide a unified framework for submodular functions with a more general domain. The major differences and their corresponding reasons are as follows:
\begin{itemize}
    \item To cover a more general discrete function domain, we optimize over points in the discrete set $\coordinatevalues$ while their algorithm optimizes over $[0,1]^n$ implemented by casting an $\epsilon-$net.
    \item To help us construct the sampling device and unbiased estimators for the bandit case, in our local optimization step, we use $\zeta^{(i)}(\hat z,z')$,
    which is a linear combination of marginal functions  $\alpha^{(i)}$ and $\beta^{(i)}$, rather than $\max\left\{\alpha^{(i)}(\hat z) - \alpha^{(i)}(z'), \beta^{(i)}(\hat z) - \beta^{(i)}(z')\right\}$, in quantifying the value decrease of $\hat z$. Recall that in this step, we choose  $\param^{(i)} \in \setofdistributions(\coordinatevalues)$ so that for all $\schange{\hat{z}} \in \coordinatevalues$, 
     $
        \E_{z' \sim \param^{(i)}} \left[ \frac12 \alpha^{(i)}(z') + \frac12 \beta^{(i)}(z') - \zeta^{(i)}(\schange{\hat{z}}, z) \right] \ge 0
    $. 
\end{itemize}

\medskip

Using the technique in \cite{niazadeh2018optimal}, as we argue next, 
 we can still find  $\param^{(i)}$  that satisfies the condition in the local optimization step. Note that the bi-greedy analysis in \cite{niazadeh2018optimal} proves that satisfying this condition implies that \Cref{alg:usm-single} is a $\frac12$-approximation algorithm for the discretized submodular maximization problem.\medskip

\textbf{Satisfying the Local Optimization Step.} Here, we show how to choose $\param^{(i)} \in \setofdistributions(\coordinatevalues)$ that satisfies the condition in the local optimization step of \Cref{alg:usm-single}. To do so, 
First, we choose $z_\ell \in \argmax_{z \in \coordinatevalues} f(z, \usmupper^{(i-1)})$ and $z_u \in \argmax_{z \in \coordinatevalues} f(z, \usmlower^{(i-1)})$. Then, we look at these two cases.
  
  \paragraph{Case (i): $z_u \le z_\ell$.} We want to prove that deterministically returning $z_\ell$ ($\param^{(i)}$ puts all its weight on $z_{\ell}$) suffices. The key realization is that in this case, $z_u$ and $z_\ell$ maximize the functions $f(\cdot,\usmupper^{(i-1)})$ and $f(\cdot,\usmlower^{(i-1)}))$ respectively:
  \begin{align}
    f(z_\ell, \usmupper^{(i-1)}) \ge f(z_u, \usmupper^{(i-1)}),\qquad 
    f(z_u, \usmlower^{(i-1)}) \ge f(z_\ell, \usmlower^{(i-1)}). \label{eq:first_two}
  \end{align}
  We know by submodularity that two points are better than their coordinate-wise max and min:
  \begin{align}\label{eq:third}
    f(z_u, \usmlower^{(i-1)}) + f(z_\ell, \usmupper^{(i-1)}) &\le f(z_\ell, \usmlower^{(i-1)}) + f(z_u, \usmupper^{(i-1)}).
  \end{align}
  Since adding up the first two inequalities in Equation \eqref{eq:first_two} yields the third inequality in Equation \eqref{eq:third}, but with the direction reversed, we know all three must hold with equality. We conclude by noting that since $z_\ell$ maximizes both functions, it also maximizes both $\alpha^{(i)}$ and $\beta^{(i)}$ at some nonnegative value and hence satisfies the desired condition in the local step optimization.

  \paragraph{Case (ii): $z_\ell < z_u$.}
  Suppose that the algorithm is able to find a $\param^{(i)}$ such that for any $\hat{z} \in [z_\ell, z_u]$, we have
  \begin{align}
      \E_{z'\sim\param^{(i)}}\left[\frac12 \alpha^{(i)}(z') + \frac12 \beta^{(i)}(z') - \zeta^{(i)}({\hat{z}}, z)\right] = \frac12 \alpha^{(i)}(z') + \frac12 \beta^{(i)}(z') - \zeta^{(i)}({\hat{z}}, z')\geq 0.\label{eq:condition}
  \end{align}
  We claim that this equation is still true for $\hat{z}$ outside of the interval $[z_\ell, z_u]$.
  
  Suppose that $\hat{z} < z_\ell$. By the choice of $z_\ell$, we know that $\beta^{(i)}(z_\ell) \ge \beta^{(i)}(\hat{z})$. By submodularity, we know that:
  \begin{align*}
    f(\hat{z}, \usmlower^{(i-1)}) + f(z_\ell, \usmupper^{(i-1)}) &\le f(z_\ell, \usmlower^{(i-1)}) + f(\hat{z}, \usmupper^{(i-1)}) \\
    \alpha^{(i)}(\hat{z}) + \beta^{(i)}(z_\ell) &\le \alpha^{(i)}(z_\ell) + \beta^{(i)}(\hat{z}) \\
    \beta^{(i)}(z_\ell) - \beta^{(i)}(\hat{z}) &\le \alpha^{(i)}(z_\ell) - \alpha^{(i)}(\hat{z}).
  \end{align*}
  Since the LHS is nonnegative by the choice of $z_{\ell},$ so is the RHS. We have shown that $z_\ell$ has strictly larger $\alpha^{(i)}$ and $\beta^{(i)}$ values (than $\hat{z}$) and hence inequality in Equation \eqref{eq:condition} must be valid for $\hat{z} < z_\ell$ as well. Analogous reasoning shows the same for the $z_r < \hat{z}$ case. Notice that the method in \cite{niazadeh2018optimal} is able to compute a $\param^{(i)}$ that guarantees
$
      \E_{z'\sim\param^{(i)}}\left[\frac12 \alpha^{(i)}(z') + \frac12 \beta^{(i)}(z') - \zeta^{(i)}({\hat{z}}, z)\right]\geq 0
$
  for any $\hat{z}\in[z_\ell,z_u],$ which means this is also true for any $\hat{z}\in\coordinatevalues.$ Recall that the payoff function is 
    \begin{equation*}
    \left[\pay\left(\param,\usmlower^{(i-1)},f\right)\right]_j = \E_{z' \sim \param} \left[ \frac12 \alpha^{(i)}(z') + \frac12 \beta^{(i)}(z') - \zeta^{(i)}({\rho_j}, z) \right],
    \end{equation*}
  so such $\param^{(i)}$ also guarantees that $\pay\left(\param^{(i)},\usmlower^{(i-1)},f\right)$ is in the positive orthant.
  
